\documentclass{article}

\PassOptionsToPackage{numbers, compress}{natbib}
\usepackage[final]{neurips_2021}

\usepackage[utf8]{inputenc} %
\usepackage[T1]{fontenc}    %
\usepackage[colorlinks=true,citecolor=black,urlcolor=black,linkcolor=black]{hyperref}       %
\usepackage{url}            %
\usepackage{booktabs}       %
\usepackage{amsfonts}       %
\usepackage{nicefrac}       %
\usepackage{microtype}      %
\usepackage{xcolor}         %
\usepackage{graphicx, nicefrac}
\usepackage{booktabs} %
\usepackage[english]{babel}
\usepackage{subcaption}
\usepackage{xspace}
\usepackage{float}
\usepackage{multirow}
\usepackage{comment}
\usepackage{rotating}

\usepackage[utf8]{inputenc}

\usepackage[most]{tcolorbox}

\usepackage{blindtext}
\usepackage{ifthen}
\newif\ifapxinsamepage

\apxinsamepagetrue

\usepackage{xr}

\ifapxinsamepage{
} \else {
	\externaldocument[]{appendix_main}
}
\fi

\usepackage{hyperref}
\usepackage{amsmath,amsfonts,amssymb,mathrsfs, amsthm}
\usepackage[capitalize,nameinlink]{cleveref}

\crefname{section}{Sec.}{Secs.}
\crefname{proposition}{Prop.}{Props.}
\crefname{lemma}{Lem.}{Lems.}
\crefname{model}{Mod.}{Mods.}
\crefname{appendix}{App.}{Apps.}
\crefname{algorithm}{Alg.}{Algs.}

\usepackage{algorithm}
\usepackage{algorithmic}

\usepackage[normalem]{ulem} %

\usepackage{cancel} 
\usepackage{amsmath}
\usepackage{mathtools}
\newtheorem{theorem}{Theorem}[section]

\newtheorem{lemma}[theorem]{Lemma}

\renewcommand{\paragraph}[1]{{\bf#1}~~}

\usepackage{tikz,pgfplots,pgfmath}
\usetikzlibrary{plotmarks,shapes,shapes.arrows,positioning,calc,spy,intersections}
\usetikzlibrary{decorations.pathreplacing,decorations.text}
\pgfplotsset{compat=newest} 
\pgfkeys{/pgf/number format/.cd,1000 sep={}}
\usepackage{placeins}

\renewcommand{\vec}[1]{\mathbf{#1}}
\definecolor{blue}{rgb}{0.12156862745098039, 0.4666666666666667, 0.7058823529411765}
\definecolor{red}{rgb}{0.8392156862745098, 0.15294117647058825, 0.1568627450980392}

\newlength\figureheight
\newlength\figurewidth

\usepackage{tabularx}
\usepackage{array}
\newcommand{\PreserveBackslash}[1]{\let\temp=\\#1\let\\=\temp}
\newcolumntype{C}[1]{>{\PreserveBackslash\centering}p{#1}}

\pgfplotsset{/pgf/number format/.cd, 1000 sep={}}
\pgfkeys{/pgf/number format/.cd,1000 sep={}}

\pgfplotsset{every axis/.append style={
		grid style={line width=0.6pt,dotted,gray}}}

\pgfplotsset{every axis/.append style={
		legend style={inner xsep=1pt, inner ysep=0.5pt, nodes={inner sep=1pt, text depth=0.1em},draw=none,fill=none}
}}

\pgfplotsset{ignore legend/.style={every axis legend/.code={\let\addlegendentry\relax}}}

\usepgfplotslibrary{groupplots}

\newcommand\inputpgf[2]{{
\let\pgfimageWithoutPath\pgfimage
\renewcommand{\pgfimage}[2][]{\pgfimageWithoutPath[##1]{#1/##2}}
\input{#1/#2}
}}

\DeclareUnicodeCharacter{2212}{-}

\makeatletter
\newcommand\thefontsize[1]{{#1 The current font size is: \f@size pt\par}}
\makeatother

\usepackage{wrapfig}

\newcommand{\kron}{\raisebox{1pt}{\ensuremath{\:\otimes\:}}} %
\usepackage{bm}

\newcommand{\mbf}[1]{\mathbf{#1}}

\newcommand{\KL}{\mathrm{KL}}

\newcommand{\vu}{\mbf{u}}
\newcommand{\vf}{\mbf{f}}

\newcommand{\vq}{\mbf{q}}
\newcommand{\vg}{\mbf{g}}
\newcommand{\vm}{\mbf{m}}
\newcommand{\vs}{\mbf{s}}

\newcommand{\vmu}{\bm{\mu}}

\newcommand{\vbeta}{\bm{\beta}}

\newcommand{\vxi}{\bm{\xi}}
\newcommand{\vtheta}{\bm{\theta}}

\newcommand{\vphi}{\bm{\phi}}
\newcommand{\MA}{\mbf{A}}
\newcommand{\MB}{\mbf{B}}
\newcommand{\MC}{\mbf{C}}
\newcommand{\MD}{\mbf{D}}
\newcommand{\ME}{\mbf{E}}
\newcommand{\MG}{\mbf{G}}
\newcommand{\MQ}{\mbf{Q}}
\newcommand{\MP}{\mbf{P}}
\newcommand{\MV}{\mbf{V}}
\newcommand{\MY}{\mbf{Y}}
\newcommand{\MX}{\mbf{X}}
\newcommand{\MZ}{\mbf{Z}}
\newcommand{\MW}{\mbf{W}}

\newcommand{\MH}{\mbf{H}}
\newcommand{\MS}{\mbf{S}}
\newcommand{\MT}{\mbf{T}}
\newcommand{\MR}{\mbf{R}}
\newcommand{\MK}{\mbf{K}}
\newcommand{\MI}{\mbf{I}}
\newcommand{\MJ}{\mbf{J}}
\newcommand{\MLambda}{\mbf{\Lambda}}
\newcommand{\N}{\mathrm{N}}
\newcommand{\T}{\top}    %
\newcommand{\R}{\mathbb{R}}    %
\newcommand{\BO}{\mathcal{O}}    %
\newcommand{\intd}[1]{\,\mathrm{d} #1 \,}
\newcommand{\GP}{\mathcal{GP}}
\newcommand{\E}{\mathbb{E}}    %

\newcommand{\MPhi}{\mbf{\Phi}}

\newcommand{\MF}{\mbf{F}}
\newcommand{\ML}{\mbf{L}}
\newcommand{\vw}{\mbf{w}}

\newcommand{\LL}{\mathcal{L}}

\newcommand{\RR}{ \mathbb{R} }
\newcommand{\mprod}{\textstyle\prod}
\newcommand{\mint}{\textstyle\int}

\newcommand{\diff}[2]{\mathrm{\frac{\partial\mathit{#1}}{\partial\mathit{#2}}}}

\newcommand{\intg}[4]{\int_{#3}^{#4} #1 \, \mathrm{d}#2}

\newcommand{\iK}[1]{{\MK^{-1}_{#1 #1}}} %
\newcommand{\K}[2]{{\MK_{#1 #2}}} %
\newcommand{\sK}[3]{{\MK^{(#3)}_{#1 #2}}} %
\newcommand{\siK}[2]{{\MK^{-(#2)}_{#1 #1}}} %

\newcommand{\kt}{\Time} %
\newcommand{\ks}{\Space} %

\newcommand{\kx}{\MX} %
\newcommand{\kz}{\MZ} %
\newcommand{\skx}{\SPACE} %
\newcommand{\skz}{\MZ_{\Space}} %
\newcommand{\tkx}{\MX_{\kt}} %
\newcommand{\tkz}{\MZ_{\kt}} %

\newcommand{\Kzz}{\K{\kz}{\kz}}
\newcommand{\iKzz}{\iK{\kz}}

\newcommand{\Kxx}{\K{\kx}{\kx}}

\newcommand{\Kxz}{\K{\kx}{\kz}}
\newcommand{\Kzx}{\K{\kz}{\kx}}

\newcommand{\sKzz}{\sK{\skz}{\skz}{\ks}}
\newcommand{\siKzz}{\siK{\skz}{\ks}}
\newcommand{\tKzz}{\sK{\tkz}{\tkz}{t}}
\newcommand{\sKxx}{\sK{\skx}{\skx}{\ks}}

\newcommand{\tKxx}{\sK{\Time}{\Time}{t}}
\newcommand{\tiKxx}{\siK{\Time}{t}}

\newcommand{\sKxz}{\sK{\skx}{\skz}{\ks}}
\newcommand{\tKxz}{\sK{\tkx}{\tkz}{t}}
\newcommand{\sKzx}{\sK{\skz}{\skx}{\ks}}

\newcommand{\sKzzn}{\sK{\MZ_{\Space}}{\MZ_{\Space}}{\ks}}
\newcommand{\siKzzn}{\siK{\MZ_{\Space}}{\ks}}

\newcommand{\sKxzn}{\sK{\SPACE_{\spaceindex}}{\MZ_{\Space}}{\ks}}
\newcommand{\sKzxn}{\sK{\MZ_{\Space}}{\SPACE_{\spaceindex}}{\ks}}

\newcommand{\sKxxn}{\sK{\gridX_{\dindex}}{\gridX_{\dindex}}{\ks}}

\newcommand{\tKxxn}{\sK{\gridX_{\dindex}}{\gridX_{\dindex}}{t}}

\newcommand{\vectext}{\text{vec}}

\newcommand{\param}{\boldsymbol\xi}
\newcommand{\natp}{\boldsymbol\lambda}
\newcommand{\apxnatp}{\boldsymbol{\widetilde{\lambda}}}
\newcommand{\priornatp}{\boldsymbol\eta}
\newcommand{\meanp}{\boldsymbol\mu}

\newcommand{\gradm}{\vg(\vm)}
\newcommand{\grads}{\vg(\postcov)}
\newcommand{\gradmu}{\vg(\meanp)}

\newcommand{\apxy}{\widetilde{\MY}}
\newcommand{\apxv}{\widetilde{\MV}}

\newcommand{\M}{M}
\newcommand{\Nt}{N_t}
\newcommand{\Ns}{N_\Space}

\newcommand{\Ms}{M_\Space}

\newcommand{\bigO}{ {\cal O}}

\newcommand{\state}{\bar{\vf}}
\newcommand{\SPACE}{{\MS}}
\newcommand{\Space}{{\vs}}
\newcommand{\Time}{{\bm{t}}}
\newcommand{\dd}{\,\mathrm{d}}
\newcommand{\Feedback}{{\MF}}
\newcommand{\postcov}{{\MP}}

\newcommand{\SpaceCov}{\MK^{(\Space)}_{\SPACE\SPACE}}

\newcommand{\timeindex}{n}
\newcommand{\spaceindex}{k}

\newcommand{\dindex}{{\timeindex, \spaceindex}}

\newcommand{\gridX}{\MX}

\renewcommand{\mid}{\,|\,}

\newcommand{\qum}{\vm^{(\vu)}}
\newcommand{\qfm}{\vm}
\newcommand{\qus}{\postcov^{(\vu)}}
\newcommand{\qfs}{\postcov}

\newcommand{\filtmean}{\hat{\vm}}
\newcommand{\smoothmean}{\bar{\vm}}
\newcommand{\filtcov}{\hat{\MP}}
\newcommand{\smoothcov}{\bar{\MP}}

\newcommand{\theo}[1]{{\color{red}Theo: [#1]}}

\usepackage[nolist,smaller]{acronym}
\newcommand{\acro}[1]{\textsc{#1}\xspace}

\newcommand{\gp}{\acro{GP}}
\newcommand{\gps}{\acro{GPs}}
\newcommand{\CVI}{\acro{CVI}}

\newcommand{\VGP}{\acro{VGP}}
\newcommand{\SVGP}{\acro{SVGP}}

\newcommand{\STCVI}{\STVGP}
\newcommand{\MFSTCVI}{\MFSTVGP}

\newcommand{\STVGP}{\acro{ST-VGP}}
\newcommand{\STSVGP}{\acro{ST-SVGP}}
\newcommand{\MFSTVGP}{\acro{MF-ST-VGP}}
\newcommand{\MFSTSVGP}{\acro{MF-ST-SVGP}}

\newcommand{\SKIP}{\acro{SKIP}}
\newcommand{\SKI}{\acro{SKI}}

\newcommand{\FITC}{\acro{FITC}}

\newcommand{\KLD}{\acro{KLD}}

\newcommand{\ELBO}{\acro{ELBO}}

\newcommand{\ulemma}{Lemma }

\usepackage{xspace}

\newcommand{\ie}{\textit{i.e.}\xspace}

\newcommand{\etc}{\textit{etc.}\xspace}

\newcommand{\NYC}{\acro{nyc}}

\newcommand{\nyccrime}{\acro{nyc-crime}}
\newcommand{\airquality}{\acro{air-quality}}

\newcommand{\pmten}{$\text{PM}_{10}$\xspace}

\newcommand{\ZZ}{$\mathbf{\MZ}$\xspace}

\newcommand{\nipstitle}[1]{{%
    \def\toptitlebar{\hrule height4pt \vskip .25in \vskip -\parskip} 
    \def\bottomtitlebar{\vskip .29in \vskip -\parskip \hrule height1pt \vskip .09in} 
    \phantomsection\hsize\textwidth\linewidth\hsize%
    \vskip 0.1in%
    \toptitlebar%
    \begin{minipage}{\textwidth}%
        \centering{\LARGE\bf #1\par}%
    \end{minipage}%
    \bottomtitlebar%
    \addcontentsline{toc}{section}{#1}%
}}

\title{Spatio-Temporal Variational Gaussian Processes}

\author{%
  Oliver Hamelijnck\thanks{equal contribution}\\
  The Alan Turing Institute /\\
  University of Warwick\\
  \texttt{ohamelijnck@turing.ac.uk}
  \And
  William J.\ Wilkinson\footnotemark[1]\\
  Aalto University\\
  \texttt{william.wilkinson@aalto.fi}
  \And
  Niki A.\ Loppi\\
  NVIDIA\\
  \texttt{nloppi@nvidia.com}   
  \AND
  Arno Solin\\
  Aalto University\\
  \texttt{arno.solin@aalto.fi}
  \And
  Theodoros Damoulas\\
  The Alan Turing Institute /\\
  University of Warwick\\
  \texttt{tdamoulas@turing.ac.uk}
}

\begin{document}

\maketitle

\begin{abstract}
  We introduce a scalable approach to Gaussian process inference that combines spatio-temporal filtering with natural gradient variational inference, resulting in a non-conjugate \gp method for multivariate data that scales linearly with respect to time. Our natural gradient approach enables application of parallel filtering and smoothing, further reducing the temporal span complexity to be logarithmic in the number of time steps. We derive a sparse approximation that constructs a state-space model over a reduced set of spatial inducing points, and show that for separable Markov kernels the full and sparse cases exactly recover the standard variational \gp, whilst exhibiting favourable computational properties. To further improve the spatial scaling we propose a mean-field assumption of independence between spatial locations which, when coupled with sparsity and parallelisation, leads to an efficient and accurate method for large spatio-temporal problems.
\end{abstract}

\section{Introduction}
\label{sec:introduction}

Most real-world processes occur across space and time, exhibit complex dependencies, and are observed through noisy irregular samples. Take, for example, the task of modelling air pollution across a city. Such a task involves large amounts of noisy, partially-observed data with strong seasonal effects governed by weather, traffic, human movement, \etc This setting motivates a probabilistic perspective, allowing for the incorporation of prior knowledge and the quantification of uncertainty.%
Gaussian processes (\gps, \cite{rasmussen2003gaussian}) provide such a probabilistic modelling paradigm, but their inherent cubic computational scaling in the number of data, $N$, limits their applicability to spatio-temporal tasks. Arguably the most successful methods to address this issue are sparse \gps \cite{unifying_sparse_gps:2005}, which summarise the true \gp posterior through a reduced set of $M$ \emph{inducing points} and have dominant computational scaling $\bigO(NM^2)$, and spatio-temporal \gps \citep{Sarkka+Solin+Hartikainen:2013}, which rewrite the \gp prior as a state-space model and use filtering to perform inference in $\bigO(Nd^3)$, where $d$ is the dimensionality of the state-space. Sparse \gps and spatio-temporal \gps have been combined by constructing a Markovian system in which a set of \emph{spatial} inducing points are tracked over time \cite{hartikainen_sparse_st_gps:2011,wilkinson2020state}.
However, existing methods for spatio-temporal \gps make approximations to the prior conditional model \cite{hartikainen_sparse_st_gps:2011} or do not exploit natural gradients \cite{tebbutt2021combining}, meaning they do not provide the same inference and learning results as state-of-the-art variational \gps \cite{hensman_gp_for_big_data:2013} in the presence of non-conjugate likelihoods or sparsity, which has hindered their widespread adoption. %
We introduce \emph{spatio-temporal variational \gps} (\STVGP), which provide the \emph{exact} same results as standard variational \gps, whilst reducing the computational scaling in the temporal dimension from cubic to linear. %
\STVGP is derived using a natural gradient variational inference approach based on filtering and smoothing. We also derive this method's sparse variant, and demonstrate how it enables the use of significantly more inducing points than the standard approach, leading to improved predictive performance. %

We then show how the spatio-temporal structure can be exploited even further to improve both the temporal and spatial scaling. We demonstrate for the first time how to apply parallel filtering and smoothing \cite{sarkka_temporal2021} to non-conjugate \gps to reduce the temporal (span) complexity to be logarithmic. We then reformulate the model to enable an efficient mean-field approximation across space, improving the complexity with respect to the number of spatial points. We analyse the practical performance and scalability of our proposed methods, demonstrating how they make it possible to apply \gps to large-scale spatio-temporal scenarios without sacrificing inference quality.
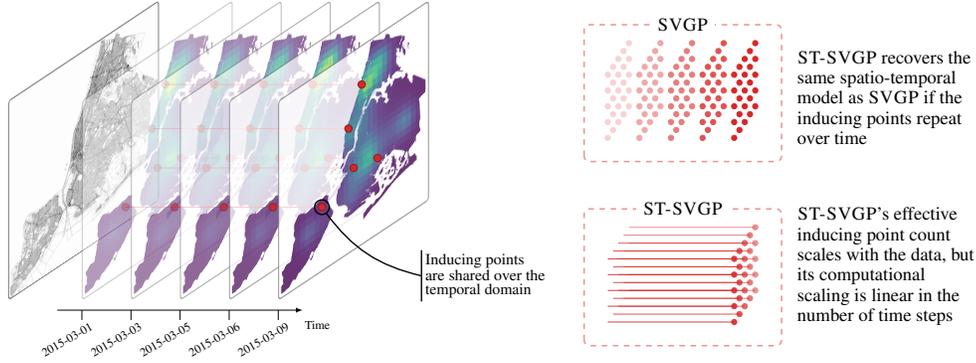
\begin{figure}[t!]
  \centering
  \begin{subfigure}{.55\textwidth}
	\centering
	\scriptsize
	\resizebox{\textwidth}{!}{%
		\begin{tikzpicture}[outer sep=0]

			\newcommand{\drawmask}[4]{%
				\node at (#1,#2) {
					\pgfdeclareimage[width=4cm]{themask}{#3}
					\pgflowlevel{\pgftransformcm{.75}{0.5}{0}{1}{\pgfpoint{0}{0}}}
					\tikz\node[inner sep=1pt,fill=white,draw=black!50,thick,rounded corners=2pt,opacity=#4]{\pgfuseimage{themask}};
				};}

			\draw[draw=white,outer sep=0] (-2,-3) rectangle (6.5,4.2);

			\tikzstyle{inducing} = [circle, draw=black!30!red, fill=red, inner sep=1.5pt]

			\def\dates{2015-03-01, 2015-03-03, 2015-03-05, 2015-03-06, 2015-03-09}

			\def\inducing{-0.5701511529340699/1.5, -0.3/2.4, -0.46209069176044193/0.7, 0.024181383520883837/0.8999999999999999, -1.1104534588022097/-0.10000000000000009}

			\drawmask{0}{0}{./fig/basemap}{1}

			\foreach \x [count=\i] in \dates {

				\ifnum\i>1
				\foreach \u / \v [count=\j] in \inducing {
					\draw[red,opacity={.2*\i}] ({.5+\i+\u},\v) -- ({.5+\i+\u-1},\v);
				}
				\fi

				\foreach \u / \v [count=\j] in \inducing {
					\node[inducing,opacity={.2*\i}] at ({.5+\i+\u},\v) {};
				}      

				\drawmask{{.5+\i}}{0}{./fig/map-\i}{{.3+.1*\i}}

				\foreach \u / \v [count=\j] in \inducing {
					\node[inducing,opacity={.2*\i}] (ind-\i-\j) at ({.5+\i+\u},\v) {};
				}

				\draw[black] ({.5+\i-2},-2.2) -- ({.5+\i-2},-2.4);
				\node[rotate=30, align=left] at ({.5+\i-2.3},-2.8) {\x};
				
			}

			\draw[black,-latex] (-1,-2.2) -> (4,-2.2);
			\node at (4.3,-2.5) {Time};

			\node[text width=3cm] (lab) at (8,-1.5) {\footnotesize Inducing points are shared over the \mbox{temporal domain}};
			\node[circle,draw=black,thick,inner sep=3pt] (ind-mark) at (ind-5-5) {};
			\draw[black,thick] (lab.north west) -- (lab.south west);
			\draw[black,thick] (lab.west) to [bend left=20] (ind-mark);

	\end{tikzpicture}}
  \end{subfigure}
  \begin{subfigure}{.35\textwidth}
    \centering
    \scalebox{.7}{%
    \begin{tikzpicture}[outer sep=0]

			\tikzstyle{inducing} = [circle,draw=red, fill=red, inner sep=1pt]

			\draw[draw=red!50,rounded corners=3pt,thick,dashed] (10.25,1) rectangle (14,3.6);
			\draw[draw=red!50,rounded corners=3pt,thick,dashed] (10.25,-2.5) rectangle (14,0.1);

			\node[fill=white] at (12.125,3.6) {\SVGP};
			\node[fill=white] at (12.125,0.1) {\STSVGP};

			\foreach \z [count=\k] in {0,1,2,3,4}{
				\foreach \x [count=\i] in {1,2,3,4,5} {
					\foreach \y [count=\j] in {1,2,3,4,5} {
						\node[inducing,opacity={.2*\k}] at ({10+.6*\k+.1*\i},{1+.3*\j+.15*\i}) {};
					}
				}
			}

			\foreach \x [count=\i] in {1,2,3,4,5} {
				\foreach \y [count=\j] in {1,2,3,4,5} {
					\draw[red,opacity={1-.1*\i}] 
					({10+.6*1+.1*\i},{-2.5+.3*\j+.15*\i}) --
					({10+.6*5+.1*\i},{-2.5+.3*\j+.15*\i});
				}
			}			
			
			\foreach \x [count=\i] in {1,2,3,4,5} {
				\foreach \y [count=\j] in {1,2,3,4,5} {
					\node[inducing,opacity={1-.1*\i}] at ({10+.6*5+.1*\i},{-2.5+.3*\j+.15*\i}) {};
				}
			}

			\node[text width=3.4cm,align=flush left] at (16,.5) {\STSVGP recovers the same spatio-temporal model as \SVGP if the inducing points repeat over time\\[3em]\STSVGP's effective inducing point count scales with the data, but its computational scaling is linear in the number of time steps};

    \end{tikzpicture}}
  \end{subfigure}
   \\[-0.6em]
	\caption{A demonstration of the spatio-temporal sparse variational GP (\STSVGP) applied to crime count data in New York. \STSVGP tracks spatial points over time via spatio-temporal filtering. The colourmap is the posterior mean, and the red dots are spatial inducing points. The diagram shows the difference between how inducing points are treated in \STSVGP and \SVGP.  %
	}
	\label{fig:teaser}
	\vspace*{-0.75em}
\end{figure}

\subsection{Related Work}
\label{sec:related_work}

\gps are commonly used for spatio-temporal modelling in both machine learning and spatial statistics \cite{rasmussen2003gaussian, gelfand_stats+gps:2016, cressie_spati_temporal_stats:2011, banerjee_spatial:2004}. Many approaches to overcome their computational burden have been proposed, from nearest neighbours \cite{datta_nearest_neighbor_gps:2016} to parallel algorithms on GPUs \citep{wang_gp_million_points:2019}. Within machine learning, the sparse GP approach %
 is perhaps the most popular \cite{unifying_sparse_gps:2005, titsias_sparse_gp:2009}, and is typically combined with mini-batching to allow training on massive datasets \cite{hensman_gp_for_big_data:2013}. However, it fails in practical cases where the number of inducing points must grow with the size of the data, such as for time series \cite{wilkinson_end_to_end:2019}. 

When the data lie on a grid, separable kernels exhibit Kronecker structure which can be exploited for efficient inference \citep{saatci_thesis:2011}. This approach has been generalised to the partial grid setting \citep{wilson_gpatt:2014}, and to structured kernel interpolation (\SKI, \cite{wilison_ski_gp:2015}) which requires only that inducing points be on a grid. %
Generally, these approaches are limited to the conjugate case, although Laplace-based extensions exist \citep{flaxman_non_gaussian_kronecker_gp:2015}. \citet{bruinsma2020scalable} present an approach to spatio-temporal modelling that performs an orthogonal projection of the data to enforce independence between the latent processes. %

It has been shown that variational \gps can be computed in linear time either by exploiting sparse precision structure \cite{durrande_banded_gp:2019} or via filtering and smoothing \cite{chang_fast_vi:2020}.
Other inference schemes such as Laplace and expectation propagation have also been proposed \citep{nickisch2018state, wilkinson2020state}.
In the spatio-temporal case, sparsity has been used in the spatial dimension \cite{hartikainen_sparse_st_gps:2011, Sarkka+Solin+Hartikainen:2013}. These methods historically suffered from the fact that \emph{i)} filtering was not amenable to fast automatic differentiation due to its recursive nature, and \emph{ii)} state-of-the-art inference schemes had not been developed to make them directly comparable to other methods. The first is no longer an issue since many machine learning frameworks are now capable of efficiently differentiating recursive models \cite{chang_fast_vi:2020}. We address the second point with this paper. A similar algorithm to ours that is also sparse in the temporal dimension has been developed \cite{wilkinson2021sparse, vincent_doubly_sparse:2020}, and relevant properties of the spatio-temporal model presented here are also analysed in \cite{tebbutt2021combining}. Fourier features \cite{hensman_vff:2017} are an alternative approach to scalable \gps, but are not suited to very long time series with high variability due to the need for impractically many inducing features. %

\section{Background}
\label{sec:background}

We consider data lying on a spatio-temporal grid comprising input--output pairs, $\{\gridX^{(st)} \in \R^{N_t \times N_\Space \times D}, \MY^{(st)} \in \R^{N_t \times N_\Space} \}$, where $N_t$ is the number of temporal points, $N_\Space$ the number of spatial points, and $D = 1 + D_\Space$ the input dimensionality (with $D_\Space$ being the number of spatial dimensions). We use $t$ and $\Space$ to represent time and space respectively. The assumption of the grid structure is relaxed via the introduction of sparse methods in \cref{sec:st_cvi}, and by the natural handling of missing data.
	
For consistency with the \gp literature we let $\MX = \text{vec}(\gridX^{(st)}) \in \R^{N \times D}$, $\MY = \text{vec}(\MY^{(st)}) \in \R ^ {N \times 1}$, where $N=N_t N_\Space$ is the total number of data points. We use the operator $\text{vec}(\cdot)$ to simply convert data from a spatio-temporal grid into vector form, whilst keeping observations ordered by time and then space. For notational convenience we define $\MX_{n,k} = \MX^{(st)}_{n,k}$, $\MY_{n,k} = \MY^{(st)}_{n,k}$, which indexes data at time index $n$ and spatial index $k$. We use $t_\timeindex$ to denote the $\timeindex$'th time point, $\SPACE \in \R^{N_\Space \times D_\Space}$ to denote all spatial grid points and $\SPACE_\spaceindex$ the $\spaceindex$'th one. Let $f: \R^D \rightarrow \R$ to be a random function with a zero-mean \gp prior, then for a given likelihood $p(\MY\mid f(\MX))$ the generative model is,
\begin{equation}\label{eq:gp}
		f(x) \sim \GP(0, \kappa(x, x')), \quad \MY \mid \vf \sim \mprod_{\timeindex=1}^{N_t} \mprod_{\spaceindex=1}^{N_\Space} p(\MY_{\timeindex,\spaceindex} \mid \vf_{\timeindex,\spaceindex}) ,
\end{equation}
where $\vf_{\timeindex,\spaceindex} = f(\gridX_{\timeindex,\spaceindex})$, and we let $\vf_{\timeindex}$ be the function values of all spatial points at time $t_\timeindex$. When the kernel $\kappa$ is evaluated at given inputs we write the corresponding gram matrix as $\MK_{\MX\MX'} = \kappa(\MX, \MX')$. To make it explicit that $f$ takes spatio-temporal inputs we also abuse the notation slightly to write $f(x) = f(t, \Space)$ and  $\kappa(x, x')=\kappa(t, \Space, t', \Space')$. A summary of all notation used is provided in \cref{sec:app:nomencalture}. For Gaussian likelihoods the posterior, $p(\vf \mid \MY)$, is available in closed form, otherwise approximations must be used. In either case, inference typically comes at a cubic cost of $\bigO(\Nt^3 N_\Space^3)$.
\subsection{State Space Spatio-Temporal Gaussian Processes}
\label{sec:state-space-spatio-temporal}

One method for handling the cubic scaling of \gps is to reformulate the prior in \cref{eq:gp} as a state space model, reducing the computational scaling to linear in the number of time points \cite{Sarkka+Solin+Hartikainen:2013}.
The enabling assumption is that the kernel is both Markovian and separable between time and space: $\kappa(t,\Space,t',\Space') = \kappa_t(t,t')\,\kappa_{\Space}(\Space, \Space')$. We use the term \emph{Markovian kernel} to refer to a kernel which can be re-written in state-space form (see \cite{Solin:2016} for an overview).  
First, we write down the \gp prior as a stochastic partial differential equation (SPDE, see \cite{DaPrato+Zabczyk:1992})
$\partial_t\state(t,\Space) = \mathcal{A}_{\Space} \,\state(t, \Space) + \mathcal{L}_{\Space} \,\vw(t,\Space)$,
where $\vw(t,\Space)$ is a (spatio-temporal) white noise process and $\mathcal{A}_{\Space}$ a suitable (pseudo-)differential operator \citep[see][]{sarkka2019applied}. By appropriately defining the model matrices and the white noise spectral density function, SPDEs of this form can represent a large class of separable and non-separable \gp models.

When the kernel is separable, this SPDE can be simplified to a finite-dimensional SDE \cite{hartikainen_sparse_st_gps:2011} by marginalising to a finite set of spatial locations, $\SPACE \in \RR^{N_{\Space} \times D_\Space}$, giving,
$\mathrm{d}\state(t) = \Feedback \, \state(t) \, \mathrm{d} t  + \ML \, \mathrm{d}\vbeta(t)$, %
where $\state(t)$ is the Gaussian distributed state at the spatial points $\SPACE$ at time $t$, with dimensionality $d=N_{\Space} d_t$, where $d_t$ is the dimensionality of the state-space model induced by $\kappa_t(\cdot, \cdot)$. $\mathrm{d}\vbeta(t)$ has spectral density $\MQ_c$, and the matrix $\MH$ extracts the function value from the state: $\vf_\timeindex = \MH \state(t_\timeindex)$. $\Feedback$ and $\ML$ are the feedback and noise effect matrices \cite{sarkka2019applied}.
This simplification to an SDE is possible due to the independence between spatial points at time $t$ and all other time steps, given the current state \cite{tebbutt2021combining}. This follows from the fact that for \emph{any} separable kernel, $f(t,\Space)$ and $f(t',\Space')$ are independent given $f(t',\Space)$ \cite{ohagan1998markov}.
For a step size $\Delta_\timeindex = t_{\timeindex+1} - t_{\timeindex}$, the discrete-time model matrices are,
\begin{equation}
		\MA_{\timeindex} = \MPhi( \Feedback \Delta_\timeindex), \qquad
		\MQ_{\timeindex} =  \mint_{0}^{\Delta_\timeindex} \MPhi(\Delta_\timeindex-\tau)\,\ML\, \MQ_c\, \ML^{\top}\, \MPhi(\Delta_\timeindex-\tau)^{\top} \dd \tau ,
\end{equation}
where $\MPhi(\cdot)$ is the matrix exponential. The resulting discrete model is,
\begin{equation} \label{eq:st-statespace}
  \state(t_{\timeindex+1}) = \MA_{\timeindex} \, \state(t_{\timeindex}) + \vq_n,  \qquad\quad
		\MY_\timeindex \mid \state(t_\timeindex) \sim p(\MY_\timeindex \mid \MH \, \state(t_\timeindex)),
\end{equation}
where $\vq_\timeindex \sim \N(\bm{0}, \MQ_\timeindex)$. %
If $p(\MY_\timeindex \mid \MH \, \state(t_\timeindex))$ is Gaussian then Kalman smoothing algorithms can be employed to perform inference in \cref{eq:st-statespace} in $\bigO(N_td^3)=\bigO(N_t N_{\Space}^3 d_t^3)$.

\paragraph{Markovian GPs with Spatial Sparsity} Sparse \gps re-define the \gp prior over a smaller set of $M$ \emph{inducing points}: let $\vu =f(\MZ) \in \R^{\M \times 1}$ be the inducing variables at inducing locations $\MZ \in \R^{\M \times D}$, then the augmented prior is $p(\vf, \vu) = p(\vf \mid \vu) p(\vu)$, where $p(\vu) = \N(\vu \mid \bm{0}, \Kzz)$,
and with Gaussian conditional $p(\vf \mid \vu)$. 
If the inducing points are placed on a spatio-temporal grid, with $\MZ_\Space \in \R^{M_\Space \times D_\Space}$ being the spatial inducing locations, the conditional $p(\vf \mid \vu)$ can be simplified to (see \cref{sec_app:sparse_kronecker_decomposition}):
\begin{equation}
	p(\vf \mid \vu) = \N\big(\vf \mid \big[ \MI \kron (\sKxx \kron \siKzz) \big]\vu, \tKxx \kron \widetilde{\MQ}_\Space\big),
\label{eqn:kronecker_conditional}
\end{equation}
where $\widetilde{\MQ}_\Space = \sKxx - \sKxz  \siKzz  \sKzx$ (see \cref{sec:app:nomencalture} for notational details). %
The \emph{fully independent training conditional} (FITC) method \citep{unifying_sparse_gps:2005} approximates the full conditional covariance with its diagonal, leading to the following convenient property: %
$q_{\FITC}(\vf \mid \vu) = \prod^{N_t}_{\timeindex=1} q_{\FITC}(\vf_\timeindex \mid \vu) = \prod^{N_t}_{\timeindex=1} q_{\FITC}(\vf_\timeindex \mid \vu_\timeindex)$,
where the last equality holds because $\MI \kron (\sKxx \kron \siKzz)$ is block diagonal. This factorisation across time allows the model to be cast into the state-space form of \cref{eq:st-statespace}, but where the state $\state(t)$ is defined over the reduced set of spatial inducing points \citep{hartikainen_sparse_st_gps:2011}. %
Inference can be performed in $\BO(\Nt \Ms^3 d^3_t)$. %

\subsection{Sparse Variational GPs} \label{sec:variational_gps}

To perform approximate inference in the presence of sparsity or non-Gaussian likelihoods, variational methods cast inference as optimisation through minimisation of the Kullback--Leibler divergence (\KLD) from the true posterior to the approximate posterior \citep{blei_vi_review:2017}. Although direct computation of the \KLD is intractable, it can be rewritten as the maximisation of the evidence lower bound (\ELBO).

Unlike \FITC, the sparse variational GP (\SVGP, \cite{titsias_sparse_gp:2009}) does not approximate the conditional $p(\vf \mid \vu)$, but instead approximates the posterior as $q(\vf, \vu) = p(\vf \mid \vu) \, q(\vu)$, where $q(\vu) = \N(\vu \mid \vm, \postcov)$ is a Gaussian whose parameters are to be optimised. The \SVGP \ELBO is:
\begin{align}
	&\LL_{\SVGP} = \E_{q(\vu)} \left[ \E_{q(\vf \mid \vu)} \left[ \log p(\MY \mid \vf) \right] \right] - \KL \left[ q(\vu)\, \| \,p(\vu) \right] , %
\end{align}
which can be computed in $\BO(NM^2 + M^3)$. \SVGP has many benefits over methods such as \FITC, including: 
non-Gaussian likelihoods can be handled through quadrature or Monte-Carlo approximations \cite{hensman_classification:2015, krauth_autogp:2017}, it is applicable to big data through stochastic VI and mini-batching \cite{hensman_gp_for_big_data:2013}, and the inducing locations are `variationally protected' and hence prevent overfitting \cite{bauer_fitc_vfe:2016}.

\paragraph{Natural Gradients}
Natural gradient descent calculates gradients in \textit{distribution} space rather than \textit{parameter} space, and has been shown to improve inference time and quality for variational \gps \citep{hensman_gp_for_big_data:2013, salimbeni_natural_gradients:2018}. %
A natural descent direction is obtained by scaling the standard gradient by the inverse of the Fisher information matrix, $\E_{q(\cdot)} \left[ \nabla^{2} \log q(\cdot ) \right]$ \citep{amari_natgrad:1996}.
For a Gaussian approximate posterior, the natural gradient of target $\LL$ with respect to the natural parameters $\natp$ can be calculated without directly forming the Hessian, since it can be shown to be equivalent to the gradient with respect to the mean parameters $\meanp=[\vm, \vm \vm^\T + \postcov]$ \cite{hensman_fast_vi:2012}. The natural parameter update, with learning rate $\beta$, becomes,
\begin{equation}
	\natp \leftarrow \natp + \beta \, \diff{\LL}{\meanp} .
	\label{eq:natgrad_hensman}
\end{equation}
To update the approximation posterior, $\natp$ can be simply transformed to the moment parameterisation $[\vm, \postcov]$. A table of mappings between the various parametrisations is given in \cref{sec_app:exp_family}. %

\paragraph{\CVI and the Approximate Likelihood}
\citet{khan_cvi:2017} show that when the prior and approximate posterior are conjugate (as in the \gp case), further elegant properties of exponential family distributions mean that \cref{eq:natgrad_hensman} is equivalent to a two step Bayesian update:
\begin{equation}\label{eq:cvi_update}
	\apxnatp \leftarrow (1-\beta) \, \apxnatp + \beta \, \diff{\,\E_{q(\vf)} [\log p(\MY \mid \vf)]}{\meanp} \, , \qquad \qquad %
	\natp \leftarrow \priornatp + \apxnatp \, , %
\end{equation}
where $\priornatp$ are the natural parameters of the prior and $\apxnatp$ are the natural parameters of the likelihood contribution. Letting $\vg(\cdot) = \diff{\E_q [\log p(\MY \mid \vf)]}{\cdot}$, the gradients can be computed in terms of the mean and covariance via the chain rule: $\gradmu = \big[ \gradm - 2 \, \grads \, \vm, \, \grads \big]$. %
\cref{eq:cvi_update} shows that, since the prior parameters $\priornatp$ are known, natural gradient variational inference is completely characterised by updates to an \emph{approximate likelihood}, which we denote $\N(\apxy \mid \vf, \apxv)$, parameterised by covariance $\apxv=(-2\apxnatp^{(2)})^{-1}$ and mean $\apxy=\apxv \apxnatp^{(1)}$ (see \cref{sec:app:nomencalture}). The approximate posterior then has the form,%
\begin{equation}
	q(\vf) = \frac{\N(\apxy \mid \vf, \apxv) \, p(\vf)}{\intg{\N(\apxy \mid \vf, \apxv) \, p(\vf)}{\vf}{}{}} . 
\label{eq:approx_posterior_after_update}
\end{equation}
Computing $q(\vf)$ amounts to performing conjugate \gp regression with the model prior and the approximation likelihood. This approach is called conjugate-computation variational inference (CVI, \cite{khan_cvi:2017}). To re-emphasise that the \CVI updates compute the exact same quantity as \cref{eq:natgrad_hensman}, we provide an alternative derivation in \cref{sec_app:cvi_natgrad_equal} by directly applying the chain rule to \cref{eq:natgrad_hensman}.%

\section{Spatio-Temporal Variational Gaussian Processes}
\label{sec:st_cvi}

In this section we introduce a spatio-temporal \VGP that has linear complexity with respect to time whilst obtaining the identical variational posterior as the standard \VGP. We will then go on to derive this method's sparse variant, which gives the same posterior as \SVGP when the inducing points are set similarly (\ie, on a spatio-temporal grid), but is capable of scaling to much larger values of $M$. %

\subsection{The Spatio-Temporal VGP ELBO} \label{sec:stvgp-elbo}

We first derive our proposed spatio-temporal VGP ELBO. We do this by exploiting the form of the approximate posterior after a natural gradient step in order to write the ELBO as a sum of three terms, each of which can be efficiently computed through filtering and smoothing.
As shown in \cref{sec:variational_gps}, after a natural gradient step, the approximate posterior $q(\vf) \propto \N(\apxy \mid \vf, \apxv) \, p(\vf)$ decomposes as a Bayesian update applied to the model prior with an approximate likelihood. Following \citet{chang_fast_vi:2020} we substitute \cref{eq:approx_posterior_after_update} into the VGP \ELBO: 
\begin{align}\label{eq:vgp_elbo} 
	\hspace{-0.5em}\LL_\VGP &= \E_{q(\vf)} \left[ \log \frac{p(\MY \mid \vf)\,p(\vf)}{q(\vf)}\right] = \E_{q(\vf)} \left[ \log \frac{p(\MY \mid \vf) \, \cancel{p(\vf)} \int \N(\apxy \mid \vf, \apxv) \, p(\vf) \,\mathrm{d}\vf}{ \N(\apxy \mid \vf, \apxv) \, \cancel{p(\vf)}}\right]  \\ 
	&= \sum_{\timeindex=1}^{N_t} \sum_{\spaceindex=1}^{N_\Space} \E_{q(\vf_{\timeindex, \spaceindex})} \big[ \log p(\MY_{\timeindex, \spaceindex} \mid \vf_{\timeindex, \spaceindex}) \big] - \E_{q(\vf)} \big[ \log \N (\apxy \mid \vf, \apxv) \big] + \log \E_{p(\vf)} \big[\N(\apxy \mid \vf, \apxv) \big] .  \nonumber
\end{align}
The first term is the expected log likelihood, the second is the expected log \emph{approximate likelihood}, and the final term is the log marginal likelihood of the approximation posterior, $\log p(\apxy) = \log \E_{p(\vf)} \big[\N(\apxy \mid \vf, \apxv) \big]$. Na\"ively evaluating $\LL_\VGP$ requires $\BO(N^3)$ computation for both the update to $q(\vf)$ and the marginal likelihood. We now show how to compute this with linear scaling in $N_t$.

 We observe that after a natural gradient update, $\apxv$, the approximate likelihood covariance, has the same form as the gradient $\vg(\postcov)$ because, as seen in \cref{eq:cvi_update}, $\apxnatp$ is only additively updated by $\gradmu$. Since the expected likelihood, $\E_{q(\vf)} [\log p(\MY \mid \vf)]$, factorises across observations, $\vg(\postcov)$ is diagonal and hence so is $\apxv$. The approximate likelihood therefore factorises in the same way as the true one:
\begin{equation} \label{eq:approx-lik-factorise}
	\log \N (\apxy \mid \vf, \apxv) = \sum^{N_t}_{\timeindex=1} \sum^{N_\Space}_{\spaceindex=1} \log \N (\apxy_{\timeindex,\spaceindex} \mid \vf_{\timeindex,\spaceindex}, \apxv_{\timeindex,\spaceindex}).
\end{equation}
We now turn our attention to computing the posterior and the marginal likelihood. Recall that if the kernel is separable between time and space, $\kappa(t,\Space,t',\Space') = \kappa_t(t,t')\,\kappa_{\Space}(\Space, \Space')$, then the \gp prior can be re-written as \cref{eq:st-statespace}. %
This separability property further results in the state-space model matrices having a convenient Kronecker structure, %
\begin{equation} \label{eq:st-model-matrices}
	\state(t_{\timeindex+1}) = \left[\MI_{N_{\Space}}\otimes \MA_\timeindex^{(t)}\right] \, \state(t_{\timeindex}) + \vq_n \, , \qquad \qquad
	\apxy_\timeindex \mid \state(t_\timeindex) \sim p(\apxy_\timeindex \mid \MH \, \state(t_\timeindex)), %
\end{equation}
where $\vq_n \sim \N(\bm{0}, \SpaceCov \otimes \MQ_{\timeindex}^{(t)})$ and $\MH=\MI_{N_{\Space}} \otimes \MH^{(t)}$. Here $\MA_\timeindex^{(t)} \in \RR^{d_t\times d_t}$, $\MQ_{\timeindex}^{(t)} \in \RR^{d_t\times d_t}$, and $\MH^{(t)} \in \RR^{1 \times d_t}$ are the transition matrix, process noise covariance, and measurement model of the SDE (see \cref{sec:state-space-spatio-temporal}) induced by the kernel $\kappa_t(\cdot,\cdot)$, respectively. 

Because the \gp prior is Markov and the approximate likelihood factorises across time, the approximate \gp posterior is also Markov \cite{tebbutt2021combining}. Hence marginals $q(\vf_\timeindex)$ can be computed through linear filtering and smoothing applied to \cref{eq:st-model-matrices}. %
Furthermore, the marginal likelihood of a linear Gaussian state-space model, ${p(\apxy) = p(\apxy_1)\,\prod_{n=2}^{N_t} p(\apxy_\timeindex \mid \apxy_{1:\timeindex-1})}$, can be computed sequentially by running the forward filter, since $p(\apxy_\timeindex \mid \apxy_{1:\timeindex-1}) = \mint p(\apxy_\timeindex \mid \MH\state(t_\timeindex)) \, p(\state(t_\timeindex) \mid \apxy_{1:\timeindex-1}) \dd\state(t_\timeindex)$, where $p(\state(t_\timeindex) \mid \apxy_{1:\timeindex-1})$ is the predictive filtering distribution. By combining all of the above properties we can now write the ELBO as,
\begin{align}\label{eq:st_cvi} 
	\hspace{-0.5em}\LL_\STVGP &= \sum^{N_t}_{\timeindex=1} \sum^{N_\Space}_{\spaceindex=1} \E_{q(\vf_{\timeindex, \spaceindex})} \big[ \log p(\MY_{\timeindex, \spaceindex} \mid \vf_{\timeindex, \spaceindex}) \big] - \sum^{N_t}_{\timeindex=1} \sum^{N_\Space}_{\spaceindex=1} \E_{q(\vf_{\timeindex,\spaceindex})}  \big[ \log \N (\apxy_{\timeindex,\spaceindex} \mid \vf_{\timeindex,\spaceindex}, \apxv_{\timeindex,\spaceindex}) \big] \nonumber \\
	& \quad\quad + \sum_{\timeindex=1}^{N_t} \log \E_{p(\state(t_\timeindex) \mid \apxy_{1:\timeindex-1})} \big[\N(\apxy_\timeindex \mid \MH \state(t_\timeindex), \apxv_\timeindex) \big] .
\end{align}
This ELBO can be computed with linear scaling in $N_t$: $\bigO(N_tN_\Space^3d_t^3)$. We now show that the natural gradient step for updating the parameters of $\N(\apxy \mid \vf, \apxv)$ can be computed with the same complexity. %

\subsection{Efficient Natural Gradient Updates} 
As discussed in \cref{sec:variational_gps}, a \emph{natural gradient} update to the posterior, $q(\vf) \propto p(\vf) \, \N(\apxy \mid \vf, \apxv)$, has superior convergence properties to gradient descent, and is completely characterised by an update to the approximate likelihood, $\N(\apxy \mid \vf, \apxv)$, whose mean and covariance are the free parameters of the model, and implicitly define the same \textit{variational} parameters as \VGP. Since the likelihood factorises across the data points, these updates only require computation of the marginal distribution $q(\vf_{\timeindex,\spaceindex})$ to obtain $\E_{q(\vf_{\timeindex,\spaceindex})} [\log p(\MY_{\timeindex,\spaceindex} \mid \vf_{\timeindex,\spaceindex})]$ and its gradients.

As we have shown, computation of the marginal posterior amounts to smoothing over the state, $\state \sim \N(\state \mid \bar{\bm{m}}, \bar{\MP})$, with the model in \cref{eq:st-model-matrices}. The time marginals are given by applying the measurement model to the state: $q(\vf_\timeindex)=\N(\vf_\timeindex \mid \vm_\timeindex= \MH \bar{\bm{m}}_\timeindex, \postcov_\timeindex= \MH \bar{\MP}_\timeindex \MH^\T)$ after which
$q(\vf_{\timeindex,\spaceindex})=\int\!q(\vf_{\timeindex}) \intd{\vf_{\timeindex, \neq\spaceindex}}$ can then be obtained by integrating out the other spatial points.
Given the marginal, \cref{eq:cvi_update} can be used to give the new likelihood parameters $\apxy$ and $\apxv$. The full learning algorithm iterates this process alternately with a hyperparameter update via gradient descent applied to the ELBO, \cref{eq:st_cvi}, and has computational complexity $\bigO(N_tN_\Space^3d_t^3)$. We call this method the \emph{spatio-temporal variational GP} (\STVGP).

\begin{figure}[!t]
\begin{minipage}[t]{0.49\textwidth}
\begin{algorithm}[H]%
	\footnotesize
	\caption{Spatio-temporal sparse VGP}
	\label{alg:sparse_st_cvi}
	\begin{algorithmic}
		\STATE {\bfseries Input:} Data:$\{\gridX, \MY\}$, Initial params.:$\{\apxy, \apxv\}$, Learning rates:$\{\beta, \rho\} \!\!\!\!$
		\WHILE{\ELBO ($\LL$) not converged}
		\STATE \COMMENT{CVI natural gradient step:}
		\STATE $q(\vu), \, \ell =$ \cref{alg:sparse-filter}$(\apxy, \apxv)$
		\STATE $\mathcal{E} = \E_{q(\vu)} [\E_{p(\vf\mid\vu)} [\log [p(\MY \mid \vf)]]$  %
		\STATE $\apxnatp = (1-\beta) \apxnatp + \beta \frac{\partial \mathcal{E}}{\partial \vmu}$ %
		\STATE $\apxv= ( -2 \apxnatp^{(2)} )^{-1}, \quad \apxy=\apxv \apxnatp^{(1)}$
		\STATE \COMMENT{Hyperparameter gradient step:}  
		\STATE $q(\vu), \, \ell =$ \cref{alg:sparse-filter}$(\apxy, \apxv)$
		\STATE $\mathcal{E} = \E_{q(\vu)} [\E_{p(\vf\mid\vu)} [\log [p(\MY \mid \vf)]]$ 
		\STATE $\LL=\mathcal{E} - \E_{q(\vu)} [ \log \N (\apxy \mid \vu, \apxv) ] + \ell$ \hfill \COMMENT{ELBO} 
		\STATE $\vtheta = \vtheta + \rho \diff{\LL}{\vtheta} $ %
		\ENDWHILE
	\end{algorithmic}
\end{algorithm}
\end{minipage}
\hfill
\begin{minipage}[t]{0.49\textwidth}
\begin{algorithm}[H]%
	\footnotesize
	\caption{Sparse spatio-temporal smoothing}
	\label{alg:sparse-filter}
	\begin{algorithmic}
		\STATE {\bfseries Input:} $\!$Likelihood:$\{\apxy, \apxv\}$, Space prior:$\{\MK_{\MZ_\Space\MZ_\Space}^{(\Space)} \}, $
		\STATE Time prior:$\{\MA^{(t)}, \MQ^{(t)}, \MH^{(t)}\} $
		\STATE \COMMENT{Construct model matrices:} 
		\STATE $\MA_{\timeindex} = \MI_{M_{\Space}}\otimes \MA_\timeindex^{(t)}$,
		\STATE $\MQ_{\timeindex} = \MK^{(\Space)}_{\MZ_\Space\MZ_\Space} \otimes\MQ_{\timeindex}^{(t)}$ , 
		\STATE $\MH = \MI_{M_{\Space}} \otimes \MH^{(t)} \!\!\!$\\
		\STATE \COMMENT{Filtering and smoothing:} 
		\IF{using parallel filter / smoother}
		\STATE $q(\vu), \, \ell =$ \cref{alg:parallel-filter}$(\apxy, \apxv, \MA, \MQ, \MH)$
		\ELSE
		\STATE $q(\vu), \, \ell =$ \cref{alg:sequential-filter}$(\apxy, \apxv, \MA, \MQ, \MH)$
		\ENDIF
		\STATE \COMMENT{Return posterior marginals and log likelihood:}
		\RETURN $q(\vu)$, $\ell$ \\[4pt]
	\end{algorithmic}
\end{algorithm}
\end{minipage}
\end{figure}

\subsection{Spatial Sparsity:  from $\bigO(N_tN_\Space^3d_t^3)$ to $\bigO(N_tM_\Space^3d_t^3)$} \label{sec:sparsity} %

We now introduce spatial inducing points, $\MZ_{\Space}$, in order to reduce the effective dimensionality of the state-space model. Whilst we maintain the same notation for consistency, it should be noted that the sparse model no longer requires the data to be on a spatio-temporal grid, only that the inducing points are. In this case, letting $q(\vu)=\N(\vu \mid \vm^{(\vu)}, \postcov^{(\vu)})$ be the sparse variational posterior, the marginal $q(\vf_\timeindex)=\N(\vf \mid \vm_\timeindex, \postcov_\timeindex)$ only depends on $\vm_\timeindex^{(\vu)}, \postcov_\timeindex^{(\vu)}$ due to the conditional independence property for separable kernels discussed in \cref{sec:state-space-spatio-temporal}. We compute the posterior $q(\vu)$ via filtering and smoothing over the state $\state(t)$ in a similar way to \STVGP by setting,
\begin{equation}
	\MA_{\timeindex} = \MI_{M_{\Space}}\otimes \MA_\timeindex^{(t)}, \quad\quad \MQ_\timeindex = \MK^{(\Space)}_{\MZ_\Space\MZ_\Space} \otimes\MQ_{\timeindex}^{(t)} , \quad\quad \MH = \MI_{M_{\Space}} \otimes \MH^{(t)} .
\end{equation} 
\cref{alg:sparse-filter} gives the smoothing algorithm. However, the natural gradient update, \cref{eq:cvi_update}, now becomes,
\begin{equation}
	\apxnatp \leftarrow (1-\beta) \, \apxnatp + \beta \,  \diff{\, \E_{q(\vu)} \left[ \E_{p(\vf \mid \vu)} \left[ \log p(\MY \mid \vf) \right] \right]}{\meanp^{(\vu)}} ,
\end{equation}
which results in $\apxnatp^{(2)}_\timeindex$, and hence also $\apxv_\timeindex$, being a dense matrix (\ie, $\apxv$ is block-diagonal) due to the conditional mapping, $p(\vf_\timeindex \mid \vu_\timeindex)$. Therefore the approximate likelihood for the sparse model factorises across time, but not space (see \cref{appendix:cvi_block_diagional} for details): $\log \N (\apxy \mid \vu, \apxv) =\sum_{\timeindex=1}^{N_t} \log \N (\apxy_\timeindex \mid \vu_\timeindex, \apxv_\timeindex)$.

\paragraph{The Spatio-Temporal Sparse VGP ELBO} Adding inducing points in space is equivalent to placing the inducing points on a spatio-temporal grid (\ie, inducing points exist at all time steps), and hence the variational objective directly follows from $\LL_\SVGP$ using a similar argument to \cref{sec:stvgp-elbo}:
\begin{align}\label{eqn:sparse_simplfied_elbo}
	\LL_\STSVGP &= \E_{q(\vf, \vu)} \!\bigg[ \!\log \frac{p(\MY \mid \vf) \, \cancel{p(\vf \mid \vu)}\,\cancel{p(\vu)} \int \N(\apxy \mid \vu, \apxv) \, p(\vu) \,\mathrm{d}\vu}{ \N(\apxy \mid \vu, \apxv)\,\cancel{p(\vf \mid \vu)}\,\cancel{p(\vu)}}\bigg] \nonumber \\
	&= 
	\sum^{\N_t}_{\timeindex=1} \sum^{\N_\Space}_{\spaceindex=1} \E_{q(\vu_\timeindex)} \big[ \E_{p(\vf_{\timeindex, \spaceindex} \mid \vu_\timeindex)} \big[ \log p(\MY_{\timeindex, \spaceindex} \mid \vf_{\timeindex, \spaceindex}) \big] \big]
	- \sum^{\N_t}_{\timeindex=1} \E_{q(\vu_\timeindex)} \big[ \log \N(\apxy_\timeindex \mid \vu_\timeindex, \apxv_\timeindex) \big] \nonumber \\
	& \quad\quad + \sum_{\timeindex=1}^{N_t} \log \E_{p(\state(t_\timeindex) \mid \apxy_{1:\timeindex-1})} \big[\N(\apxy_\timeindex \mid \MH \state(t_\timeindex), \apxv_\timeindex) \big] , %
\end{align}
where the final term is given by the forward filter. %

\paragraph{Efficient Natural Gradient Updates} The marginal required to compute the ELBO and natural gradient, $q(\vf_{\timeindex,\spaceindex}) = \int \!\! \int p(\vf \mid \vu) \, q(\vu) \intd{\vu} \intd{\vf_{\neq \timeindex, \spaceindex}}= \int \!\! \int p(\vf \mid \vu_\timeindex) \, q(\vu_\timeindex) \intd{\vu_\timeindex} \intd{\vf_{\neq \timeindex, \spaceindex}}$, is the predictive distribution at input $\gridX_{\timeindex, \spaceindex}$ from the posterior $q(\vu)$.
Because the inducing points have only been placed in space, this can be simplified through the Kronecker structure given by the state-space model. As shown in \cref{sec:appendix_kronecker_svgp_marginal}, 
the marginal mean and covariance are, 
\begin{gather}
	\begin{aligned}
		\vm_{\timeindex, \spaceindex} &=  \sK{\SPACE_{\spaceindex}}{\MZ_{\Space}}{\ks} \siK{\MZ_{\Space}}{\ks} \vm_{\timeindex}^{(\vu)} \, , \\ %
		\postcov_{\timeindex, \spaceindex} &= \sK{\SPACE_{\spaceindex}}{\MZ_{\Space}}{\ks} \siK{\MZ_{\Space}}{\ks} \postcov_{\timeindex}^{(\vu)} \siK{\MZ_{\Space}}{\ks}  \sK{\MZ_{\Space}}{\SPACE_{\spaceindex}}{\ks} +
		\sK{\gridX_{\hspace{-0.05em}\timeindex\hspace{-0.05em},\hspace{-0.05em} \spaceindex}}{\gridX_{\hspace{-0.05em}\timeindex\hspace{-0.05em}, \hspace{-0.05em}\spaceindex}}{t} \big( \sK{\SPACE_{\spaceindex}}{\SPACE_{\spaceindex}}{\ks}  -  \sK{\SPACE_{\spaceindex}}{\MZ_{\Space}}{\ks} \siK{\MZ_{\Space}}{\ks}  \sK{\MZ_{\Space}}{\SPACE_{\spaceindex}}{\ks} \big) , %
	\end{aligned}
\end{gather}
where $\vm_\timeindex^{(\vu)}=\MH\bar{\vm}_\timeindex$, $\postcov_\timeindex^{(\vu)}=\MH \bar{\MP}_\timeindex \MH^\T$ are given by filtering and smoothing.
By combining the above properties we see that all the terms required for the natural gradient updates and hyperparameter learning can be computed efficiently in $\bigO(N_t M_\Space^3 d_t^3)$. We call this approach the \emph{spatio-temporal sparse variational GP} (\STSVGP). The full algorithm is given in \cref{alg:sparse_st_cvi}.

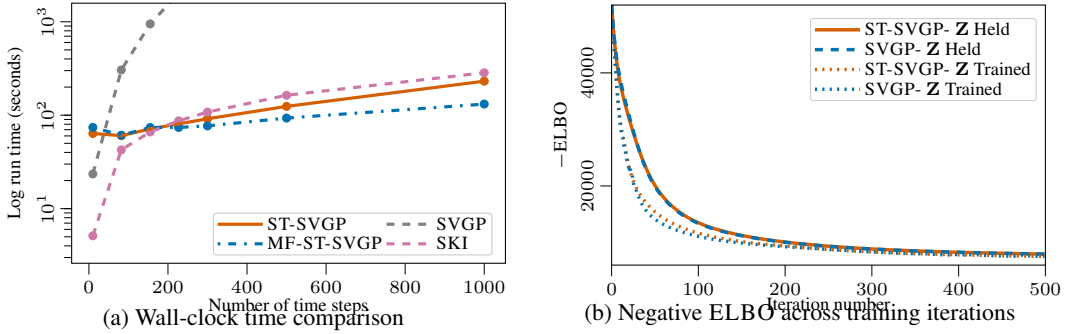
\begin{figure}[t]
	\scriptsize
	\centering
	
	\begin{subfigure}{.48\textwidth}
		\pgfplotsset{yticklabel style={rotate=90}, ylabel style={yshift=0pt}, xlabel style={yshift=3pt},scale only axis,axis on top,clip=true,clip marker paths=true}
		\pgfplotsset{legend style={inner xsep=1pt, inner ysep=1pt, row sep=0pt},legend style={at={(0.98,0.95)},anchor=north east},legend style={rounded corners=1pt}}
		\setlength{\figurewidth}{.86\textwidth}
		\setlength{\figureheight}{.6\figurewidth}
		
		\centering
\begin{tikzpicture}

\definecolor{color0}{rgb}{0.835294117647059,0.368627450980392,0}
\definecolor{color1}{rgb}{0,0.619607843137255,0.450980392156863}
\definecolor{color2}{rgb}{0,0.447058823529412,0.698039215686274}
\definecolor{color3}{rgb}{0.8,0.474509803921569,0.654901960784314}

\begin{axis}[
height=\figureheight,
legend cell align={left},
legend columns=2,
legend style={fill opacity=0.8, draw opacity=1, text opacity=1, at={(0.97,0.03)}, anchor=south east, draw=white!80!black},
log basis y={10},
tick align=outside,
tick pos=left,
width=\figurewidth,
x grid style={white!69.0196078431373!black},
xlabel={Number of time steps},
xmin=-43.6373223322329, xmax=1053.63732233223,
xtick style={color=black},
y grid style={white!69.0196078431373!black},
ylabel={Log run time (seconds)},
ymin=2.60414716846746, ymax=1579.91730472988,
ymode=log,
ytick style={color=black}
]
\addplot [only marks, mark=*, draw=color0, fill=color0, colormap/viridis, forget plot,mark size=1.5pt]
table{%
x                      y
10 63.7115128133446
82 60.4199197488837
155 71.3222890230827
227 80.712445615232
300 91.6146528605372
500 124.215907577239
1000 231.163579195458
};
\addplot [only marks, mark=*, draw=white!50.1960784313725!black, fill=white!50.1960784313725!black, colormap/viridis, forget plot,mark size=1.5pt]
table{%
x                      y
10 23.5228871065192
82 304.993248012569
155 947.628216310032
227 2018.90316602411
300 3504.87859138399
0 0
};
\addplot [only marks, mark=*, draw=color2, fill=color2, colormap/viridis, forget plot,mark size=1.5pt]
table{%
x                      y
10 74.0258787328377
82 61.0914936486632
155 73.8619329852052
227 73.9422508169897
300 76.8499107033014
500 93.1416324221529
1000 131.512665539235
};
\addplot [only marks, mark=*, draw=color3, fill=color3, colormap/viridis, forget plot,mark size=1.5pt]
table{%
x                      y
10 5.12263157707639
82 42.5824939490296
155 66.1947727684863
227 87.0526755646802
300 108.003617039789
500 163.510734979715
1000 284.482517335471
};
\addplot [semithick, color0, line width = 1.2pt]
table {%
10 63.7115128133446
82 60.4199197488837
155 71.3222890230827
227 80.712445615232
300 91.6146528605372
500 124.215907577239
1000 231.163579195458
};
\addlegendentry{\STSVGP}
\addplot [semithick, white!50.1960784313725!black, dashed, line width = 1.2pt]
table {%
10 23.5228871065192
82 304.993248012569
155 947.628216310032
227 2018.90316602411
300 3504.87859138399
500 nan
};
\addlegendentry{\SVGP}
\addplot [semithick, color2, dash pattern=on 1pt off 3pt on 3pt off 3pt, line width = 1.2pt]
table {%
10 74.0258787328377
82 61.0914936486632
155 73.8619329852052
227 73.9422508169897
300 76.8499107033014
500 93.1416324221529
1000 131.512665539235
};
\addlegendentry{\MFSTSVGP}
\addplot [semithick, color3, dashed, line width = 1.2pt]
table {%
10 5.12263157707639
82 42.5824939490296
155 66.1947727684863
227 87.0526755646802
300 108.003617039789
500 163.510734979715
1000 284.482517335471
};
\addlegendentry{\SKI}
\end{axis}

\end{tikzpicture}
 		
		\vspace{-0.4cm}
		\caption{Wall-clock time comparison}
		
		\label{fig:shutters_run_times} 
	\end{subfigure}
	\hfill
	\begin{subfigure}{.48\textwidth}
		\scriptsize
		\centering
		\pgfplotsset{yticklabel style={rotate=90}, ylabel style={yshift=0pt}, xlabel style={yshift=4pt},scale only axis,axis on top,clip=true,clip marker paths=true}
		\pgfplotsset{legend style={inner xsep=1pt, inner ysep=1pt, row sep=0pt},legend style={at={(0.98,0.95)},anchor=north east},legend style={rounded corners=1pt}}
		\pgfplotsset{scaled y ticks=false}
		\setlength{\figurewidth}{.86\textwidth}
		\setlength{\figureheight}{.6\figurewidth}
		
		\centering
\begin{tikzpicture}

\definecolor{color0}{rgb}{0.835294117647059,0.368627450980392,0}
\definecolor{color1}{rgb}{0,0.447058823529412,0.698039215686274}

\begin{axis}[
height=\figureheight,
legend cell align={left},
legend style={fill opacity=0.8, draw opacity=1, text opacity=1, draw=white!80!black},
tick align=outside,
tick pos=left,
width=\figurewidth,
x grid style={white!69.0196078431373!black},
xlabel={Iteration number},
xmin=0, xmax=500,
xtick style={color=black},
y grid style={white!69.0196078431373!black},
ylabel={\(\displaystyle -\)\ELBO},
ymin=6000, ymax=52000,
ytick style={color=black}
]
\addplot [thick, color0, line width = 1.3pt]
table {%
0 52060.3252640152
1 49892.1242668146
2 47850.883609507
3 46010.5004040539
4 44413.5604688586
5 43035.2211722249
6 41820.171823594
7 40738.3300950761
8 39767.9924329383
9 38890.5382458604
10 38089.7880669933
11 37351.1052063402
12 36661.0812024444
13 36008.6844181669
14 35385.6435761267
15 34785.6061221377
16 34203.510037865
17 33635.5453351203
18 33079.1024112424
19 32532.4978879741
20 31994.6856841964
21 31465.0657770071
22 30943.3457587285
23 30429.419678842
24 29923.2864523593
25 29425.0238830719
26 28934.7906730449
27 28452.8188352619
28 27979.3878671031
29 27514.7950979357
30 27059.3353030849
31 26613.2911480138
32 26176.9292767154
33 25750.4972760188
34 25334.2196280774
35 24928.2927204378
36 24532.8796663133
37 24148.1056918791
38 23774.0545923775
39 23410.7664710178
40 23058.2367663564
41 22716.4164656807
42 22385.2133643245
43 22064.4942268821
44 21754.0877114724
45 21453.7879217108
46 21163.3584506726
47 20882.536778545
48 20611.0388836133
49 20348.56392713
50 20094.7988780857
51 19849.4229545813
52 19612.1117741859
53 19382.5411254327
54 19160.3902951246
55 18945.3449097891
56 18737.0992728943
57 18535.3582008881
58 18339.8383796665
59 18150.2692779718
60 17966.3936650996
61 17787.9677871539
62 17614.7612592144
63 17446.5567306471
64 17283.1493780476
65 17124.3462756043
66 16969.9656866964
67 16819.8363138811
68 16673.7965375994
69 16531.6936673298
70 16393.3832228439
71 16258.7282578475
72 16127.5987337196
73 15999.8709473244
74 15875.4270139243
75 15754.154404013
76 15635.9455313123
77 15520.6973881483
78 15408.3112238299
79 15298.6922614074
80 15191.7494482004
81 15087.395235676
82 14985.5453845792
83 14886.1187915993
84 14789.0373342783
85 14694.2257312924
86 14601.6114156436
87 14511.1244186798
88 14422.6972632031
89 14336.26486423
90 14251.7644362261
91 14169.1354058645
92 14088.319329538
93 14009.2598150129
94 13931.9024467344
95 13856.1947143948
96 13782.0859444553
97 13709.5272343726
98 13638.4713893329
99 13568.8728613303
100 13500.6876904552
101 13433.8734482817
102 13368.3891832557
103 13304.1953679991
104 13241.2538484551
105 13179.5277948048
106 13118.9816540907
107 13059.581104488
108 13001.2930111655
109 12944.0853836843
110 12887.9273348792
111 12832.7890411763
112 12778.6417042963
113 12725.4575142994
114 12673.2096139257
115 12621.8720641894
116 12571.4198111843
117 12521.8286540618
118 12473.0752141424
119 12425.1369051225
120 12377.9919043426
121 12331.6191250796
122 12285.998189832
123 12241.1094045632
124 12196.9337338725
125 12153.4527770638
126 12110.6487450797
127 12068.5044382759
128 12027.0032250039
129 11986.1290209779
130 11945.8662693978
131 11906.1999218026
132 11867.1154196303
133 11828.5986764583
134 11790.6360609021
135 11753.2143801487
136 11716.3208641016
137 11679.943150118
138 11644.069268314
139 11608.6876274196
140 11573.7870011635
141 11539.3565151671
142 11505.3856343302
143 11471.8641506902
144 11438.7821717368
145 11406.1301091659
146 11373.8986680555
147 11342.0788364492
148 11310.66187533
149 11279.6393089713
150 11249.0029156503
151 11218.74471871
152 11188.856977956
153 11159.3321813764
154 11130.1630371719
155 11101.342466084
156 11072.8635940101
157 11044.7197448946
158 11016.9044338854
159 10989.411360745
160 10962.2344035077
161 10935.3676123719
162 10908.8052038196
163 10882.5415549536
164 10856.5711980448
165 10830.8888152805
166 10805.4892337075
167 10780.3674203601
168 10755.5184775688
169 10730.9376384408
170 10706.6202625055
171 10682.5618315207
172 10658.7579454311
173 10635.2043184745
174 10611.8967754298
175 10588.8312480015
176 10566.003771335
177 10543.410480658
178 10521.0476080439
179 10498.9114792914
180 10476.9985109161
181 10455.3052072511
182 10433.8281576503
183 10412.5640337928
184 10391.5095870831
185 10370.6616461434
186 10350.0171143957
187 10329.5729677293
188 10309.3262522509
189 10289.2740821139
190 10269.4136374245
191 10249.7421622209
192 10230.2569625235
193 10210.9554044534
194 10191.8349124165
195 10172.8929673505
196 10154.1271050328
197 10135.5349144472
198 10117.1140362067
199 10098.8621610308
200 10080.777028275
201 10062.8564245105
202 10045.0981821527
203 10027.5001781361
204 10010.0603326348
205 9992.77660782528
206 9975.6470066921
207 9958.66957187307
208 9941.84238454354
209 9925.1635633379
210 9908.63126330715
211 9892.24367491116
212 9875.99902304437
213 9859.89556609367
214 9843.93159502749
215 9828.10543251466
216 9812.41543207232
217 9796.85997724157
218 9781.43748079003
219 9766.14638394021
220 9750.98515562301
221 9735.95229175511
222 9721.04631453977
223 9706.2657717899
224 9691.60923627286
225 9677.07530507605
226 9662.66259899267
227 9648.36976192685
228 9634.19546031761
229 9620.13838258081
230 9606.19723856867
231 9592.37075904601
232 9578.65769518293
233 9565.05681806299
234 9551.56691820679
235 9538.18680511002
236 9524.9153067957
237 9511.75126938014
238 9498.69355665201
239 9485.74104966418
240 9472.89264633792
241 9460.14726107891
242 9447.50382440479
243 9434.9612825838
244 9422.51859728418
245 9410.17474523384
246 9397.92871789017
247 9385.77952111947
248 9373.7261748857
249 9361.76771294842
250 9349.90318256932
251 9338.13164422728
252 9326.45217134164
253 9314.86385000329
254 9303.36577871353
255 9291.95706813019
256 9280.63684082101
257 9269.4042310239
258 9258.25838441388
259 9247.19845787653
260 9236.22361928768
261 9225.33304729919
262 9214.52593113054
263 9203.80147036619
264 9193.15887475834
265 9182.59736403507
266 9172.11616771359
267 9161.71452491852
268 9151.39168420495
269 9141.14690338619
270 9130.97944936601
271 9120.88859797537
272 9110.87363381321
273 9100.93385009151
274 9091.06854848427
275 9081.27703898029
276 9071.5586397398
277 9061.91267695458
278 9052.33848471166
279 9042.83540486044
280 9033.402786883
281 9024.03998776768
282 9014.74637188575
283 9005.52131087102
284 8996.3641835024
285 8987.27437558927
286 8978.25127985956
287 8969.2942958505
288 8960.40282980189
289 8951.57629455197
290 8942.81410943551
291 8934.11570018448
292 8925.48049883078
293 8916.9079436113
294 8908.3974788751
295 8899.94855499261
296 8891.56062826689
297 8883.23316084686
298 8874.96562064237
299 8866.75748124115
300 8858.60822182758
301 8850.51732710309
302 8842.48428720837
303 8834.50859764716
304 8826.58975921158
305 8818.72727790915
306 8810.9206648911
307 8803.16943638239
308 8795.47311361291
309 8787.83122275021
310 8780.24329483353
311 8772.70886570911
312 8765.2274759668
313 8757.79867087784
314 8750.42200033394
315 8743.09701878743
316 8735.82328519259
317 8728.60036294804
318 8721.4278198403
319 8714.30522798826
320 8707.23216378878
321 8700.20820786321
322 8693.2329450049
323 8686.30596412764
324 8679.42685821507
325 8672.59522427086
326 8665.81066326982
327 8659.07278010997
328 8652.38118356518
329 8645.73548623888
330 8639.13530451835
331 8632.58025852988
332 8626.06997209467
333 8619.60407268539
334 8613.18219138353
335 8606.80396283736
336 8600.46902522066
337 8594.17702019201
338 8587.92759285478
339 8581.72039171774
340 8575.55506865627
341 8569.43127887417
342 8563.34868086608
343 8557.30693638041
344 8551.3057103829
345 8545.34467102068
346 8539.4234895869
347 8533.5418404858
348 8527.69940119841
349 8521.89585224864
350 8516.13087716993
351 8510.40416247232
352 8504.71539761007
353 8499.06427494962
354 8493.45048973811
355 8487.87374007226
356 8482.33372686771
357 8476.83015382881
358 8471.36272741876
359 8465.93115683018
360 8460.5351539561
361 8455.17443336131
362 8449.84871225409
363 8444.55771045832
364 8439.30115038595
365 8434.07875700985
366 8428.89025783697
367 8423.73538288188
368 8418.61386464066
369 8413.52543806504
370 8408.46984053698
371 8403.44681184351
372 8398.45609415184
373 8393.49743198488
374 8388.57057219701
375 8383.67526395013
376 8378.81125869
377 8373.97831012295
378 8369.17617419277
379 8364.40460905792
380 8359.66337506904
381 8354.95223474669
382 8350.27095275936
383 8345.61929590175
384 8340.99703307333
385 8336.40393525708
386 8331.83977549856
387 8327.30432888517
388 8322.79737252568
389 8318.31868552997
390 8313.86804898901
391 8309.44524595511
392 8305.05006142231
393 8300.68228230709
394 8296.34169742923
395 8292.0280974929
396 8287.74127506799
397 8283.48102457161
398 8279.24714224984
399 8275.03942615963
400 8270.85767615096
401 8266.70169384914
402 8262.57128263735
403 8258.46624763936
404 8254.38639570241
405 8250.33153538033
406 8246.30147691679
407 8242.29603222877
408 8238.31501489023
409 8234.35824011585
410 8230.42552474508
411 8226.51668722629
412 8222.63154760109
413 8218.76992748881
414 8214.93165007125
415 8211.11654007739
416 8207.32442376847
417 8203.5551289231
418 8199.8084848226
419 8196.08432223642
420 8192.38247340778
421 8188.70277203946
422 8185.04505327968
423 8181.40915370818
424 8177.79491132246
425 8174.20216552409
426 8170.63075710524
427 8167.08052823533
428 8163.55132244777
429 8160.04298462693
430 8156.55536099514
431 8153.08829909991
432 8149.64164780127
433 8146.21525725917
434 8142.80897892109
435 8139.42266550976
436 8136.05617101099
437 8132.7093506616
438 8129.38206093755
439 8126.07415954215
440 8122.78550539433
441 8119.51595861715
442 8116.26538052635
443 8113.03363361901
444 8109.82058156239
445 8106.6260891828
446 8103.45002245467
447 8100.29224848964
448 8097.15263552586
449 8094.03105291732
450 8090.92737112332
451 8087.84146169802
452 8084.77319728017
453 8081.72245158287
454 8078.68909938343
455 8075.67301651342
456 8072.67407984867
457 8069.69216729957
458 8066.72715780127
459 8063.77893130411
460 8060.84736876409
461 8057.93235213347
462 8055.03376435143
463 8052.15148933483
464 8049.28541196908
465 8046.43541809911
466 8043.6013945204
467 8040.78322897009
468 8037.98081011826
469 8035.19402755917
470 8032.42277180274
471 8029.66693426594
472 8026.92640726446
473 8024.20108400427
474 8021.4908585734
475 8018.79562593381
476 8016.11528191318
477 8013.44972319698
478 8010.79884732048
479 8008.16255266095
480 8005.54073842981
481 8002.93330466496
482 8000.34015222318
483 7997.76118277252
484 7995.19629878487
485 7992.64540352853
486 7990.10840106091
487 7987.58519622125
488 7985.07569462344
489 7982.57980264892
490 7980.09742743964
491 7977.62847689104
492 7975.17285964523
493 7972.7304850841
494 7970.30126332253
495 7967.88510520176
496 7965.48192228269
497 7963.09162683937
498 7960.71413185244
499 7958.34935100272
};
\addlegendentry{\STSVGP - \ZZ Held}
\addplot [thick, color1, dashed, line width = 1.3pt]
table {%
0 52060.261416598
1 49944.572542786
2 47996.7133964958
3 46270.8395688363
4 44776.4958237336
5 43493.8002647974
6 42364.6224834357
7 41353.3960956757
8 40436.8648694914
9 39596.4269969847
10 38816.9311799845
11 38086.0778707303
12 37393.6304656167
13 36730.9142104327
14 36090.9258475283
15 35468.4248243526
16 34859.602535714
17 34261.6157422528
18 33672.3339027077
19 33090.2902942206
20 32514.6459597852
21 31945.0820354464
22 31381.6624094577
23 30824.7189398851
24 30274.7643243754
25 29732.4146175801
26 29198.3198607998
27 28673.1200629456
28 28157.4345654988
29 27651.8686485738
30 27157.0140800396
31 26673.435072025
32 26201.6470973059
33 25742.099601494
34 25295.1676936242
35 24861.1511385988
36 24440.2764439221
37 24032.6988604779
38 23638.5030584817
39 23257.7026514243
40 22890.2394118348
41 22535.9831297878
42 22194.7328460705
43 21866.2198646883
44 21550.1126513495
45 21246.0235039173
46 20953.5167330527
47 20672.117996873
48 20401.3243767557
49 20140.6147545788
50 19889.4600520878
51 19647.3329199077
52 19413.7165153152
53 19188.1120805215
54 18970.0451208135
55 18759.0700766406
56 18554.7734771261
57 18356.7756466151
58 18164.731104523
59 17978.3278480625
60 17797.2857361763
61 17621.3542022544
62 17450.3095158825
63 17283.9517938468
64 17122.1019322282
65 16964.598598696
66 16811.2953905236
67 16662.0582320154
68 16516.7630567426
69 16375.2937962125
70 16237.5406777055
71 16103.3988198558
72 15972.7671046934
73 15845.5472986744
74 15721.6433920158
75 15600.9611247369
76 15483.4076685788
77 15368.8914358864
78 15257.3219891655
79 15148.6100280289
80 15042.6674333655
81 14939.407351629
82 14838.7443050268
83 14740.5943160263
84 14644.8750369438
85 14551.5058774367
86 14460.4081244828
87 14371.5050509235
88 14284.7220098908
89 14199.98651346
90 14117.2282946973
91 14036.3793529288
92 13957.3739825717
93 13880.1487862594
94 13804.6426732782
95 13730.7968445365
96 13658.5547654147
97 13587.8621279182
98 13518.6668035751
99 13450.9187885088
100 13384.5701420687
101 13319.5749203356
102 13255.8891057309
103 13193.4705338641
104 13132.2788186463
105 13072.2752765883
106 13013.4228510939
107 12955.6860374468
108 12899.0308090883
109 12843.4245456806
110 12788.8359633571
111 12735.2350474765
112 12682.5929881159
113 12630.8821184705
114 12580.0758562628
115 12530.1486482112
116 12481.0759175599
117 12432.8340146342
118 12385.4001703523
119 12338.7524525974
120 12292.8697253364
121 12247.7316103546
122 12203.3184514653
123 12159.6112810482
124 12116.5917887664
125 12074.2422923121
126 12032.5457100344
127 11991.4855353036
128 11951.0458124777
129 11911.2111143375
130 11871.9665208692
131 11833.2975992793
132 11795.1903851343
133 11757.6313645298
134 11720.6074571968
135 11684.1060004664
136 11648.1147340175
137 11612.6217853432
138 11577.6156558755
139 11543.0852077169
140 11509.0196509311
141 11475.4085313525
142 11442.2417188768
143 11409.5093962025
144 11377.202047993
145 11345.3104504358
146 11313.825661177
147 11282.7390096106
148 11252.0420875078
149 11221.7267399691
150 11191.7850566868
151 11162.2093635063
152 11132.9922142744
153 11104.1263829651
154 11075.6048560736
155 11047.4208252709
156 11019.5676803092
157 10992.0390021729
158 10964.8285564657
159 10937.9302870288
160 10911.3383097825
161 10885.0469067844
162 10859.0505204992
163 10833.3437482723
164 10807.9213370018
165 10782.7781780032
166 10757.9093020599
167 10733.3098746542
168 10708.9751913737
169 10684.9006734851
170 10661.0818636728
171 10637.5144219345
172 10614.1941216292
173 10591.1168456722
174 10568.2785828721
175 10545.6754244044
176 10523.3035604165
177 10501.1592767602
178 10479.2389518456
179 10457.5390536129
180 10436.0561366176
181 10414.7868392234
182 10393.7278809011
183 10372.8760596262
184 10352.2282493748
185 10331.7813977111
186 10311.5325234644
187 10291.4787144915
188 10271.6171255218
189 10251.9449760805
190 10232.4595484881
191 10213.1581859328
192 10194.0382906122
193 10175.0973219427
194 10156.3327948325
195 10137.7422780174
196 10119.3233924549
197 10101.0738097766
198 10082.9912507939
199 10065.0734840575
200 10047.3183244665
201 10029.723631927
202 10012.2873100561
203 9995.00730493195
204 9977.88160388584
205 9960.90823433623
206 9944.08526266249
207 9927.41079311664
208 9910.88296677218
209 9894.49996050805
210 9878.25998602679
211 9862.1612889053
212 9846.20214767719
213 9830.38087294538
214 9814.69580652404
215 9799.14532060837
216 9783.72781697169
217 9768.44172618847
218 9753.28550688243
219 9738.25764499895
220 9723.35665310069
221 9708.5810696857
222 9693.92945852717
223 9679.40040803417
224 9664.99253063231
225 9650.70446216402
226 9636.53486130749
227 9622.48240901359
228 9608.54580796044
229 9594.72378202458
230 9581.01507576852
231 9567.41845394392
232 9553.93270100999
233 9540.55662066636
234 9527.28903540018
235 9514.12878604688
236 9501.07473136401
237 9488.12574761788
238 9475.28072818249
239 9462.5385831503
240 9449.89823895459
241 9437.35863800277
242 9424.91873832064
243 9412.57751320681
244 9400.3339508973
245 9388.18705423981
246 9376.13584037735
247 9364.17934044102
248 9352.3165992515
249 9340.54667502905
250 9328.8686391118
251 9317.28157568191
252 9305.78458149949
253 9294.37676564394
254 9283.0572492625
255 9271.82516532581
256 9260.67965839019
257 9249.6198843665
258 9238.64501029531
259 9227.75421412817
260 9216.9466845149
261 9206.22162059658
262 9195.57823180408
263 9185.01573766204
264 9174.53336759805
265 9164.1303607569
266 9153.80596581969
267 9143.55944082775
268 9133.39005301116
269 9123.29707862163
270 9113.27980276983
271 9103.33751926684
272 9093.46953046967
273 9083.67514713065
274 9073.9536882508
275 9064.3044809367
276 9054.72686026105
277 9045.22016912665
278 9035.78375813383
279 9026.41698545108
280 9017.11921668878
281 9007.88982477618
282 8998.72818984117
283 8989.63369909301
284 8980.60574670796
285 8971.64373371743
286 8962.74706789892
287 8953.91516366944
288 8945.14744198148
289 8936.44333022126
290 8927.80226210952
291 8919.22367760432
292 8910.70702280627
293 8902.25174986578
294 8893.85731689239
295 8885.52318786614
296 8877.24883255088
297 8869.0337264095
298 8860.87735052095
299 8852.77919149912
300 8844.73874141339
301 8836.75549771095
302 8828.82896314068
303 8820.95864567867
304 8813.14405845525
305 8805.38471968363
306 8797.68015258984
307 8790.02988534427
308 8782.43345099442
309 8774.89038739915
310 8767.40023716413
311 8759.9625475786
312 8752.57687055339
313 8745.2427625601
314 8737.95978457149
315 8730.72750200297
316 8723.54548465523
317 8716.41330665795
318 8709.33054641452
319 8702.29678654783
320 8695.31161384706
321 8688.37461921533
322 8681.48539761849
323 8674.64354803465
324 8667.84867340469
325 8661.10038058362
326 8654.39828029281
327 8647.74198707307
328 8641.13111923842
329 8634.56529883081
330 8628.0441515755
331 8621.56730683723
332 8615.1343975771
333 8608.74506031017
334 8602.39893506378
335 8596.09566533654
336 8589.8348980579
337 8583.61628354852
338 8577.43947548114
339 8571.30413084212
340 8565.20990989359
341 8559.15647613616
342 8553.14349627223
343 8547.17064016983
344 8541.23758082707
345 8535.34399433703
346 8529.48955985327
347 8523.67395955578
348 8517.89687861749
349 8512.15800517121
350 8506.45703027712
351 8500.79364789067
352 8495.16755483093
353 8489.57845074949
354 8484.02603809967
355 8478.5100221062
356 8473.03011073544
357 8467.58601466581
358 8462.17744725877
359 8456.80412453017
360 8451.46576512196
361 8446.16209027427
362 8440.89282379794
363 8435.65769204729
364 8430.4564238934
365 8425.28875069764
366 8420.1544062855
367 8415.05312692095
368 8409.98465128089
369 8404.94872043014
370 8399.94507779661
371 8394.97346914679
372 8390.0336425617
373 8385.12534841294
374 8380.24833933913
375 8375.40237022269
376 8370.58719816685
377 8365.80258247289
378 8361.0482846178
379 8356.32406823207
380 8351.62969907784
381 8346.96494502725
382 8342.32957604114
383 8337.72336414786
384 8333.1460834225
385 8328.59750996622
386 8324.07742188589
387 8319.58559927403
388 8315.12182418882
389 8310.6858806345
390 8306.27755454194
391 8301.8966337494
392 8297.54290798354
393 8293.21616884067
394 8288.91620976819
395 8284.64282604624
396 8280.39581476953
397 8276.17497482945
398 8271.98010689632
399 8267.81101340184
400 8263.66749852176
401 8259.54936815874
402 8255.45642992538
403 8251.38849312747
404 8247.34536874738
405 8243.32686942768
406 8239.33280945489
407 8235.36300474349
408 8231.41727281999
409 8227.49543280726
410 8223.59730540898
411 8219.72271289432
412 8215.87147908273
413 8212.04342932887
414 8208.23839050781
415 8204.45619100026
416 8200.69666067804
417 8196.9596308897
418 8193.24493444625
419 8189.55240560706
420 8185.88188006595
421 8182.23319493735
422 8178.60618874271
423 8175.00070139692
424 8171.416574195
425 8167.85364979883
426 8164.3117722241
427 8160.79078682736
428 8157.29054029319
429 8153.81088062149
430 8150.35165711499
431 8146.9127203668
432 8143.49392224812
433 8140.09511589607
434 8136.71615570168
435 8133.35689729794
436 8130.01719754805
437 8126.69691453371
438 8123.39590754361
439 8120.11403706197
440 8116.85116475724
441 8113.60715347092
442 8110.38186720643
443 8107.17517111819
444 8103.98693150075
445 8100.81701577801
446 8097.66529249265
447 8094.53163129552
448 8091.41590293528
449 8088.31797924808
450 8085.23773314731
451 8082.1750386135
452 8079.12977068438
453 8076.10180544487
454 8073.09102001738
455 8070.09729255202
456 8067.12050221702
457 8064.16052918928
458 8061.21725464484
459 8058.29056074968
460 8055.3803306504
461 8052.48644846512
462 8049.60879927447
463 8046.74726911259
464 8043.90174495828
465 8041.07211472625
466 8038.25826725843
467 8035.46009231534
468 8032.67748056761
469 8029.91032358754
470 8027.15851384075
471 8024.42194467796
472 8021.70051032674
473 8018.99410588349
474 8016.30262730538
475 8013.62597140241
476 8010.96403582957
477 8008.31671907909
478 8005.68392047268
479 8003.06554015391
480 8000.46147908072
481 7997.8716390179
482 7995.29592252967
483 7992.7342329724
484 7990.18647448735
485 7987.65255199343
486 7985.13237118016
487 7982.62583850058
488 7980.13286116431
489 7977.65334713061
490 7975.18720510156
491 7972.73434451534
492 7970.29467553943
493 7967.86810906407
494 7965.45455669563
495 7963.05393075012
496 7960.66614424677
497 7958.29111090162
498 7955.92874512123
499 7953.5789619964
};
\addlegendentry{\SVGP - \ZZ Held}
\addplot [thick, color0, dotted, line width = 1.3pt]
table {%
0 52060.3252640152
1 49629.3087637904
2 47097.7932472182
3 44553.7496393761
4 42082.3506744105
5 39730.4114672273
6 37546.394733318
7 35553.5094851734
8 33834.4158715124
9 32483.1997781488
10 31409.3082382305
11 30465.3643529303
12 29601.366957393
13 28793.049357206
14 28001.7706818818
15 27208.9081571122
16 26416.8780783678
17 25634.6469049258
18 24869.7497612927
19 24134.577354376
20 23453.7181311799
21 22844.8931515717
22 22300.0907595454
23 21796.2064002811
24 21320.0241462749
25 20874.6212621067
26 20467.9739285881
27 20104.4987882271
28 19782.2237910778
29 19492.5106599778
30 19220.6245547203
31 18953.2930945939
32 18687.7225556889
33 18430.9576934138
34 18190.7825053508
35 17968.2067037852
36 17757.6936768063
37 17553.2882559125
38 17353.4258506871
39 17160.4751192587
40 16977.1867695317
41 16803.8163941019
42 16637.9936558536
43 16476.79225284
44 16317.810581852
45 16159.8696926909
46 16003.3635420518
47 15849.6238397412
48 15700.0257121054
49 15555.3511150429
50 15416.1370432866
51 15282.7789957059
52 15155.0664938213
53 15031.7220509463
54 14911.4743045093
55 14794.1912306008
56 14680.4550656818
57 14571.119366878
58 14466.6900371179
59 14366.7410012938
60 14270.1456166723
61 14175.8915647257
62 14083.4858779777
63 13992.8406220125
64 13904.0438714367
65 13817.2129771599
66 13732.4023192942
67 13649.5622443825
68 13568.55349179
69 13489.1941760669
70 13411.3118333025
71 13334.7828363014
72 13259.5527535817
73 13185.6372486143
74 13113.1072803234
75 13042.053350962
76 12972.5407974986
77 12904.5892246263
78 12838.1607597299
79 12773.1671302672
80 12709.4978965439
81 12647.056651713
82 12585.7852046444
83 12525.6636929149
84 12466.6923804983
85 12408.8707517831
86 12352.1842255157
87 12296.6010878787
88 12242.0774235586
89 12188.5657399717
90 12136.0232790049
91 12084.4182485499
92 12033.7315479301
93 11983.9515622619
94 11935.0678640467
95 11887.0670576793
96 11839.9298180564
97 11793.6318465081
98 11748.1468781111
99 11703.4499961429
100 11659.5199122479
101 11616.3396180594
102 11573.8955464346
103 11532.1755593448
104 11491.1664528406
105 11450.8526975329
106 11411.2161735137
107 11372.2367785627
108 11333.8942833092
109 11296.1704972223
110 11259.0502899537
111 11222.5215449397
112 11186.5744497759
113 11151.200569072
114 11116.392112856
115 11082.1414421163
116 11048.4407788496
117 11015.2820608937
118 10982.6568889972
119 10950.5565308985
120 10918.9719378585
121 10887.8937311009
122 10857.3119084005
123 10827.2143036384
124 10797.588538723
125 10768.4232659136
126 10739.7078426955
127 10711.4321319655
128 10683.5863166312
129 10656.1607619087
130 10629.1459307816
131 10602.5323648473
132 10576.3107215582
133 10550.4718387895
134 10525.0067898638
135 10499.9069059062
136 10475.163771293
137 10450.7692044735
138 10426.7152362399
139 10402.9941114698
140 10379.598314188
141 10356.5205493593
142 10333.753782227
143 10311.2912267735
144 10289.1263229907
145 10267.2527109563
146 10245.6641994279
147 10224.3547411611
148 10203.3184171406
149 10182.549428593
150 10162.0420954725
151 10141.7908580433
152 10121.7902714401
153 10102.035003644
154 10082.519834263
155 10063.2396543433
156 10044.1894677464
157 10025.3643945562
158 10006.7596757803
159 9988.37067906222
160 9970.19290465383
161 9952.22197500648
162 9934.4536337749
163 9916.88377979153
164 9899.50848098437
165 9882.32398533994
166 9865.3267281134
167 9848.51333401322
168 9831.88063278854
169 9815.42566799423
170 9799.14570123254
171 9783.03820698125
172 9767.10088096427
173 9751.33163858161
174 9735.72860656217
175 9720.29010891477
176 9705.01464745349
177 9689.9008708932
178 9674.94753127912
179 9660.15327178274
180 9645.51661241336
181 9631.03606965187
182 9616.71012505053
183 9602.53718411284
184 9588.51553069462
185 9574.6432775622
186 9560.91831668778
187 9547.33826396322
188 9533.90041315785
189 9520.6016923117
190 9507.43861627546
191 9494.40722851037
192 9481.50290677527
193 9468.72020035169
194 9456.05301016785
195 9443.49445485866
196 9431.0366605895
197 9418.67053176126
198 9406.38553982086
199 9394.16959457401
200 9382.00908632391
201 9369.88916393364
202 9357.79307513672
203 9345.69934244092
204 9333.58394235063
205 9321.42525010227
206 9309.20706810365
207 9296.91714986742
208 9284.54990528334
209 9272.11416588643
210 9259.63394450646
211 9247.11941863805
212 9234.57643931634
213 9222.02145592749
214 9209.47146594363
215 9196.95009787818
216 9184.49758693321
217 9172.1264918041
218 9159.83130108111
219 9147.60688672448
220 9135.44713295796
221 9123.33616855307
222 9111.22405336609
223 9099.02753439456
224 9086.70364237981
225 9074.27865529155
226 9061.81077670445
227 9049.3438239995
228 9036.88938348244
229 9024.43636780584
230 9011.9732387808
231 8999.50652929989
232 8987.05698058154
233 8974.63303384717
234 8962.23646675322
235 8949.86072369156
236 8937.49203672448
237 8925.11386759006
238 8912.71387014576
239 8900.28808958019
240 8887.84094830371
241 8875.38345787309
242 8862.93231205103
243 8850.51014895842
244 8838.14801937231
245 8825.8858245891
246 8813.76919687208
247 8801.84349376074
248 8790.14547905504
249 8778.69827027224
250 8767.50794480359
251 8756.55007582836
252 8745.79387016348
253 8735.211920092
254 8724.78071164344
255 8714.47984857237
256 8704.29138390728
257 8694.19953235696
258 8684.19059154957
259 8674.25313965965
260 8664.37851197968
261 8654.56106253653
262 8644.79820565265
263 8635.09009332909
264 8625.43889314625
265 8615.84778125318
266 8606.31989179376
267 8596.85741282198
268 8587.4610261229
269 8578.13008606114
270 8568.86292927635
271 8559.65726682357
272 8550.51056526601
273 8541.42005424187
274 8532.38310670101
275 8523.39800933004
276 8514.46432094374
277 8505.58328133684
278 8496.75844301318
279 8487.99699540181
280 8479.30876453934
281 8470.70492178165
282 8462.19711412601
283 8453.79611488622
284 8445.50546096963
285 8437.32251990343
286 8429.24370375955
287 8421.26499776771
288 8413.38184838601
289 8405.58897974941
290 8397.88046629561
291 8390.25009618404
292 8382.69185416814
293 8375.20033993573
294 8367.77097870371
295 8360.39999818749
296 8353.08431393819
297 8345.8205426985
298 8338.60503277127
299 8331.43428515976
300 8324.30507760734
301 8317.21461375081
302 8310.16068709771
303 8303.14190217653
304 8296.15796216747
305 8289.2100499681
306 8282.30123861873
307 8275.43679524947
308 8268.62418127888
309 8261.87251916284
310 8255.19142160449
311 8248.58878986875
312 8242.06329619649
313 8235.60684066214
314 8229.2115779811
315 8222.87120310514
316 8216.58152765026
317 8210.34050358675
318 8204.14775858544
319 8198.00382770174
320 8191.90933996183
321 8185.86438013324
322 8179.86816812877
323 8173.91908134319
324 8168.0149343659
325 8162.15338014726
326 8156.33229298417
327 8150.55003290667
328 8144.80554953864
329 8139.09833880468
330 8133.42830085341
331 8127.79556315148
332 8122.20031750257
333 8116.64271910779
334 8111.12284083643
335 8105.64066861268
336 8100.19612084691
337 8094.78907076455
338 8089.41935873863
339 8084.08679067166
340 8078.79112791469
341 8073.53207593277
342 8068.30927880122
343 8063.12232441221
344 8057.97076372492
345 8052.85414148319
346 8047.77203572977
347 8042.72410319958
348 8037.71012735596
349 8032.73006563553
350 8027.78407014951
351 8022.87244779086
352 8017.9957492367
353 8013.15475956276
354 8008.3504011453
355 8003.5835778286
356 7998.85494608558
357 7994.16464306635
358 7989.51225286498
359 7984.89688194051
360 7980.31734259266
361 7975.77241974926
362 7971.2611472068
363 7966.78300883453
364 7962.33797552273
365 7957.92644176995
366 7953.54902108228
367 7949.20624875944
368 7944.89833801425
369 7940.62505132159
370 7936.38571956353
371 7932.17937039763
372 7928.00491869131
373 7923.86135708235
374 7919.74789375847
375 7915.66401263783
376 7911.60944634733
377 7907.58409922677
378 7903.58794574339
379 7899.62093583542
380 7895.68293274578
381 7891.77369529059
382 7887.89290111849
383 7884.04019513715
384 7880.21523862474
385 7876.41773799539
386 7872.6474629547
387 7868.90423377575
388 7865.18789005471
389 7861.49825694343
390 7857.83512639295
391 7854.1982458745
392 7850.58732064189
393 7847.00203131334
394 7843.44205487829
395 7839.90708204833
396 7836.39682661537
397 7832.91102571776
398 7829.44943327722
399 7826.01181002915
400 7822.597908448
401 7819.20746794754
402 7815.84021918937
403 7812.49588533797
404 7809.17418538668
405 7805.87483547617
406 7802.59754827127
407 7799.34203154281
408 7796.1079857694
409 7792.89510165865
410 7789.70305858912
411 7786.53152436912
412 7783.38015631083
413 7780.24860322299
414 7777.13650781291
415 7774.04350889271
416 7770.96924214072
417 7767.91333758733
418 7764.87541949801
419 7761.85510705319
420 7758.8520129123
421 7755.86574118002
422 7752.89588564627
423 7749.94202881787
424 7747.00374048756
425 7744.08057627922
426 7741.17207615161
427 7738.27776302837
428 7735.3971410196
429 7732.52969343374
430 7729.67488170801
431 7726.832143926
432 7724.00089368338
433 7721.18052862868
434 7718.37044854759
435 7715.57003965052
436 7712.77867005807
437 7709.99568664579
438 7707.22041066389
439 7704.45213249282
440 7701.69010559662
441 7698.933540261
442 7696.18162247808
443 7693.43353733995
444 7690.68843693507
445 7687.9454316585
446 7685.20358054778
447 7682.4618652312
448 7679.71902206322
449 7676.97349989498
450 7674.22347346824
451 7671.46675201519
452 7668.70064716343
453 7665.92187109655
454 7663.12806297798
455 7660.31868120893
456 7657.49427029785
457 7654.65641244398
458 7651.8075557387
459 7648.95037270033
460 7646.08787751245
461 7643.22365008533
462 7640.36137263636
463 7637.50450187107
464 7634.65613381999
465 7631.818861258
466 7628.99448557934
467 7626.18368130268
468 7623.38558993752
469 7620.59748727755
470 7617.81610257068
471 7615.0382236516
472 7612.26088025409
473 7609.48162483733
474 7606.6987050603
475 7603.91112420531
476 7601.1185856638
477 7598.32134165462
478 7595.51998464551
479 7592.71516819601
480 7589.90714519925
481 7587.09534974995
482 7584.27890002906
483 7581.45703937137
484 7578.62932440187
485 7575.79555063935
486 7572.95573836851
487 7570.11013713507
488 7567.25919581993
489 7564.40356754304
490 7561.54416466722
491 7558.68225388614
492 7555.81955715653
493 7552.95823093497
494 7550.10091907109
495 7547.25075883095
496 7544.41117200183
497 7541.58544498626
498 7538.7756767564
499 7535.98298337203
};
\addlegendentry{\STSVGP - \ZZ Trained}
\addplot [thick, color1, dotted, line width = 1.3pt]
table {%
0 52060.261416598
1 49629.7210756338
2 47121.81507657
3 44600.6309462422
4 42125.8230538813
5 39799.5893818331
6 37642.1744460468
7 35649.3359500134
8 33829.7075976869
9 32228.6519072052
10 30840.6133403446
11 29615.591305289
12 28710.8578873756
13 28484.8423863034
14 27170.0692586612
15 27056.4837673631
16 25979.779070783
17 24981.1244036121
18 24211.1295299113
19 23561.6916224995
20 23242.1682441116
21 23612.7270172048
22 22504.9831425209
23 22388.0024614083
24 21772.4197415882
25 20579.2376005778
26 20041.020677072
27 19916.5475859554
28 19471.3833708485
29 18804.0350283352
30 18369.1866328851
31 18095.1711002551
32 17774.1542901279
33 17403.6825853193
34 17069.9224696712
35 16800.7342161483
36 16558.4173309985
37 16294.9914468719
38 16030.9779761697
39 15810.4205710874
40 15624.3409559696
41 15444.7178328483
42 15263.2208493072
43 15090.9624559923
44 14928.0290411036
45 14773.2563839169
46 14625.830367663
47 14491.4271247433
48 14361.2518120872
49 14234.8127727249
50 14117.9016385338
51 14011.9273070474
52 13905.9602971257
53 13801.0292125064
54 13702.3798286915
55 13606.382878106
56 13514.0733608484
57 13426.6704387703
58 13343.4258482571
59 13260.407070372
60 13177.5967560576
61 13099.5710628825
62 13024.523162856
63 12950.9877797418
64 12879.3825183826
65 12809.5878263593
66 12740.9837066985
67 12675.6578502413
68 12612.1019305198
69 12548.4547001537
70 12485.8540997572
71 12424.6588690036
72 12365.1509570791
73 12307.0028201127
74 12250.1568064727
75 12194.4284220155
76 12139.6857518119
77 12085.7909428705
78 12033.4508240675
79 11981.5927345985
80 11930.7382232829
81 11880.3667096857
82 11831.2025833471
83 11782.7560854482
84 11735.5065844101
85 11688.5730905658
86 11642.9908994814
87 11598.0256971234
88 11554.1345285737
89 11510.7061346966
90 11468.3098789035
91 11426.8984617786
92 11386.1599333554
93 11346.1204256059
94 11306.730881521
95 11268.027863882
96 11229.9908553675
97 11192.6397654437
98 11156.0360457686
99 11119.9248672316
100 11084.5394727541
101 11049.654188045
102 11015.4649582417
103 10981.6656192013
104 10948.6418568877
105 10916.2647953823
106 10884.7576845743
107 10853.778616669
108 10823.6028550582
109 10794.0068192948
110 10765.054978834
111 10736.4170495076
112 10708.2275556439
113 10680.415158905
114 10653.1215404211
115 10626.2988659261
116 10599.9692390755
117 10573.9237476023
118 10548.2880958526
119 10522.9975675291
120 10498.1046342231
121 10473.6221449047
122 10449.5052604639
123 10425.7708112262
124 10402.3495546073
125 10379.2690844193
126 10356.5720997992
127 10334.2174994404
128 10312.1639107085
129 10290.5232116571
130 10269.3036544598
131 10248.4328570866
132 10227.6239247872
133 10207.0145025536
134 10186.858012924
135 10167.2291474768
136 10147.7085662901
137 10128.3117595641
138 10109.1746199364
139 10090.4882215491
140 10072.0826876625
141 10053.7363900257
142 10035.5878634928
143 10017.8927015136
144 10000.3489901638
145 9982.9312157301
146 9965.67958555309
147 9948.82690351736
148 9932.05037232987
149 9915.51218262144
150 9899.04361676051
151 9882.90310467303
152 9866.87944574202
153 9851.20194008206
154 9835.45586348101
155 9820.19140652804
156 9804.72016702807
157 9789.6687489014
158 9774.71859478459
159 9759.90796714859
160 9745.31498029649
161 9730.88979716829
162 9716.54869776768
163 9702.46076888534
164 9688.41970564782
165 9674.43487669478
166 9660.66078349391
167 9646.96969793642
168 9633.42990939727
169 9620.07216092774
170 9606.72483663817
171 9593.516593378
172 9580.44924107518
173 9567.52295956888
174 9554.74802161364
175 9542.00945023352
176 9529.39947091946
177 9516.84920584715
178 9504.39930408791
179 9492.05799661307
180 9479.70179672728
181 9467.58451418331
182 9455.50954271031
183 9443.63644623289
184 9431.73513498858
185 9420.29674395687
186 9408.34344779409
187 9396.89978418951
188 9385.71035679461
189 9374.13011713569
190 9362.96795067068
191 9351.37100053596
192 9339.75444957033
193 9329.71687275815
194 9317.94487377819
195 9306.74722138288
196 9295.95185425685
197 9285.43903841238
198 9274.72946341808
199 9264.55392840701
200 9254.51761061322
201 9244.20623834257
202 9233.9611883845
203 9224.85293626433
204 9213.28167001547
205 9206.86110789125
206 9195.54153067592
207 9184.79092441472
208 9179.255046461
209 9169.82044558319
210 9160.9265908499
211 9150.16961258251
212 9139.24085562941
213 9127.07462821175
214 9119.39673359331
215 9110.92272052836
216 9102.14433048394
217 9090.25047618575
218 9081.02226996583
219 9068.16457138759
220 9061.02233734485
221 9050.10868914683
222 9040.7147823458
223 9032.49672596869
224 9021.51194506124
225 9011.81115005272
226 9002.94355291148
227 8993.02074049441
228 8983.93362776748
229 8974.60920631442
230 8964.80569410992
231 8955.15236572943
232 8945.5946873122
233 8936.24270435278
234 8927.15289051051
235 8916.96997927859
236 8907.42568727453
237 8897.67736364423
238 8887.96298400479
239 8878.56558820816
240 8868.38648891187
241 8858.38539283008
242 8848.6091314288
243 8839.06103427458
244 8829.25240187219
245 8819.29307577889
246 8809.44532947626
247 8800.04837084409
248 8790.38589601154
249 8780.74603974654
250 8771.44556611935
251 8761.69693276883
252 8750.85544847776
253 8741.8954766726
254 8731.0312115869
255 8721.4985033562
256 8711.46772039073
257 8701.30348926702
258 8693.80727880338
259 8681.29298785093
260 8675.42873847741
261 8666.30076922877
262 8657.83891822487
263 8650.35539625346
264 8644.1907305956
265 8630.2560094866
266 8625.94851967416
267 8610.68881104141
268 8607.91349504768
269 8603.74412694273
270 8589.38406400174
271 8586.44902690444
272 8570.12308844745
273 8568.76470149223
274 8556.64630754406
275 8543.14993021673
276 8543.80834079081
277 8533.12582243833
278 8527.27233131884
279 8532.58553519947
280 8518.57957518602
281 8504.71786618721
282 8496.33293179296
283 8488.6109463287
284 8476.54472763295
285 8467.45025706928
286 8456.56733935589
287 8450.21608222922
288 8434.97183073732
289 8426.43083392692
290 8420.03047234822
291 8409.1321471262
292 8399.54556792073
293 8390.97806052533
294 8382.10209984771
295 8372.54173709478
296 8364.56713191074
297 8357.31277060364
298 8364.2619667057
299 8343.83060249766
300 8341.93256767187
301 8338.75866775775
302 8323.73966768668
303 8317.35082119186
304 8302.91777006257
305 8299.00873428903
306 8298.60888847355
307 8281.33602696543
308 8278.72729076058
309 8286.90668369974
310 8269.98696719533
311 8258.74539818037
312 8259.62940498139
313 8261.20245718445
314 8240.77853583292
315 8234.83078941134
316 8221.46927816088
317 8212.60921888348
318 8205.00948747155
319 8202.0758045339
320 8191.56360167244
321 8180.24252490103
322 8172.67180297447
323 8167.36624236599
324 8160.9799403389
325 8150.43198735641
326 8142.34306004853
327 8134.90273936316
328 8127.41440767296
329 8118.21253695447
330 8109.86931509318
331 8103.90607347637
332 8096.17738295541
333 8089.9173796355
334 8092.37684967759
335 8079.62517073402
336 8076.7818533624
337 8070.70963746534
338 8081.65506161184
339 8077.96898946
340 8072.49947762219
341 8066.42359855489
342 8056.07001451878
343 8051.87244338883
344 8041.09054335385
345 8034.38382177103
346 8030.21439685908
347 8022.4720434651
348 8014.92642090781
349 8009.48016069612
350 8001.76638372864
351 7998.21911464323
352 7992.36833161318
353 7994.92312059508
354 8002.72668688543
355 8025.54121717508
356 8000.24514861234
357 8043.17414089695
358 8012.13079768076
359 8021.95213896022
360 7994.61176881234
361 7999.99262419346
362 7992.08700715436
363 7985.41703486378
364 7974.31909965178
365 7962.94964982755
366 7948.10646222359
367 7940.5092651975
368 7928.93988094433
369 7920.78175843367
370 7917.49320535065
371 7908.1266869474
372 7897.60502177335
373 7889.94638282838
374 7882.57727089472
375 7878.51857145342
376 7875.3499314606
377 7868.76843266859
378 7859.35638809632
379 7857.22681819103
380 7850.57850689048
381 7843.63728150066
382 7850.81533615309
383 7833.14403140877
384 7840.09856805255
385 7823.79714959352
386 7818.66533893096
387 7814.22670439848
388 7812.49959885188
389 7804.80519607988
390 7803.8280563308
391 7801.36067104931
392 7838.21856524322
393 7833.34939830848
394 7857.39991140878
395 7889.8074601005
396 7879.4161054138
397 7865.07419279427
398 7854.52073385473
399 7844.31394384479
400 7850.13494410816
401 7840.33434639275
402 7840.12800359027
403 7815.69988104691
404 7830.85547517993
405 7866.27443248573
406 7811.40508138327
407 7884.14148511031
408 7896.19854535622
409 7962.45134007813
410 7900.6658475976
411 7850.92887536739
412 7879.86958544508
413 7845.1894291263
414 7868.27419290547
415 7841.3455827281
416 7841.45346795157
417 7818.96295756895
418 7803.44968004988
419 7817.49003022018
420 7786.22538996151
421 7829.55885593243
422 7808.40289457636
423 7813.3742916502
424 7804.5162075259
425 7790.17209205367
426 7778.10865280794
427 7762.25780239401
428 7749.26549884566
429 7744.69638501244
430 7732.77806557809
431 7718.60515293003
432 7711.27745277812
433 7700.76392860609
434 7693.95895391764
435 7686.57692053446
436 7681.71606054402
437 7674.1099295648
438 7668.75393152998
439 7665.10198239795
440 7658.76494208562
441 7652.11518114523
442 7650.06572680009
443 7651.34772988979
444 7645.19270599913
445 7641.7011727822
446 7634.46607127776
447 7627.63893256668
448 7618.87313052401
449 7619.57704296293
450 7613.89379129957
451 7607.6485227651
452 7602.70161495768
453 7600.64434666297
454 7601.15442108285
455 7599.45658200818
456 7593.71866561202
457 7591.75245514941
458 7590.60762894467
459 7589.9397741466
460 7627.92281566526
461 7610.29251424347
462 7609.56177373113
463 7655.10210969046
464 7615.14209544306
465 7610.31180995747
466 7638.32144276345
467 7620.77300318514
468 7663.01180345572
469 7705.99248666311
470 7715.12399329735
471 7723.22793133189
472 7751.57193944844
473 7726.03897120829
474 7709.0618708737
475 7684.00058989037
476 7681.86525885511
477 7666.84221015525
478 7673.2122176029
479 7657.72048330746
480 7654.13809823598
481 7654.24288897876
482 7637.94556877929
483 7624.60503154841
484 7607.68859724527
485 7597.88236875562
486 7587.69097623894
487 7570.71474047775
488 7563.35579004058
489 7551.44544777904
490 7544.80007007105
491 7535.7077823353
492 7532.82743670629
493 7522.96544043353
494 7520.46757530328
495 7528.94056564491
496 7529.05201688076
497 7518.86285371736
498 7531.97994892181
499 7523.9787284014
};
\addlegendentry{\SVGP - \ZZ Trained}
\end{axis}

\end{tikzpicture}
 		
		\vspace{-0.4cm}
		\caption{Negative \ELBO across training iterations}
		
		\label{fig:nyc_small_run_times} 
	\end{subfigure}
	\caption{(a)~Log wall-clock time, including any startup costs, across 7 synthetic spatio-temporal datasets with an increasing number of time steps (average across 5 runs). (b)~Negative \ELBO during training for the small-scale \nyccrime dataset.} %
\end{figure}

\section{Further Improving the Temporal and Spatial Scaling} \label{sec:scaling}

We now propose two approaches to further improve the computational properties of \STVGP and \STSVGP. First, we show how parallel filtering and smoothing can be used for non-conjugate GP inference, which results in a theoretical span complexity of $\bigO(\log N_t d^3)$. %
We then present a spatial mean-field approximation, which can be used independently, or in combination with sparsity.

\subsection{Parallel Bayesian Filtering and Smoothing} \label{sec:parallel-filtering}
The associative scan algorithm \cite{blelloch1989scans} uses a divide-and-conquer approach combined with parallelisation to convert $N$ sequential \emph{associative} operations into $\log N$ sequential steps (for an operator $*$, associativity implies $(a * b) * c = a * (b * c)$). This algorithm has been made applicable to conjugate Markov \gps by deriving a new form of the Kalman filtering and smoothing operations that are associative \cite{sarkka_temporal2021, corenflos2021gaussian}. We give the form of these associative operators in \cref{sec_app:parallel} and show, for the first time, how these methods can be adapted to the non-conjugate setting. This follows directly from the use of the CVI approach to natural gradient VI, which requires only conjugate computations, \ie, the linear filter and smoother. Consider the \STSVGP ELBO: 
\begin{equation*}
	\LL_\STSVGP = \underbrace{\E_{q(\vu)} \big[ \E_{p(\vf \mid \vu)} \big[ \log p(\MY \mid \vf) \big] \big]}_{\begin{subarray}{l}\text{factorises across time and space,}\\
			\text{compute in parallel}\end{subarray}}
	- \underbrace{\E_{q(\vu)} \big[ \log \N (\apxy \mid \vu, \apxv) \big]}_{\begin{subarray}{l}\text{factorises across time,}\\
			\text{compute in parallel}\end{subarray}} + \underbrace{\log  \E_{p(\vu)} \big[\N(\apxy \mid \vu, \apxv) \big]}_{\text{compute with parallel filter}} . %
\end{equation*}
The first two terms can be computed in parallel, since they decompose across time given the marginals $q(\vu_\timeindex)$. The final term can be computed via the parallel filter, and the required marginals $q(\vu_\timeindex)$ via the parallel smoother, which makes \STSVGP a highly parallelisable algorithm. \cref{alg:parallel-filter} gives the filtering and smoothing algorithm and \cref{alg:sparse-filter} shows how this method can be used in place of the sequential filter when performing inference. One drawback of the parallel filter is that when the state dimension is large, many of the available computational resources may be consumed by the arithmetic operations involved in a single filtering step, and the logarithmic scaling may be lost. Fortunately, the spatial mean-field approximation presented in \cref{sec:mean-field} helps to alleviate this issue. In \cref{sec_app:parallel} we provide more details on the method as well as a detailed examination of its practical properties.

\subsection{Spatial Mean-Field Approximation} \label{sec:mean-field}

We reconsider the state space model for the spatio-temporal GP derived in \cref{sec:state-space-spatio-temporal}, which has process noise $\vq_\timeindex \sim \N(\bm{0},\SpaceCov \otimes\MQ_{\timeindex}^{(t)})$. This Kronecker structure implies $N_\Space$ independent processes are linearly mixed using spatial covariance $\SpaceCov$. The linearity of this operation makes it possible to reformulate the model to include the mixing as part of the measurement, rather than the prior:%
\begin{equation} \label{eq:mf-model}
\state(t_{\timeindex+1}) = \MA_{\timeindex} \, \state(t_{\timeindex}) + \vq_\timeindex, \qquad
\MY_\timeindex \mid \state(t_\timeindex) \sim p(\MY_\timeindex \mid [\MC^{(\Space)}_{\SPACE\SPACE} \otimes \MH^{(t)}] \, \state(t_\timeindex)),
\end{equation}
where $\vq_\timeindex \sim \N(\bm{0}, \MI_{N_{\Space}} \otimes \MQ_{\timeindex}^{(t)})$ and $\MC^{(\Space)}_{\SPACE\SPACE}$ is the Cholesky factor of $\SpaceCov$ (see \cref{sec_app:sde_reformulation} for the derivation).
This has the benefit that now both $\MA_\timeindex$ and $\MQ_\timeindex=\MI_{N_{\Space}} \otimes \MQ_{\timeindex}^{(t)}$, are block diagonal, such that under the prior the latent processes are fully independent. This enables a mean-field assumption between the $N_{\Space}$ latent posterior processes:
$q(\state(t)) \approx \prod_{\spaceindex=1}^{N_\Space} q(\state_\spaceindex(t))$,
where $\state_\spaceindex(t)$ is the $d_t$-dimensional state corresponding to spatial point $\SPACE_\spaceindex$. %
This approximation enforces block-diagonal structure in the state covariance, such that matrix operations acting on the full state can be avoided. Dependence between the latent processes is still captured via the measurement model (likelihood), and our experiments show that this approach still approximates the true posterior well (see \cref{sec:experiments} and \cref{app:approx-comparison}) whilst providing significant computational gains when $N_{\Space}$ is large. %

\begin{figure*}[t]
	 \resizebox{\textwidth}{!}{%
	 \input{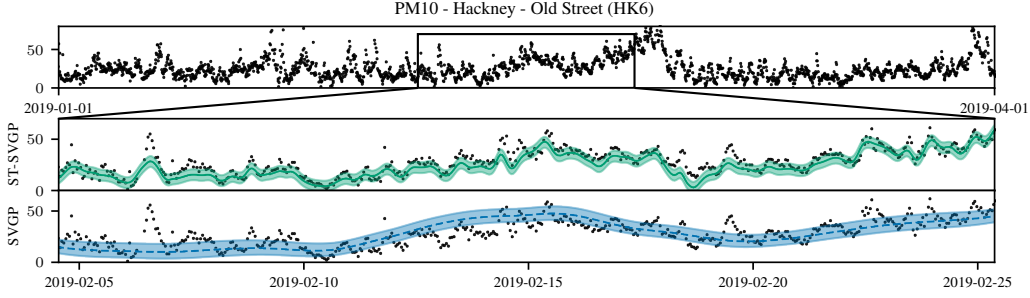}}
	 \caption{Observations of \pmten\ at site HK6, showing rich short-scale structure (\textbf{top}). Mean and 95\% confidence of \STSVGP trained with 30 spatial inducing points (totalling $64{,}770$ inducing points) and \SVGP with 2000 inducing points (minibatch size 100). Both models have similar training times. \STSVGP captures the complex structure of the time series whereas \SVGP smooths the data. }
	\label{fig:air_quality_timeseries}
\end{figure*}

\section{Experiments}
\label{sec:experiments}

We examine the scalability and performance of \STVGP and its variants. Throughout, we use a Mat\'ern-$\nicefrac{3}{2}$ kernel and optimise the hyperparameters by maximising the ELBO using Adam \citep{kingma_adam:2014}. We use learning rates of $\rho=0.01$, $\beta=1$ in the conjugate case, and $\rho=0.01$, $\beta=0.1$ in the non-conjugate case.
We use 5-fold cross-validation (\ie, 80--20 train-test split), train for 500 iterations (except for \airquality where we train for 300) and report RMSE, negative log predictive density (NLPD, see \cref{sec:appendix_metrics}) and average per-iteration training times on CPU and GPU. When using a GPU, the parallel filter and smoother are used. %

\paragraph{Synthetic Experiment} We construct 7 toy datasets with rich temporal structure but smooth spatial structure (see \cref{sec:app_pseudo_periodic_function}) and varying size: $N_t=10, 82, 155, 227, 300, 500, 1000$, and construct a $500 {\times} 100$ grid that serves as a test set for all cases. As the dataset size increases we expect the predictive performance of all methods to improve at the expense of run time. We compare against \SKI and \SVGP (see \cref{sec:related_work}). \cref{fig:shutters_run_times} shows that whilst \SVGP becomes infeasible for more than 300 time steps, the \STSVGP variants scale linearly with time and are faster than \SKI (except for the very small datasets, in which the model compile time in JAX dominates). In \cref{sec:app_further_experiment_details} we show the test performance of all models.%

\begin{table}
  \centering
  \begin{minipage}[t]{.38\textwidth}
	\scriptsize
	\setlength{\tabcolsep}{2pt}
	\renewcommand{\arraystretch}{.75}
	\caption{\nyccrime (small) results. \STSVGP = \SVGP when $\MZ$ is fixed.}
	\label{table:nyc_small_results}
	\vskip 0.05in
	\begin{sc}
	\begin{tabularx}{\textwidth}{clcc}
		\toprule
		Train $\MZ$  & Model & RMSE &  NLPD \\
		\midrule
		$\times$ & \STSVGP &  3.02 $\pm$ 0.13 & 1.72 $\pm$ 0.04  \\
		$\times$ &\SVGP &  3.02 $\pm$ 0.13 & 1.72 $\pm$ 0.04  \\
		\checkmark & \STSVGP & 2.79 $\pm$ 0.15 & 1.64 $\pm$ 0.04  \\
		\checkmark &\SVGP &  2.94 $\pm$ 0.12 & 1.65 $\pm$ 0.05  \\
		\bottomrule
	\end{tabularx}
	\end{sc}
  \end{minipage}
  \hfill
\begin{minipage}[t]{.58\textwidth}
	\caption{Test performance on matching average run time in seconds for the \nyccrime (large) count dataset.}
	\label{table:nyc_crime}
	\vskip 0.05in
	\setlength{\tabcolsep}{3pt}
	\renewcommand{\arraystretch}{.75}
	\scriptsize
	\begin{sc}
		\begin{tabularx}{\textwidth}{lcccc}
			\toprule
			Model & Time (CPU) & Time (GPU) &  RMSE & NLPD \\
			\midrule
			\STSVGP &  20.86 $\pm$ 0.46   & 0.61 $\pm$ 0.00  & 2.77 $\pm$ 0.06        & 1.66 $\pm$ 0.02  \\
			\MFSTSVGP & 20.69 $\pm$ 0.86   &  0.32 $\pm$ 0.00  & 2.75 $\pm$ 0.04        & 1.63 $\pm$ 0.02  \\
			\SVGP-1500 &  12.67 $\pm$ 0.11 &  0.13 $\pm$ 0.00  &  3.20 $\pm$ 0.14  & 1.82 $\pm$ 0.05  \\
			\SVGP-3000 &  80.80 $\pm$ 3.42 &  0.45 $\pm$ 0.01  &  3.02 $\pm$ 0.18  & 1.76 $\pm$ 0.05  \\
			\bottomrule
		\end{tabularx}
	\end{sc}
\end{minipage}  	
  \vspace*{-1em}
\end{table}

\paragraph{\nyccrime\ -- Count Dataset}
We model crime numbers across New York City, USA (\NYC), using daily complaint data from \citep{crime_data}. Crime data has seasonal trends and is spatially dependent. Accurate modelling can lead to more efficient allocation of police resources \cite{flaxman2019scalable, algietti_multi_task_lgcp:2019}. %
We first consider a small subset of the data to highlight when our methods exactly recover \SVGP. We bin the data from 1\textsuperscript{st} to 10\textsuperscript{th} of January 2014 ($N_t=10$) into a spatial grid of size $30 {\times} 30$ and drop cells that do not intersect with land ($N_\Space=447$, $N=4470$). We run \STSVGP and \SVGP with inducing points initialised to the same locations. %
We plot the training \ELBO in \cref{fig:nyc_small_run_times} and performance in \cref{table:nyc_small_results}. For fixed inducing points, both models have the same training curve and provide the same predictions (up to numerical differences). Optimising the inducing points improves both methods. A comparable inference method for non-conjugate likelihoods has not yet been developed for \SKI.

We next consider observations from January to July 2014, with daily binning ($N_t=182$) and the same spatial grid ($N_\Space=447$, $N=81{,}354$). We run \STSVGP and its mean-field variant (\MFSTSVGP) with 30 spatial inducing points (equivalent to \SVGP with $M=30 {\times} 182=5460$). \cref{table:nyc_crime} shows that our methods outperform \SVGP (with $M=1500$, $M=3000$ and batch sizes $1500$, $3000$ respectively) because they can afford more inducing points for the same computational budget.

\paragraph{Regression: \airquality}
Finally, we model \pmten\ ($\mu\textrm{g}/\textrm{m}^3$) air quality across London, UK. The measurements exhibit periodic fluctuations and highly irregular behaviour due to events like weather and traffic jams. Using hourly data from the London air quality network \cite{data_laqn} between January 2019 and April 2019 ($N_t=2159$), we drop sensors that are not within the London boundaries or have more than 40\% of missing data ($N_\Space=72$, $N=155{,}448$).
We run \STSVGP and \MFSTSVGP with $30$ inducing points in space (equivalent to \SVGP with $M=30 {\times} 2159 = 64{,}770$ inducing points). To ensure the run times are comparable on both CPU and GPU, we run \SVGP with 2000, 2500, 5000, and 8000 inducing points with mini-batch sizes of 600, 800, 2000, and 3000 respectively. We run \SKI with $N_t$ temporal inducing points and 6 inducing points in each spatial dimension. %

\begin{table}[h]
	\vspace{-1em}
\caption{\airquality regression. \STSVGP fits the fast-varying structure well, whereas \SVGP smooths the data. Average run time and standard deviation in seconds shown for a single training step. \STSVGP and \MFSTSVGP use $30$ spatial inducing points, equivalent to \SVGP with $30 {\times} 2159 = 64{,}770$ inducing points. Number of inducing points chosen to make run time comparable.}
\label{table:air_quality}
\scriptsize
\setlength{\tabcolsep}{7pt}
\renewcommand{\arraystretch}{.75}
\begin{center}
\begin{tabularx}{\textwidth}{lccccl}
\toprule
	\sc Model (batch size) & \sc Time (CPU) & \sc Time (GPU) & \sc RMSE & \sc NLPD & \\
\midrule
\STSVGP & 16.79 $\pm$ 0.63 & \phantom{1}4.47 $\pm$ 0.01  & \bf \phantom{1}9.96 $\pm$ 0.56 & \bf\phantom{1}8.29 $\pm$ 0.80 & ~ $\leftarrow$ full spatio-temporal model  \\
\MFSTSVGP & 13.74 $\pm$ 0.49 & \phantom{1}0.85 $\pm$ 0.01 & 10.42 $\pm$ 0.63 & \phantom{1}8.52 $\pm$ 0.91 & ~ $\leftarrow$ with mean-field assumption \\
\SVGP-2000 (600) & 20.21 $\pm$ 0.28 &  \phantom{1}0.17 $\pm$ 0.00 & 15.46 $\pm$ 0.44 & 12.93 $\pm$ 0.95 &  \multirow{ 5}{*}{$\begin{rcases} ~\\[3em] \end{rcases} \text{ baselines }$}  \\
\SVGP-2500 (800) &  40.90 $\pm$ 1.11  &  \phantom{1}0.25 $\pm$ 0.00 &  15.53 $\pm$ 1.09  &  13.48 $\pm$ 1.85  & \\
\SVGP-5000 (2000) &  \phantom{1}---  &  \phantom{1}1.19 $\pm$ 0.00 &  14.20 $\pm$ 0.44  &  12.73 $\pm$ 0.73  & \\
\SVGP-8000 (3000) &  \phantom{1}---  & \phantom{1}4.09 $\pm$ 0.01 & 13.83 $\pm$ 0.47 & 12.40 $\pm$ 0.75  & \\
SKI &  23.36 $\pm$ 1.01  & \phantom{1}3.61 $\pm$ 0.01 & 12.01 $\pm$ 0.55 & 10.32 $\pm$ 0.79  & \\
\bottomrule
\end{tabularx}
\end{center}
\end{table}
Our methods significantly outperform the \SVGP baselines because they can afford considerably more inducing points. As shown in \cref{fig:air_quality_timeseries} the \SVGP drastically smooths the data, whereas \STSVGP fits the short-term structure well. The mean-field approach is significantly more efficient, especially when using the parallel algorithm, but we do observe a slight reduction in prediction quality. %

\section{Conclusions and Discussion}
\label{sec:conclusions}

We have shown that variational inference and spatio-temporal filtering can be combined in a principled manner, introducing an approach to GP inference for spatio-temporal data that maintains the benefits of variational \gps, whilst exhibiting favourable scaling properties. Our experiments confirm that \STSVGP outperforms baseline methods because the effective number of inducing points grows with the temporal horizon, without introducing a significant computational burden. Crucially, this leads to improved predictive performance, because fast varying temporal information is captured by the model. We demonstrated how to apply parallel filtering and smoothing in the non-conjugate \gp case, but our empirical analysis identified a maximum state dimension of around $d\approx50$, after which the sub-linear temporal scaling is lost. However, our proposed spatial mean-field approach alleviates this issue somewhat, making the combined algorithm extremely efficient even when both the number of time steps and spatial points are large.
The resemblance of our framework to the \VGP approach suggests many potential extensions, such as multi-task models \citep{alvarez_multi_task:2011} and deep GPs \citep{daminaou_dgp:2013}.
We provide JAX code for all methods at \url{https://github.com/AaltoML/spatio-temporal-GPs}.

\subsection{Limitations and Societal Impact}
We believe our work takes an important step towards allowing sophisticated GP models to be run on both resource constrained CPU machines and powerful GPUs, greatly expanding the usability of such models whilst also reducing unnecessary consumption of resources. %
However, when using predictions from such methods, the limitations of the model assumptions and potential inaccuracies when using approximate inference should be kept in mind, especially in cases such as crime rate monitoring where actions based on biased or incorrect predictions can have harmful societal consequences. 

\begin{ack}
This work was supported by the Academy of Finland Flagship programme: Finnish Center for Artificial Intelligence (FCAI). Parts of this project were done while OH was visiting Finland as part of the FCAI--Turing initiative, supported by HIIT/FCAI.
OH acknowledges funding from The Alan Turing Institute PhD fellowship programme.
AS and WJW acknowledge funding from the Academy of Finland (grant numbers 324345 and 339730).
NAL contributed under the NVIDIA AI Technology Center (NVAITC) Finland program.
TD acknowledges support from EPSRC (EP/T004134/1), UKRI Turing AI Fellowship (EP/V02678X/1), and the Lloyd's Register Foundation programme on Data Centric Engineering. 
We acknowledge the computational resources provided by the Aalto Science-IT project and CSC -- IT Center for Science, Finland.
\end{ack}

\phantomsection%
\addcontentsline{toc}{section}{References}
\begingroup
\small
\bibliographystyle{abbrvnat}
\bibliography{references}
\endgroup

\clearpage

\ifapxinsamepage {
	\clearpage

\nipstitle{{\Large Supplementary Material for} \\ 
    Spatio-Temporal Variational Gaussian Processes}

\appendix

\section{Nomenclature}\label{sec:app:nomencalture}
\vspace{-0.5cm}
\begin{table}[h]
\caption{Overview of notation. Vectors: bold lowercase. Matrices: bold uppercase.}
\label{table:app_notation}
\footnotesize
\begin{center}
\begin{tabularx}{\textwidth}{cc p{.6\textwidth}}
\toprule
	Symbol & Size & Description \\
\midrule
	$\gridX^{(st)}$ & $N_t \times N_\Space \times D$ & Data input locations \\
	$\MY^{(st)}$ & $N_t \times N_\Space$ & Observations \\
	$\gridX$ & $N \times D$ & Data input locations in matrix form. $N=N_tN_\Space$\\
	$\MY$ & $N \times 1$  &  Observations in vector form\\
	$\bm{t}$ & $N_t \times 1$ & Vector of time steps \\
	$\SPACE$ & $N_\Space \times D_\Space$ & Spatial input locations. $D_\Space=D-1$ \\
	$\MZ_{\bm{t}}$ & $M_t \times 1$ & Temporal inducing inputs (we set $\MZ_{\bm{t}} = \bm{t}$) \\
	$\MZ_\SPACE$ & $M_\Space \times D_\Space$ & Spatial inducing inputs \\
	$\tKxx$ & $N_t \times N_t$ & Temporal kernel evaluated at temporal data points \\
	$\tKzz$ & $M_t \times M_t$ & Temporal kernel evaluated at temporal inducing points (we set $M_t=N_t$) \\
	$\sKxx$ & $N_\Space \times N_\Space$ & Spatial kernel evaluated at spatial data points \\
	$\sKzz$ & $M_\Space \times M_\Space$ & Spatial kernel evaluated at spatial inducing points \\
	$\siKzzn$ & $M_\Space \times M_\Space$ & $(\sKzz)^{-1}$ \\
	$\sKxzn$ & $N_\Space \times M_\Space$ & Prior covariance between the spatial data points and spatial inducing points \\
	$\vm_\timeindex$ & $N_\Space \times 1$ & Posterior mean at time $t_\timeindex$ \\
	$\postcov_\timeindex$ & $N_\Space \times N_\Space$ & Posterior covariance at time $t_\timeindex$ \\
	$\state(t)$ & -- & Gaussian distributed state, $\state(t_\timeindex)\sim \N(\vs(t_\timeindex) \mid \bar{\vm}_\timeindex, \bar{\MP}_\timeindex)$\\
	$\bar{\vm}_\timeindex$ & $d \times 1$ & State mean at time $t_\timeindex$ \\
	$\bar{\MP}_\timeindex$ & $d \times d$ & State covariance at time $t_\timeindex$ \\
	$\MA_\timeindex$ & $d \times d$ & Discrete state transition model \\
	$\MQ_\timeindex$ & $d \times d$ & Discrete state process noise covariance \\
	$\MH$ & $1 \times d$ & State measurement model, $f(t)=\MH \vs(t)$ \\
	$\MA_\timeindex^{(t)}$ & $d_t \times d_t$ & Discrete state transition model for a single latent component (determined by $\kappa_t$) \\
	$\MQ_\timeindex^{(t)}$ & $d_t \times d_t$ & Discrete state process noise covariance for a single latent component (determined by $\kappa_t$) \\
	$\MH^{(t)}$ & $1 \times d_t$ & State measurement model for a single latent component (determined by $\kappa_t$) \\
	$\natp=\{\natp^{(1)},\natp^{(2)}\}$ & -- & Natural parameters of the approximate posterior \\
	$\apxnatp=\{\apxnatp^{(1)},\apxnatp^{(2)}\}$ & -- & Natural parameters of the approximate likelihood \\
	$\priornatp$ & -- & Natural parameters of the prior \\ 
	$\apxy$ & $ N \times 1$ & Approximate likelihood mean  \\
	$\apxv$ & $ N \times N$ & Approximate likelihood covariance \\
\bottomrule
\end{tabularx}
\end{center}
\vspace*{-1em}
\end{table}

\section{Kronecker Identities}\label{app:kronecker}

Assuming that the matrices conform and are invertible when required, the following properties hold:
\begin{alignat}{2}
	(\MA \kron \MB)^{-1} &= \MA^{-1} \kron \MB^{-1} \label{eqn:kron_inv}\\
	\MA \kron (\MB + \MC) &= \MA \kron \MB + \MA \kron \MC \label{eqn:kron_addition_lhs}\\
	 (\MB + \MC) \kron \MA  &= \MB \kron \MA + \MC \kron \MA \label{eqn:kron_addition_rhs}\\
	 (\MA \kron \MB) (\MC \kron \MD) &= (\MA \MC) \kron (\MB \MD) \label{eqn:kron_mixed_property}
\end{alignat}
For further properties of kronecker products in the context of \gp regression see Ch.~5 of \citet{saatci_thesis:2011}.

\section{Filtering and Smoothing Algorithms}\label{app:filter-smoother}

\begin{minipage}[t]{0.51\textwidth}
\begin{algorithm}[H]%
	\footnotesize
	\caption{Sequential filtering and smoothing}
	\label{alg:sequential-filter}
	\begin{algorithmic}
		\STATE {\bfseries Input:} $\!$Likelihood:$\{\apxy, \apxv\}$, Initial state:$\{\bar{\vm}_0, \bar{\MP}_0\}$,
		\STATE Model matrices:$\{\MA, \MQ, \MH\} $
		\STATE \COMMENT{Run filter:}
		\FOR{$\timeindex=1:N_t$}
		\STATE \COMMENT{Filter predict:}
		\STATE $\bar{\vm}_\timeindex = \MA_\timeindex \bar{\vm}_{\timeindex-1}, \quad \bar{\MP}_\timeindex = \MA_\timeindex \bar{\MP}_{\timeindex-1} \MA_\timeindex^\T + \MQ_\timeindex$
		\STATE $\MLambda_\timeindex = \MH \bar{\MP}_\timeindex \MH^\T + \apxv_\timeindex, \quad \MW_\timeindex = \bar{\MP}_\timeindex \MH^\T \MLambda_\timeindex^{-1}$
		\STATE \COMMENT{Compute log likelihood:}
		\STATE $\ell_\timeindex = \log \N(\apxy_\timeindex \mid \MH \bar{\vm}_\timeindex, \MLambda_\timeindex)$ 
		\STATE \COMMENT{Filter update:}
		\STATE $\bar{\vm}_\timeindex = \bar{\vm}_\timeindex + \MW_\timeindex (\apxy - \MH \bar{\vm}_\timeindex)$
		\STATE $\bar{\MP}_\timeindex = \bar{\MP}_\timeindex - \MW_\timeindex \MLambda_\timeindex \MW_\timeindex^\T$
		\ENDFOR
		\STATE \COMMENT{Run smoother:}
		\FOR{$\timeindex=N_t-1:1$}
		\STATE $\MG_\timeindex = \bar{\MP}_\timeindex \MA_{\timeindex+1} \bar{\MP}_\timeindex^{-1}$
		\STATE $\MR_{\timeindex+1} = \MA_{\timeindex+1} \bar{\MP}_\timeindex \MA_{\timeindex+1}^\T + \MQ_{\timeindex+1}$
		\STATE $\bar{\vm}_\timeindex = \bar{\vm}_\timeindex + \MG_\timeindex(\bar{\vm}_{\timeindex+1} - \MA_{\timeindex+1} \bar{\vm}_\timeindex)$
		\STATE $\bar{\MP}_\timeindex = \bar{\MP}_\timeindex + \MG_\timeindex (\bar{\MP}_{\timeindex+1} - \MR_{\timeindex+1})\MG_\timeindex^\T$
		\ENDFOR
		\STATE \COMMENT{Return marginals and log likelihood:}
		\RETURN $q(\vf_\timeindex)=\N(\vf_\timeindex \mid \MH\bar{\vm}_\timeindex, \MH \bar{\MP}_\timeindex \MH^\T) ,\, \forall \timeindex$ \\
		\hspace{3em}
		$\ell=\sum_\timeindex \ell_\timeindex$%
	\end{algorithmic}
\end{algorithm}
\end{minipage}
\hfill
\begin{minipage}[t]{0.51\textwidth}
	\begin{algorithm}[H]%
		\footnotesize
		\caption{Parallel filtering and smoothing}
		\label{alg:parallel-filter}
		\begin{algorithmic}
			\STATE {\bfseries Input:} $\!$Likelihood:$\{\apxy, \apxv\}$, Initial state:$\{\bar{\vm}_0, \bar{\MP}_0\}$,
			\STATE Model matrices:$\{\MA, \MQ, \MH\} $
			\STATE \COMMENT{Initialise filtering elements:} 
			\STATE $\MA_0 = \MI$, \qquad $\MQ_0 = \bar{\MP}_0$
			\FOR{$\timeindex=1:N_t$ \textbf{in parallel}}
			\STATE $\MT_\timeindex = \MH \MQ_{\timeindex-1} \MH^\T + \apxv_\timeindex$, 
			\STATE $\MK_\timeindex = \MQ_{\timeindex-1} \MH^\T \MT_\timeindex^{-1}$,
			\STATE $\MB_\timeindex = \MA_{\timeindex-1} - \MK_\timeindex \MH \MA_{\timeindex-1}$, 
			\STATE $\filtmean_\timeindex = \MK_\timeindex \apxy_\timeindex$, \qquad $\filtcov_\timeindex = \MQ_{\timeindex-1} - \MK_\timeindex \MH \MQ_{\timeindex-1}$, 
			\STATE $\vphi_\timeindex = \MA_{\timeindex-1}^\T \MH^\T \MT_\timeindex^{-1} \apxy_\timeindex$,
			\STATE $\MJ_\timeindex = \MA_{\timeindex-1}^\T \MH^\T \MT_\timeindex^{-1} \MH \MA_{\timeindex-1}$
			\ENDFOR
			\STATE $\filtmean_1 = \bar{\vm}_0 + \MK_1 (\apxy_1 - \MH \bar{\vm}_0 )$
			\STATE \COMMENT{Run associative scan:} 
			\STATE $\filtmean, \filtcov=$\textbf{associative$\_$scan}$((\MB, \filtmean, \filtcov, \vphi, \MJ), \stackrel{\textrm{filter}}{*})$ 
			\STATE where operator $\stackrel{\textrm{filter}}{*}$ defined by \cref{eq:assoc_op1} and \cref{eq:assoc_op2}
			\STATE \COMMENT{Compute log likelihood:} 
			\FOR{$\timeindex=1:N_t$ \textbf{in parallel}}
			\STATE $\MLambda_\timeindex = \MH (\MA_\timeindex \filtcov_{\timeindex-1} \MA_\timeindex^\T + \MQ_\timeindex ) \MH^\T $
			\STATE $\ell_\timeindex = \log \N(\apxy_\timeindex \mid \MH \MA_\timeindex \filtmean_{\timeindex-1}, \MLambda_\timeindex)$ 
			\ENDFOR
			\STATE \COMMENT{Initialise smoothing elements:} 
			\STATE $\ME_{N_t}=\bm{0}$, \quad $\smoothmean_{N_t}=\filtmean_{N_t}$, \quad $\smoothcov_{N_t}=\filtcov_{N_t}$
			\FOR{$\timeindex=1:N_t-1$ \textbf{in parallel}}
			\STATE $\ME_\timeindex = \filtcov_\timeindex \MA_\timeindex^\T (\MA_\timeindex \filtcov_\timeindex \MA_\timeindex + \MQ_\timeindex )^{-1}$
			\STATE $\smoothmean_\timeindex =  \filtmean_\timeindex - \ME_\timeindex \MA_\timeindex \filtmean_\timeindex$
			\STATE $\smoothcov_\timeindex = \filtcov_\timeindex - \ME_\timeindex \MA_\timeindex \filtcov_\timeindex$
			\ENDFOR
			\STATE \COMMENT{Run associative scan:} 
			\STATE $\smoothmean, \smoothcov=$\textbf{associative$\_$scan}$((\ME, \smoothmean, \smoothcov), \stackrel{\textrm{smoother}}{*})$ 
			\STATE where operator $\stackrel{\textrm{smoother}}{*}$ defined by \cref{eq:assoc_op_smooth1} and \cref{eq:assoc_op_smooth2} \hspace{-1em}
			\STATE \COMMENT{Return marginals and log likelihood:}
			\RETURN $q(\vf_\timeindex)=\N(\vf_\timeindex \mid \MH\bar{\vm}_\timeindex, \MH \bar{\MP}_\timeindex \MH^\T) ,\, \forall \timeindex$ \\
			\hspace{3em}
			$\ell=\sum_\timeindex \ell_\timeindex$%
		\end{algorithmic}
	\end{algorithm}
\end{minipage} 
\section{Sparse Kronecker Decomposition}
\label{sec_app:sparse_kronecker_decomposition}

The input locations and inducing points are ordered such that they lie on a space time grid, and we define $\MX = \vectext(\gridX)$. See \cref{sec:background} and \cref{app:kronecker} for details. We have assumed a separable kernel between space and time, hence we can decompose $\MK$ as a Kronecker product:
\begin{align}
	\Kxx &= \tKxx \kron \sKxx , \label{eqn:def_kxx} \\
	\Kxz &= \tKxz \kron \sKxz = \tKxx \kron \sKxz , \label{eqn:def_kxz} \\
	\Kzz &= \tKzz \kron \sKzz = \tKxx \kron \sKzz . \label{eqn:def_kzz}
\end{align}
The sparse conditional is,
\begin{equation}
	p(\vf \mid \vu) = \N(\vf \mid \MK_{\MX, \MZ} \MK^{-1}_{\MZ, \MZ} \vu, \MK_{\MX, \MX} - \MK_{\MX, \MZ} \MK^{-1}_{\MZ, \MZ} \MK_{\MZ, \MX}) ,
\end{equation}
which we can decompose using the Kronecker formulation of $\MK$. Starting with the mean term:
\begin{align}
	\mu &= \Kxz \iKzz \vu &  \nonumber \\
	&=  \left[ (\tKxx \kron \sKxz) (\tKxx \kron \sKzz )^{-1} \right] \vu & \text{sub \cref{eqn:def_kxz} and \cref{eqn:def_kzz}} \nonumber \\
	&= \left[ (\tKxx \kron \sKxz) (\tiKxx \kron \siKzz ) \right] \vu & \text{apply \cref{eqn:kron_inv}} \nonumber \\
	&= \left[ (\tKxx \tiKxx) \kron (\sKxz \siKzz) \right] \vu & \text{apply \cref{eqn:kron_mixed_property}} \nonumber \\
	&= \left[ \MI \kron (\sKxz \siKzz) \right] \vu, & 
\end{align}
And now the covariance term:
\begin{align}
	\Sigma &= \Kxx - \Kxz\iKzz \Kzx & \nonumber \\
	&= (\tKxx \kron \sKxx) - (\tKxx \kron \sKxz)(\tiKxx \kron \siKzz)(\tKxx \kron \sKzx) \,\, \text{sub \cref{eqn:def_kxx}--\eqref{eqn:def_kzz}, \cref{eqn:kron_inv}} \nonumber \\
	&= (\tKxx \kron \sKxx) - (\tKxx  \tiKxx  \tKxx) \kron (\sKxz  \siKzz  \sKzx) \qquad\quad\,\,\,\, \text{apply \cref{eqn:kron_mixed_property}} \nonumber \\
	&= (\tKxx \kron \sKxx) - (\tKxx) \kron (\sKxz  \siKzz  \sKzx)  \nonumber \\
	&= \tKxx \kron  (\sKxx - \sKxz  \siKzz  \sKzx). \qquad\qquad\qquad\qquad\qquad\qquad\,\,\, \text{apply \cref{eqn:kron_addition_lhs}}
\end{align}
Substituting this back into the conditional we have:
\begin{align}
	p(\vf \mid \vu) &= \N(\vf \mid \left[ \MI \kron (\sKxx \kron \siKzz) \right]\vu, \tKxx \kron  (\sKxx - \sKxz  \siKzz  \sKzx) ) .
\label{eqn:simplified_conditional}
\end{align}
At this point the covariance matrix is dense and so we cannot decompose any further. 

\section{Parallel Filtering and Smoothing for Spatio-Temporal VGP}
\label{sec_app:parallel}

Here we provide more details of the parallel filtering and smoothing method for \STSVGP as well as performance profiling and discussion of the benefits and drawbacks of the parallel and sequential approaches in practice.

\citet{sarkka_temporal2021} derive a new (equivalent) form of the linear filtering and smoothing operations that are \emph{associative}. This property states that for an operator $*$ we have $(a * b) * c = a * (b * c)$. Associativity allows for application of the associative scan algorithm \cite{blelloch1989scans}, which uses a divide-and-conquer approach to construct a computational `tree', each level of which involves applying the operator $*$ to pairs of states in parallel before propagating the result up the tree. The height of this tree is $\log N_t$, and hence the computational span complexity is $\bigO(\log N_t)$ if there is enough parallel compute capacity to fully parallelise all of the required operations.

The associative filtering operator acts on a sequence of five elements $(\MB_\timeindex, \filtmean_\timeindex, \filtcov_\timeindex, \vphi_\timeindex, \MJ_\timeindex)$, where the elements correspond to the parameters of the following quantities,
\begin{gather}
	\begin{aligned}
	p(\state_\timeindex \mid \state_{\timeindex-1}, \MY_\timeindex) &= \N(\state_\timeindex \mid \MB_\timeindex \state_{\timeindex-1} + \filtmean_\timeindex, \filtcov_\timeindex) \, , \\
	p(\MY_\timeindex \mid \state_{\timeindex-1}) &= \N(\state_{\timeindex-1} \mid \MJ_\timeindex \vphi_\timeindex , \MJ_\timeindex^{-1}) .
	\end{aligned}
\end{gather}
$\vphi_\timeindex$, $\MJ_\timeindex$ are the precision-adjusted mean and precision of $p(\MY_\timeindex \mid \state_{\timeindex-1})$. The elements are first initialised as follows (letting $\MA_0 = \MI$ and $\MQ_0 = \MP_\infty$),
\begin{gather}
	\begin{aligned}
		\MT_\timeindex & = \MH \MQ_{\timeindex-1} \MH^\T + \MV_\timeindex , \\
		\MK_\timeindex & = \MQ_{\timeindex-1} \MH^\T \MT_\timeindex^{-1} , \\
		\MB_\timeindex & = \MA_{\timeindex-1} - \MK_\timeindex \MH \MA_{\timeindex-1} , \\
		\filtmean_\timeindex &= \MK_\timeindex \MY_\timeindex , \\
		\filtcov_\timeindex &= \MQ_{\timeindex-1} - \MK_\timeindex \MH \MQ_{\timeindex-1} , \\
		\vphi_\timeindex &= \MA_{\timeindex-1}^\T \MH^\T \MT_\timeindex^{-1} \MY_\timeindex , \\
		\MJ_\timeindex &= \MA_{\timeindex-1}^\T \MH^\T \MT_\timeindex^{-1} \MH \MA_{\timeindex-1} ,
	\end{aligned}
\end{gather}
and we set $\filtmean_1 = \bar{\vm}_0 + \MK_1 (\MY_1 - \MH \bar{\vm}_0 )$ to account for the initial mean $\bar{\vm}_0$. Here $\MA_\timeindex$, $\MQ_\timeindex$ and $\MH$ are the model matrices defining the \gp prior, as laid out in the main paper. Once initialised, the associative operator $\stackrel{\textrm{filter}}{*}$ is defined as,
\begin{equation} \label{eq:assoc_op1}
	(\MB_{i,j}, \filtmean_{i,j}, \filtcov_{i,j}, \vphi_{i,j}, \MJ_{i,j}) = (\MB_i, \filtmean_i, \filtcov_i, \vphi_i, \MJ_i) \stackrel{\textrm{filter}}{*} (\MB_j, \filtmean_j, \filtcov_j, \vphi_j, \MJ_j) ,
\end{equation}
where,
\begin{gather} \label{eq:assoc_op2}
	\begin{aligned}
		\MW_{i,j} &= (\filtcov^{-1}_i + \MJ_j)^{-1} \filtcov^{-1}_i , \\
		\MB_{i,j} &= \MB_j \MW_{i,j} \MB_i , \\
		\filtmean_{i,j} &= \MB_j \MW_{i,j} (\filtmean_i + \filtcov_i \vphi_j ) + \filtmean_j , \\
		\filtcov_{i,j} &= \MB_j \MW_{i,j} \filtcov_i \MB^\T_j + \filtcov_j, \\
		\vphi_{i,j} &= \MB_i^\T \MW_{i,j}^\T (\vphi_j - \MJ_j \filtmean_i ) + \vphi_i , \\
		\MJ_{i,j} &= \MB_i^\T \MW_{i,j}^\T \MJ_j \MB_i + \MJ_i . \\
	\end{aligned}
\end{gather}
The associative scan algorithm (which is implemented in various machine learning frameworks) is then applied to the operator defined by \cref{eq:assoc_op1} and \cref{eq:assoc_op2} to obtain the filtered elements, of which $\filtmean$ and $\filtcov$ correspond to the filtering means and covariances.

A similar approach leads to a parallel version of the Rauch-Tung-Striebel smoother by defining an associative operator which acts on the elements $(\MG_\timeindex, \smoothmean_\timeindex, \smoothcov_\timeindex)$. The elements are initialised as,
\begin{gather}
	\begin{aligned}
		\MG_\timeindex & = \filtcov_\timeindex \MA_\timeindex^\T (\MA_\timeindex \filtcov_\timeindex \MA_\timeindex + \MQ_\timeindex )^{-1} , \\
		\smoothmean_\timeindex & =  \filtmean_\timeindex - \MG_\timeindex \MA_\timeindex \filtmean_\timeindex , \\
		\smoothcov_\timeindex & = \filtcov_\timeindex - \MG_\timeindex \MA_\timeindex \filtcov_\timeindex,
	\end{aligned}
\end{gather}
for $n<N_t$ and we set $\MG_{N_t}=\bm{0}$, $\smoothmean_{N_t}=\filtmean_{N_t}$, $\smoothcov_{N_t}=\filtcov_{N_t}$. Note that the initial value of $\MG_\timeindex$ corresponds to the standard smoothing gain.
The associative smoothing operator $\stackrel{\textrm{smoother}}{*}$ is defined as,
\begin{equation} \label{eq:assoc_op_smooth1}
	(\MG_{i,j}, \smoothmean_{i,j}, \smoothcov_{i,j}) = (\MG_i, \smoothmean_i, \smoothcov_i) \stackrel{\textrm{smoother}}{*} (\MG_j, \smoothmean_j, \smoothcov_j) ,
\end{equation}
where
\begin{gather} \label{eq:assoc_op_smooth2}
	\begin{aligned}
		\MG_{i,j} & = \MG_i \MG_j , \\
		\smoothmean_{i,j} & = \MG_i \smoothmean_j + \smoothmean_i , \\
		\smoothcov_{i,j} & = \MG_i \smoothcov_j \MG_i^\T + \smoothcov_i .
	\end{aligned}
\end{gather}
After applying the associative scan to these elements using the operator defined by \cref{eq:assoc_op_smooth1} and \cref{eq:assoc_op_smooth2}, $\smoothmean$ and $\smoothcov$ correspond to the smoothed means and covariances. The full filtering and smoothing algorithm is given in \cref{alg:parallel-filter}.

\subsection{Profiling the Parallel Filter and Smoother} \label{sec_app:parallel_profiling}

Here we provide detailed analysis of \STVGP applied to a spatial log-Gaussian Cox process where the bin widths can be altered to modify the number of temporal and spatial points. %

\cref{fig:profiling} (a) plots the ratio of mean iteration wall-times obtained with the parallel and sequential filter approaches on a single NVIDIA Tesla V100 GPU. We can see that the parallel filter outperforms the sequential filter when the number of spatial points is $N_\Space \leq 20$. In the low-dimensional case, where $N_\Space = 10$ and the number of time steps $N_t = 3000$, the parallel filter achieves over 29x speed-up relative to the sequential filter. In contrast, the sequential filter outperforms the parallel filter when $N_\Space \geq 50$. When $N_\Space = 300$ and $N_t = 200$ the parallel filter is approximately $2{\times}$ slower than the sequential filter.

\cref{fig:profiling} (b) plots the ratio of mean iteration wall-times obtained with the sequential filter on a single Intel Xeon CPU E5-2698 v4 CPU and the parallel filter on a single NVIDIA Tesla V100 GPU. The parallel filter GPU runs outperform the CPU runs in all settings, with speed-ups as high as 16x for large $N_\Space$. Surprisingly, higher speed-ups relative to the CPU runs are obtained with $N_\Space = 10$ than with $N_\Space = 20$. This suggest that the parallel filter is particularly efficient at very low $N_\Space$. %

\cref{fig:profiling} (c) plots the ratio of mean iteration wall-times obtained with the sequential filter on a single Intel Xeon CPU E5-2698 v4 CPU and the sequential filter on a single NVIDIA Tesla V100 GPU. The sequential filter GPU runs outperform the CPU runs in all cases apart from those at $N_\Space = 10$, with speed-up factors as has high as 37x. On the whole, these results suggest that the parallel filter is crucial to achieving good performance at low $N_\Space$. However, it is not as effective as the sequential filter at high $N_\Space$ and comes with increased memory footprint. Having both filtering options appears to be the best solution for a general-purpose spatio-temporal GP algorithm. All data associated with \cref{fig:profiling} are tabulated in \cref{table:profiling1,table:profiling2,table:profiling3}, where the dashes (---) denote configurations that run out of GPU memory.

In addition to the wall-time experiments, we also studied the characteristics of the parallel and sequential filtering approaches on GPUs via performance profiling. The primary reason causing the sequential filtering approach to be ill-suited for very low $N_\Space$ when using a GPU is decreased computational intensity of operations. JAX is unable to fuse all operations in the algorithm and therefore, with low-dimensional data, the CUDA kernel execution times can become smaller than the kernel launch overheads. This issue is alleviated with increasing $N_\Space$ as soon as the kernel executions take more time than kernel launches, since then the kernel launches can be completely overlapped with the execution of the previous kernel. With the parallel filtering approach, we are able to tackle the kernel launch overhead issue by introducing the associative scan operation, which JAX is able to combine as a single (but more computationally demanding) kernel. The reason the parallel filter is outperformed by the sequential filter at high $N_\Space$ is more intricate. One of the reasons appears to be that the associative scan operation introduces more reads and writes to global GPU memory, and with high-dimensional data the operation becomes increasingly limited by memory bandwidth.

\begin{figure*}[t]
	\scriptsize
	\centering
	\pgfplotsset{yticklabel style={rotate=90}, ylabel style={yshift=-2pt}, title style={yshift=-4pt}, xlabel style={yshift=3pt},scale only axis,axis on top,clip=true,clip marker paths=true}
	\pgfplotsset{legend style={inner xsep=1pt, inner ysep=1pt, row sep=0pt},legend style={at={(0.98,0.95)},anchor=north east},legend style={rounded corners=1pt}, colorbar style={width=7, xshift=-6pt}}
	\setlength{\figurewidth}{.38\textwidth}
	\setlength{\figureheight}{.5\figurewidth}
	\begin{subfigure}{.48\textwidth}
\begin{tikzpicture}

\begin{axis}[
colorbar,
colorbar style={ytick={0.61,1.11,1.61,2.11},yticklabels={0.5,1.0,1.5,2.0},ylabel={}},
colormap={mymap}{[1pt]
  rgb(0pt)=(0,0,1);
  rgb(1pt)=(1,1,1);
  rgb(2pt)=(1,0,0)
},
height=\figureheight,
point meta max=2.22239314336454,
point meta min=0.0344219012255231,
tick align=outside,
tick pos=left,
width=\figurewidth,
x grid style={white!69.0196078431373!black},
xlabel={spatial points, \(\displaystyle N_\Space\)},
xmin=0, xmax=370,
xtick style={color=black},
y grid style={white!69.0196078431373!black},
ylabel={time steps, \(\displaystyle N_t\)},
ymin=0, ymax=3200,
ytick style={color=black}
]
\addplot [only marks, scatter, scatter src=explicit, colormap={mymap}{[1pt]
  rgb(0pt)=(0,0,1);
  rgb(1pt)=(1,1,1);
  rgb(2pt)=(1,0,0)
}]
table [x=x, y=y, meta=colordata]{%
x                      y                      colordata
10 200 0.149764885003304
50 200 1.53732880927631
100 200 2.11194214124245
150 200 1.86382240020898
200 200 1.95170938136085
250 200 2.22239314336454
300 200 2.18874022561368
20 200 0.733685726563444
10 500 0.0676457241133448
50 500 1.54128932640639
100 500 2.11596571418653
150 500 1.84818389915223
20 500 0.712783255871181
10 1000 0.0511480940088568
50 1000 1.54132480312013
100 1000 2.13612973701079
20 1000 0.715263804635496
10 1500 0.0430527058155519
50 1500 1.53882383084844
100 1500 2.12592903468639
20 1500 0.709679706720979
10 2000 0.0380381147388468
50 2000 1.53734985971155
20 2000 0.714987017563605
10 2500 0.0344698659621282
50 2500 1.54410397108976
20 2500 0.701055689034195
10 3000 0.0344219012255231
50 3000 1.54200320866304
20 3000 0.724070904655641
};
\end{axis}

\end{tikzpicture}
 \vspace{-2em} \caption{parallel (GPU) / sequential (GPU)}
	\end{subfigure}\label{fig:pargpu_vs_sqegpu}
	\hfill
	\begin{subfigure}{.48\textwidth}
\begin{tikzpicture}

\begin{axis}[
colorbar,
colorbar style={ylabel={}},
colormap={mymap}{[1pt]
  rgb(0pt)=(1,0.96078431372549,0.941176470588235);
  rgb(1pt)=(0.996078431372549,0.87843137254902,0.823529411764706);
  rgb(2pt)=(0.988235294117647,0.733333333333333,0.631372549019608);
  rgb(3pt)=(0.988235294117647,0.572549019607843,0.447058823529412);
  rgb(4pt)=(0.984313725490196,0.415686274509804,0.290196078431373);
  rgb(5pt)=(0.937254901960784,0.231372549019608,0.172549019607843);
  rgb(6pt)=(0.796078431372549,0.0941176470588235,0.113725490196078);
  rgb(7pt)=(0.647058823529412,0.0588235294117647,0.0823529411764706);
  rgb(8pt)=(0.403921568627451,0,0.0509803921568627)
},
height=\figureheight,
point meta max=16.5654907907927,
point meta min=1.40143815067336,
tick align=outside,
tick pos=left,
width=\figurewidth,
x grid style={white!69.0196078431373!black},
xlabel={spatial points, \(\displaystyle N_\Space\)},
xmin=0, xmax=370,
xtick style={color=black},
y grid style={white!69.0196078431373!black},
ylabel={time steps, \(\displaystyle N_t\)},
ymin=0, ymax=3200,
ytick style={color=black}
]
\addplot [only marks, scatter, scatter src=explicit, colormap={mymap}{[1pt]
  rgb(0pt)=(1,0.96078431372549,0.941176470588235);
  rgb(1pt)=(0.996078431372549,0.87843137254902,0.823529411764706);
  rgb(2pt)=(0.988235294117647,0.733333333333333,0.631372549019608);
  rgb(3pt)=(0.988235294117647,0.572549019607843,0.447058823529412);
  rgb(4pt)=(0.984313725490196,0.415686274509804,0.290196078431373);
  rgb(5pt)=(0.937254901960784,0.231372549019608,0.172549019607843);
  rgb(6pt)=(0.796078431372549,0.0941176470588235,0.113725490196078);
  rgb(7pt)=(0.647058823529412,0.0588235294117647,0.0823529411764706);
  rgb(8pt)=(0.403921568627451,0,0.0509803921568627)
}]
table [x=x, y=y, meta=colordata]{%
x                      y                      colordata
10 200 1.40143815067336
50 200 5.98584666783343
100 200 5.99061120242822
150 200 14.0543730652866
200 200 15.343049554267
250 200 15.6923367778909
300 200 16.5654907907927
20 200 2.31649618035934
10 500 2.8289163654795
50 500 6.06391524108263
100 500 6.08611375432671
150 500 15.8115379659326
20 500 1.94202784327073
10 1000 4.3448211119794
50 1000 5.99248807340588
100 1000 6.67149925293302
20 1000 1.76805695896848
10 1500 4.94013238129984
50 1500 5.91395387783365
100 1500 6.72144765141171
20 1500 1.69725231666337
10 2000 5.55484610435323
50 2000 5.98659885407062
20 2000 1.66540621925705
10 2500 6.27679828820824
50 2500 5.99957028615988
20 2500 1.63341888110838
10 3000 6.38449877796227
50 3000 6.09570503547349
20 3000 1.64130360225254
};
\end{axis}

\end{tikzpicture}
 \vspace{-2em} \caption{sequential (CPU) / parallel (GPU)}
	\end{subfigure}\label{fig:pargpu_vs_sqecpu}
	\\
	\vspace{0.5em}
	\begin{subfigure}{.48\textwidth}
\begin{tikzpicture}

\begin{axis}[
colorbar,
colorbar style={ylabel={}},
colormap={mymap}{[1pt]
  rgb(0pt)=(1,0.96078431372549,0.941176470588235);
  rgb(1pt)=(0.996078431372549,0.87843137254902,0.823529411764706);
  rgb(2pt)=(0.988235294117647,0.733333333333333,0.631372549019608);
  rgb(3pt)=(0.988235294117647,0.572549019607843,0.447058823529412);
  rgb(4pt)=(0.984313725490196,0.415686274509804,0.290196078431373);
  rgb(5pt)=(0.937254901960784,0.231372549019608,0.172549019607843);
  rgb(6pt)=(0.796078431372549,0.0941176470588235,0.113725490196078);
  rgb(7pt)=(0.647058823529412,0.0588235294117647,0.0823529411764706);
  rgb(8pt)=(0.403921568627451,0,0.0509803921568627)
},
height=\figureheight,
point meta max=37.7775763350719,
point meta min=0.191364095998952,
tick align=outside,
tick pos=left,
width=\figurewidth,
x grid style={white!69.0196078431373!black},
xlabel={spatial points, \(\displaystyle N_\Space\)},
xmin=0, xmax=370,
xtick style={color=black},
y grid style={white!69.0196078431373!black},
ylabel={time steps, \(\displaystyle N_t\)},
ymin=0, ymax=3200,
ytick style={color=black}
]
\addplot [only marks, scatter, scatter src=explicit, colormap={mymap}{[1pt]
  rgb(0pt)=(1,0.96078431372549,0.941176470588235);
  rgb(1pt)=(0.996078431372549,0.87843137254902,0.823529411764706);
  rgb(2pt)=(0.988235294117647,0.733333333333333,0.631372549019608);
  rgb(3pt)=(0.988235294117647,0.572549019607843,0.447058823529412);
  rgb(4pt)=(0.984313725490196,0.415686274509804,0.290196078431373);
  rgb(5pt)=(0.937254901960784,0.231372549019608,0.172549019607843);
  rgb(6pt)=(0.796078431372549,0.0941176470588235,0.113725490196078);
  rgb(7pt)=(0.647058823529412,0.0588235294117647,0.0823529411764706);
  rgb(8pt)=(0.403921568627451,0,0.0509803921568627)
}]
table [x=x, y=y, meta=colordata]{%
x                      y                      colordata
10 200 0.209886223474839
50 200 9.20221453037094
100 200 12.6518242502073
150 200 26.1948553399749
200 200 29.9451737537473
250 200 34.874541658552
300 200 36.257556050841
350 200 37.7775763350719
20 200 1.69958018316838
10 500 0.191364095998952
50 500 9.34624783731368
100 500 12.8780080367944
150 500 29.2226298894709
200 500 31.9580997699973
250 500 36.4651892382613
300 500 36.4718512931
20 500 1.384244929119
10 1000 0.222229318687188
50 1000 9.23637049994204
100 1000 14.2511879446355
150 1000 29.3505732654635
200 1000 31.9931636164138
20 1000 1.26462714728406
10 1500 0.212686066101984
50 1500 9.10053316174894
100 1500 14.2893207172608
150 1500 28.8867621937637
20 1500 1.20450552632117
10 2000 0.211295873474024
50 2000 9.20349690845478
100 2000 14.347139542643
150 2000 29.0351721196433
20 2000 1.19074382573848
10 2500 0.216360395665854
50 2500 9.2639603036916
100 2500 14.270900475984
20 2500 1.1451175991769
10 3000 0.21976658630949
50 3000 9.39959672376359
100 3000 14.5256631477435
20 3000 1.18842018409756
};
\end{axis}

\end{tikzpicture}
 \vspace{-2em} \caption{sequential (CPU) / sequential (GPU)}
	\end{subfigure}\label{fig:seqgpu_vs_seqcpu}
	\hfill
	\vspace*{-0.1cm}
	\caption{Comparison of relative runtime in seconds when running \STVGP with the sequential filter/smoother on CPU and GPU and the parallel filter/smoother on GPU. The parallel algorithm outperforms the others when the number of spatial points is small, otherwise the sequential (GPU) method is best. The sequential (CPU) algorithm is competitive when 20 spatial points are used.}
	\label{fig:profiling}
\end{figure*}
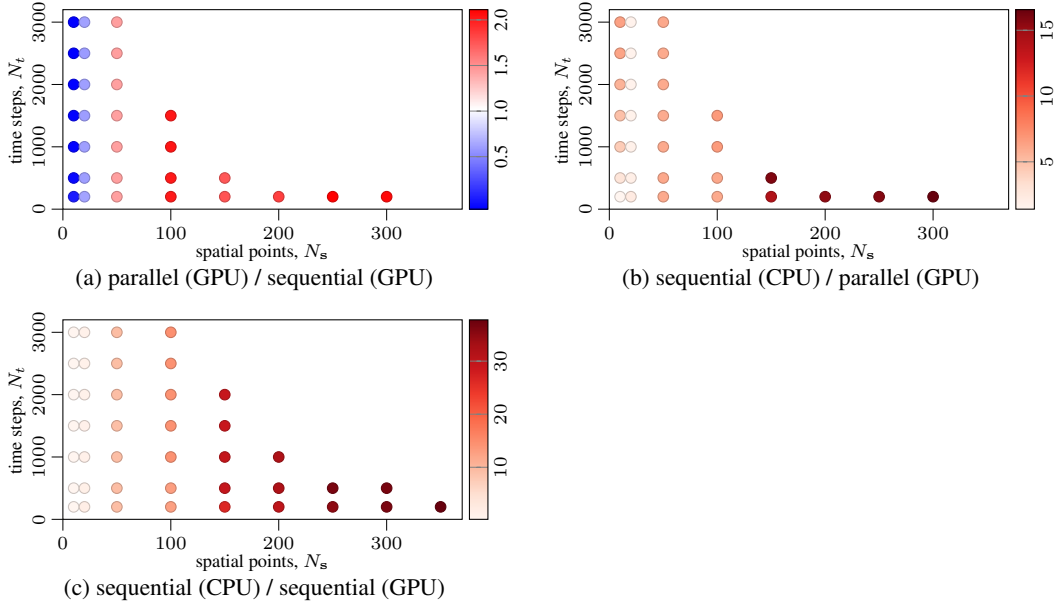

\begin{table}[h]
	\scriptsize
	\setlength{\tabcolsep}{5.5pt}
	\renewcommand{\arraystretch}{1.25}
	\begin{center}
		\begin{sc}
			\caption{\textbf{Sequential (CPU)}: average training step time (secs) for sequential filter/smoother on CPU.}
			\begin{tabularx}{\textwidth}{|p{0.25cm}|p{0.5cm}|p{0.65cm}|p{0.65cm}|p{0.65cm}|p{0.65cm}|p{0.65cm}|p{0.65cm}|p{0.65cm}|p{0.65cm}|p{0.65cm}|}
				\cline{1-11}
				\multicolumn{2}{|c|}{} & \multicolumn{9}{c|}{Spatial points, $N_\Space$} \\   \cline{3-11}
				\multicolumn{2}{|c|}{} & 10 & 20 & 50 & 100 & 150 & 200 & 250 & 300 & 350  \\ \cline{1-11}
				\multirow{8}{*}{\begin{sideways}$\,$Time steps, $N_t$ \end{sideways}} & 200 & 0.05 & 0.48 & 5.61  & 14.47 & 43.28 & 66.17 & 96.17  & 136.83 & 173.17  \\   \cline{2-11}
				& 500  & 0.12 & 0.97 & 14.23  & 36.93 & 121.43 & 176.67 & 252.40  & 344.02 &     \\ \cline{2-11}
				& 1000  & 0.27 & 1.77 & 28.15  & 81.34 & 239.29 & 353.62 &   &  &     \\ \cline{2-11}
				& 1500  & 0.40 & 2.54 & 41.69  & 122.86 & 357.36 &  &   &  &     \\ \cline{2-11}
				& 2000  & 0.52 & 3.32 & 56.24  & 164.59 & 475.51 &  &   &  &     \\ \cline{2-11}
				& 2500  & 0.65 & 4.06 & 70.62  & 203.48 &  &  &   &  &     \\ \cline{2-11}
				& 3000  & 0.78 & 4.91 & 86.14  & 247.33 &  &  &   &  &    \\ \cline{1-11}
			\end{tabularx}   	\label{table:profiling1}
			\caption{\textbf{Sequential (GPU)}: average training step time (secs) for sequential filter/smoother on GPU.}
			\begin{tabularx}{\textwidth}{|p{0.25cm}|p{0.5cm}|p{0.65cm}|p{0.65cm}|p{0.65cm}|p{0.65cm}|p{0.65cm}|p{0.65cm}|p{0.65cm}|p{0.65cm}|p{0.65cm}|}
				\cline{1-11}
				\multicolumn{2}{|c|}{} & \multicolumn{9}{c|}{Spatial points, $N_\Space$} \\   \cline{3-11}
				\multicolumn{2}{|c|}{} & 10 & 20 & 50 & 100 & 150 & 200 & 250 & 300 & 350  \\ \cline{1-11}
				\multirow{8}{*}{\begin{sideways}$\,$Time steps, $N_t$ \end{sideways}} & 200 & 0.24 & 0.28 & 0.61  & 1.14 & 1.65 & 2.21 & 2.76  & 3.77 & 4.58  \\   \cline{2-11}
				& 500  & 0.62 & 0.70 & 1.52  & 2.87 & 4.16 & 5.53 & 6.92  & 9.43 & ---    \\ \cline{2-11}
				& 1000  & 1.19 & 1.40 & 3.05  & 5.71 & 8.15 & 11.05 & ---  & --- & ---    \\ \cline{2-11}
				& 1500  & 1.86 & 2.11 & 4.58  & 8.60 & 12.37 & --- & ---  & --- & ---    \\ \cline{2-11}
				& 2000  & 2.44 & 2.79 & 6.11  & 11.47 & 16.38 & --- & ---  & --- & ---    \\ \cline{2-11}
				& 2500  & 3.01 & 3.55 & 7.62  & 14.26 & --- & --- & ---  & --- & ---    \\ \cline{2-11}
				& 3000  & 3.56 & 4.13 & 9.16  & 17.03 & --- & --- & ---  & --- & ---    \\ \cline{1-11}
			\end{tabularx}   \label{table:profiling2}
			\caption{\textbf{Parallel (GPU)}: average training step time (secs) for parallel filter/smoother on GPU.}
			\begin{tabularx}{\textwidth}{|p{0.25cm}|p{0.5cm}|p{0.65cm}|p{0.65cm}|p{0.65cm}|p{0.65cm}|p{0.65cm}|p{0.65cm}|p{0.65cm}|p{0.65cm}|p{0.65cm}|}
				\cline{1-11}
				\multicolumn{2}{|c|}{} & \multicolumn{9}{c|}{Spatial points, $N_\Space$} \\   \cline{3-11}
				\multicolumn{2}{|c|}{} & 10 & 20 & 50 & 100 & 150 & 200 & 250 & 300 & 350  \\ \cline{1-11}
				\multirow{8}{*}{\begin{sideways}$\,$Time steps, $N_t$ \end{sideways}} & 200 & 0.04 & 0.21 & 0.94  & 2.42 & 3.08 & 4.31 & 6.13  & 8.26 & --  \\   \cline{2-11}
				& 500  & 0.04 & 0.50 & 2.35  & 6.07 & 7.68 & --- & --- & ---  & ---     \\ \cline{2-11}
				& 1000  & 0.06 & 1.00 & 4.70  & 12.19 & --- & --- & ---  & --- & ---    \\ \cline{2-11}
				& 1500  & 0.08 & 1.49 & 7.05  & 18.28 & --- & --- & ---  & --- & ---   \\ \cline{2-11}
				& 2000  & 0.09 & 2.00 & 9.39  & --- & --- & --- & ---  & --- & ---    \\ \cline{2-11}
				& 2500  & 0.10 & 2.49 & 11.77  & --- & --- & --- & ---  & --- & ---    \\ \cline{2-11}
				& 3000  & 0.12 & 2.99 & 14.13  & --- & --- & --- & ---  & --- & ---    \\ \cline{1-11}
			\end{tabularx}   \label{table:profiling3}
		\end{sc}
	\end{center}
\end{table}

\section{Reformulation of the Spatio-Temporal State Space Model for Spatial Mean-Field}
\label{sec_app:sde_reformulation}

The process noise covariance of the state $\state(t)$ for a spatio-temporal GP is (see \cref{sec:state-space-spatio-temporal}) $\MQ_\timeindex=\MK_{\SPACE\SPACE}^{(\Space)} \otimes \MQ_\timeindex^{(t)}$, where $\MK_{\SPACE\SPACE}^{(\Space)}$ is the spatial kernel evaluated at inputs $\MR$, and $\MQ_\timeindex^{(t)}$ is the process noise covariance of the state-space model induced by the temporal kernel. Similarly, the stationary distribution is given by $\MP_\infty=\MK_{\SPACE\SPACE}^{(\Space)} \otimes \MP_{\infty}^{(t)}$. Letting $\MC_{\SPACE\SPACE}^{(\Space)}$ be the Cholesky decomposition of $\MK_{\SPACE\SPACE}^{(\Space)}$, we can use the Kronecker identities from \cref{app:kronecker} to rewrite the stationary state covariance as,
\begin{align}
\SpaceCov \otimes \MP_{\infty}^{(t)} &= (\SpaceCov \MI_{N_\Space}) \otimes (\MI_{d_t} \MP_{\infty}^{(t)}) \nonumber \\
&= (\SpaceCov \otimes \MI_{d_t}) (\MI_{N_\Space} \otimes \MP_{\infty}^{(t)}) \nonumber \\
&= (\MC_{\SPACE\SPACE}^{(\Space)} \otimes \MI_{d_t}) (\MI_{N_\Space} \otimes \MP_{\infty}^{(t)}) (\MC_{\SPACE\SPACE}^{(\Space)} \otimes \MI_{d_t})^\top,
\end{align}
and hence, recalling that the measurement model is $\MH=\MI_{N_\Space} \otimes \MH^{(t)}$, the prior covariance of the function, $\vf_\timeindex=\MH \vs(t_\timeindex)$, is given by,
\begin{align}
\text{Cov}[\vf_\timeindex] &= (\MI_{N_\Space} \otimes \MH^{(t)}) (\SpaceCov \otimes \MP_{\infty}^{(t)} ) (\MI_{N_\Space} \otimes \MH^{(t)})^\top \nonumber \\
&= (\MI_{N_\Space} \otimes \MH^{(t)}) (\MC_{\SPACE\SPACE}^{(\Space)} \otimes \MI_{d_t}) (\MI_{N_\Space} \otimes \MP_{\infty}^{(t)}) (\MC_{\SPACE\SPACE}^{(\Space)} \otimes \MI_{d_t})^\top (\MI_{N_\Space} \otimes \MH^{(t)})^\top \nonumber \\
&= \underbrace{(\MC_{\SPACE\SPACE}^{(\Space)} \otimes \MH^{(t)})}_{\text{Measurement model, } \MH} \underbrace{(\MI_{N_\Space} \otimes \MP_{\infty}^{(t)})}_{\MP_{\infty}} (\MC_{\SPACE\SPACE}^{(\Space)} \otimes \MH^{(t)})^\top .
\end{align}

We see from the above that the contribution from the spatial kernel can be included as part of the measurement model, $\MH$, rather than the stationary state covariance. The process noise covariance $\MQ_\timeindex$ can be deconstructed in a similar way (for stationary kernels, $\MQ_\timeindex=\MP_{\infty} - \MA_\timeindex \MP_{\infty} \MA_\timeindex^\T$).
Arguably, as discussed in \cref{sec:mean-field}, this presentation of the model is more intuitive since it becomes clear that $N_{\Space}$ latent processes, each with an independent GP prior, are correlated via a measurement model in which the spatial covariance mixes the latent processes to generate the observations.

\paragraph{Sparse Spatial Mean-Field} A similar argument holds for the sparse version of the model (\STSVGP):
\begin{equation}
\MK_{\MZ_{\Space}\MZ_{\Space}}^{(\Space)} \otimes \MP_{\infty}^{(t)} = (\MC_{\MZ_{\Space}\MZ_{\Space}}^{(\Space)} \otimes \MI_{d_t}) (\MI_{M_\Space} \otimes \MP_{\infty}^{(t)}) (\MC_{\MZ_{\Space}\MZ_{\Space}}^{(\Space)} \otimes \MI_{d_t})^\top,
\end{equation}
where $\MC^{(\Space)}_{\MZ_\Space\MZ_\Space}$ is the Cholesky factor of $\MK^{(\Space)}_{\MZ_\Space\MZ_\Space}$. Then,
\begin{align}
\text{Cov}[\vf_\timeindex] &= ([\MK^{(\Space)}_{\SPACE\MZ_\Space}\MK^{-(\Space)}_{\MZ_\Space\MZ_\Space}] \otimes \MH^{(t)}) (\MK_{\MZ_{\Space}\MZ_{\Space}}^{(\Space)} \otimes \MP_{\infty}^{(t)} ) ([\MK^{(\Space)}_{\SPACE\MZ_\Space}\MK^{-(\Space)}_{\MZ_\Space\MZ_\Space}] \otimes \MH^{(t)})^\top \nonumber \\
&= ([\MK^{(\Space)}_{\SPACE\MZ_\Space}\MK^{-(\Space)}_{\MZ_\Space\MZ_\Space}] \otimes \MH^{(t)}) (\MC_{\MZ_{\Space}\MZ_{\Space}}^{(\Space)} \otimes \MI_{d_t}) (\MI_{M_\Space} \otimes \MP_{\infty}^{(t)}) (\MC_{\MZ_{\Space}\MZ_{\Space}}^{(\Space)} \otimes \MI_{d_t})^\top ([\MK^{(\Space)}_{\SPACE\MZ_\Space}\MK^{-(\Space)}_{\MZ_\Space\MZ_\Space}] \otimes \MH^{(t)})^\top \nonumber \\
&= \underbrace{([\MK^{(\Space)}_{\SPACE\MZ_\Space}\MC^{-(\Space)}_{\MZ_\Space\MZ_\Space}] \otimes \MH^{(t)})}_{\text{Measurement model, } \MH} \underbrace{(\MI_{M_\Space} \otimes \MP_{\infty}^{(t)})}_{\MP_{\infty}} ([\MK^{(\Space)}_{\SPACE\MZ_\Space}\MC^{-(\Space)}_{\MZ_\Space\MZ_\Space}] \otimes \MH^{(t)})^\top .
\end{align}
Again, this reformulation represents the same model as the standard form (and gives identical results), but enables the sparse and mean-field approximations to be combined since $\MP_{\infty}$ is now block-diagonal.

\section{Exponential Family -- Multivariate Gaussian Distribution}
\label{sec_app:exp_family}

A Gaussian distribution $q(\vu) = \N(\vu \mid \vm , \postcov)$ is part of the exponential family with
\begin{align}
	\param &= (\vm, \postcov), \\
	\natp &= (\postcov^{-1} \vm, -\frac{1}{2} \postcov^{-1} ), \\
	\meanp &= (\vm, \vm \vm^\top + \postcov),
\end{align}
where $\param$ are the moment parameters, $\natp$ are the natural parameters, and $\meanp$ are the expectation parameters. For further information see \citep{wainwright_exp_family:2008}. For completeness we provide a table to convert between the above parameterisations:

\begin{table}[h]
\caption{Table of conversions between exponential family parameterisations}
\label{table:appendix_gaussian_param_convert}
\vskip 0.15in
\begin{center}
\begin{scriptsize}
\begin{sc}
\begin{tabularx}{\textwidth}{lccc}
\toprule
	 & $\rightarrow \param$ & $\rightarrow \natp$ & $\rightarrow \meanp$ \\
\midrule
	$\param$  & -- & $[\param^{-1}_2 \param_1, - \frac{1}{2} \param^{-1}_2]$ & $[\param_1, \param_1 \param^T_1 + \param_2]$ \\
	$\natp$ & $[\left[ -2\natp_2 \right]^{-1} \natp_1, \left[ -2\natp_2 \right]^{-1} ]$  & -- & $[\left[ -2\natp_2 \right]^{-1} \natp_1, (\left[ -2\natp_2 \right]^{-1} \natp_1)^2 + \left[ -2\natp_2 \right]^{-1}]$\\
	$\meanp$    & $[\meanp_1, \meanp_2-\meanp_1 \meanp_1^T]$ & $[[\meanp_2-\meanp_1 \meanp_1^T]^{-1}\meanp_1, -\frac{1}{2} [\meanp_2-\meanp_1 \meanp_1^T]^{-1}]$ & -- \\
\bottomrule
\end{tabularx}
\end{sc}
\end{scriptsize}
\end{center}
\vskip -0.1in
\end{table}

\section{Alternative Derivation of \CVI Update Equations}
\label{sec_app:cvi_natgrad_equal}

We now show that after a natural gradient step the variational distribution can be decomposed as a conjugate Bayesian step with the model prior and an approximate likelihood. 

\subsection{\CVI Update}
Applying the chain rule to the  approximate likelihood parameters of the \CVI update in \cref{eq:cvi_approx_nat_update} results in:
\begin{align}
	\apxnatp^{(1)}_{t+1} &= (1-\beta) \apxnatp^{(1)}_{t} + \beta \cdot (\gradm - 2\grads \vm)  \label{eqn:app_cvi_update_1} \\
	\apxnatp^{(2)}_{t+1} &= (1-\beta) \apxnatp^{(2)}_{t} + \beta \cdot \grads \label{eqn:app_cvi_update_2}
\end{align}
where, as defined in the main paper,
\begin{align}
	\gradm &= \diff{\E_q \left[ \log p(\MY \mid \vf)\right]}{\vm} \label{eqn:app_grad_m}\\
	\grads &= \diff{\E_q \left[ \log p(\MY \mid \vf)\right]}{\postcov} \label{eqn:app_grad_s}
\end{align}
We now show that we recover \cref{eqn:app_cvi_update_1,eqn:app_cvi_update_2} from standard natural gradients from \cite{hensman_gp_for_big_data:2013}. Recall from \cref{eq:natgrad_hensman} the natural gradient is given by
\begin{equation}
	\natp_{t+1} = \natp_t + \beta \diff{\LL}{\meanp}
\end{equation}
and applying chain rule:
\begin{align}
	\natp^{(1)}_{t+1} &= \natp^{(1)}_t + \beta \left( \diff{\LL}{\vm} \ - 2 \diff{\LL}{\postcov}   \vm \right) \label{eqn:nat_update_1}\\
	\natp^{(2)}_{t+1} &= \natp^{(2)}_t + \beta \left(\diff{\LL}{\postcov} \right) \label{eqn:nat_update_2}
\end{align}
For ease of presentation we consider both natural parameters separately and for both will need the following result: 

\begin{lemma}
Recursions of the form:
\begin{align}
	R_{t+1} = (1-\beta)R_t + \beta b_t + \beta a
\end{align}
where $R_1 = a$ can be rewritten as:
\begin{equation}
	R_{t+1} = \widetilde{R}_{t+1} + a
\end{equation}
where
\begin{equation}
	\widetilde{R}_{t+1} = (1-\beta) \widetilde{R}_t + \beta b_t
\end{equation}
with $\widetilde{R}_1 = 0$.
\label{lemma:simplify_recursions}
\end{lemma}
\begin{proof}
The proof follows by induction. Using the fact that $R_1=a$:
\begin{align}
	R_2 &= (1-\beta) a + \beta b_1 + \beta a \nonumber \\
	&= (1-\beta) \widetilde{R}_1 + \beta b_1 + a \nonumber \\
	&= \widetilde{R}_2 + a
\end{align}	
with $\widetilde{R}_2=(1-\beta) \widetilde{R}_1 + \beta b_1$, and $\widetilde{R}_1$ = 0. In the next step:
\begin{align}
	R_3 &= (1-\beta) R_2 + \beta b_2 + \beta a \nonumber \\
	&= (1-\beta) (\beta b_1 + a) + \beta b_2 + \beta a \nonumber \\
	&= (1-\beta) (\beta b_1) + \beta b_2 + a \nonumber \\
	&= (1-\beta) \widetilde{R}_2 + \beta b_2 + a \nonumber \\
	&= \widetilde{R}_3 + a
\end{align}
where $\widetilde{R}_3=(1-\beta) \widetilde{R}_2 + \beta b_2$. In the general case:
\begin{align}
	R_{t+1} &= (1-\beta) R_t + \beta b_t + \beta a \nonumber \\
	&= (1-\beta)  (\widetilde{R}_{t} + a) + \beta b_t + \beta a \nonumber \\
	&= (1-\beta)\widetilde{R}_{t} + \beta b_t + a \nonumber \\
	&= \widetilde{R}_{t+1} + a
\end{align}
We have shown that the lemma holds on the first step and in a general step and so the proof holds by induction.
\end{proof}

\subsubsection{First Natural Parameter}

We first show that $\natp^{(1)}_{t+1}$ can be computed efficiently using the \CVI updates. First; substituting \cref{eqn:app_grad_m,eqn:app_grad_s} into \cref{eqn:nat_update_1}:
\begin{align}
	\natp^{(1)}_{t+1} &= \natp^{(1)}_t + \beta \left( \gradm - \MK^{-1} \vm - 2 \left[ \grads - \frac{1}{2} \left[ - \postcov^{-1} + \MK^{-1} \right] \right] \vm\right) \nonumber \\
	&= \natp^{(1)}_t + \beta \left(\gradm - 2 \grads \vm  - \postcov^{-1} \vm \right)
\end{align}
Substituting $\natp^{(1)}_t = \postcov^{-1} \vm$ and adding $\priornatp^{(1)}=0$:
\begin{align}
	\natp^{(1)}_{t+1} &= (1-\beta) \natp^{(1)}_t + \beta \left(\gradm - 2 \grads \vm \right) + \beta \priornatp^{(1)}
\end{align}
Applying \ulemma \ref{lemma:simplify_recursions} we can directly rewrite the recursion as:
\begin{equation}
	\natp^{(1)}_{t+1} = \apxnatp^{(1)}_{t} + \priornatp^{(1)} ~~ \text{where} ~~ \apxnatp^{(1)}_t = (1-\beta) \apxnatp^{(1)}_{t-1} + \beta (\gradm - 2 \grads \vm)
\end{equation}
and $\apxnatp^{(1)}_1 = 0$ and  $\natp^{(1)}_1=\priornatp^{(1)} = 0$. This recovers the \CVI update in \cref{eqn:app_cvi_update_1}.

\subsubsection{Second Natural Parameter}

Following the steps for the first natural parameters we first substitute \cref{eqn:app_grad_s} into \cref{eqn:nat_update_2}:
\begin{align}
	\natp^{(2)}_{t+1} &= \natp^{(2)}_{t} + \beta \left( \grads - \frac{1}{2} \left[ - \postcov^{-1} + \MK^{-1} \right] \right) \nonumber \\
	&= \natp^{(2)}_{t} + \beta \frac{1}{2} \postcov^{-1} + \beta \left(\grads - \frac{1}{2}\MK^{-1} \right) 
\end{align}
substituting $\natp^{(2)}_t= - \frac{1}{2}\postcov^{-1}$ and $\priornatp^{(2)}_t = - \frac{1}{2}\MK^{-1}$:
\begin{equation}
	\natp^{(2)}_{t+1} = (1-\beta) \natp^{(2)}_t + \beta \grads + \beta \priornatp^{(2)}_t
\end{equation}
Applying \ulemma \ref{lemma:simplify_recursions} the above simplifies to:
\begin{equation}
	\natp^{(2)}_{t+1} = \apxnatp^{(2)}_t + \priornatp^{(2)} ~~ \text{where} ~~ \apxnatp^{(2)}_t = (1-\beta) \apxnatp^{(2)}_{t-1} + \beta \grads
\end{equation}
and $\apxnatp^{(2)}_1=0$ and $\natp^{(2)}_1=\priornatp^{(2)}$. This recovers the \CVI update in \cref{eqn:app_cvi_update_2}.

\section{Kronecker Structured Gaussian Marginals}\label{sec:appendix_kronecker_svgp_marginal}

The marginal $q(\vf) = \intg{p(\vf \mid \vu)q(\vu)}{\vu}{}{}$ is a Gaussian of the form:
\begin{equation}
	q(\vf) = \N(\vf \mid  \MA \qum, \Kxx - \MA \Kzx + \MA \qus \MA^T) ~~ \text{where} ~~ \MA = \Kxz \iKzz.
\end{equation}
When $\MK$ can be written as a Kronecker product the above can be simplified. 
\begin{lemma}
	Let $\Kxx = \tKxx \kron \sKxx$, $\Kxz = \tKxz \kron \sKxz$, $\Kzz = \tKxx \kron \sKzz$ then $q(\vf)$ can be decomposed as $q(\vf) = \N(\vf \mid \qfm, \qfs)$ where,
	\begin{align}
		\qfm &= \left[ \MI \kron (\sKxz \siKzz) \right] \qum \\
		\qfs &= \left[ \tKxx \kron \left( \sKxx - \sKxz \siKzz \sKzx \right) \right] + \left[ \MI \kron \left( \sKxz \siKzz \right) \right] \qus  \left[ \MI \kron \left( \siKzz \sKzx \right) \right]
	\end{align}
 
\label{lemma:kronecker_marginal}
\end{lemma}
\begin{proof}
	Starting with $\qfm$:
	\begin{align}
		\qfm &= \Kxz \iKzz \qum & \nonumber \\
			 &= \big[ ( \tKxx \tiKxx) \kron (\sKxz \siKzz) \big] \qum & \text{apply \cref{eqn:kron_mixed_property}} \nonumber \\
			 &= \big[ \MI \kron (\sKxz \siKzz) \big] \qum .
	\end{align}
	And now dealing with $\qfs$. Let $G = \Kxx - \MA \Kzx$, $K = \MA \qus \MA^T$. From the above we have shown that $\Kxz \iKzz = \big[ \MI \kron (\sKxz \siKzz) \big]$. First we substitute this into $G$:
	\begin{align}
		G &= \Kxx - \big[ \MI \kron (\sKxz \siKzz) \big] \Kzx   &  \nonumber \\
		&= (\tKxx \kron \sKxx) - \big[ \MI \kron (\sKxz \siKzz) \big] \big[ \tKxx \kron \sKzx \big] \nonumber \\
		&= (\tKxx \kron \sKxx) - \big[ \tKxx \kron (\sKxz \siKzz \sKzx) \big] & \text{apply \cref{eqn:kron_mixed_property}} \nonumber \\
		&= \tKxx \kron \big( \sKxx - \sKxz \siKzz \sKzx \big) & \text{apply \cref{eqn:kron_addition_lhs}}
	\end{align}
	And now substituting into $K$:
	\begin{align}
		K = \MA \qus \MA^T &= \big[ \MI \kron (\sKxz \siKzz) \big] \qus \big[ \MI \kron (\sKxz \siKzz) \big]^T \nonumber \\
		&= \big[ \MI \kron (\sKxz \siKzz) \big] \qus \big[ \MI \kron (\siKzz \sKzx ) \big]
	\end{align}
	Combining $G, K$:
	\begin{align}
		\qfs &= G + K \nonumber \\	
		&= \tKxx \kron \big( \sKxx - \sKxz \siKzz \sKzx \big) + \big[ \MI \kron (\sKxz \siKzz) \big] \qus \big[ \MI \kron (\siKzz \sKzx ) \big]
	\end{align}
	which completes the proof.
\end{proof}
When the likelihood factorises across observations only the marginal $q(\qfm_\dindex)$ is required to compute the expected log likelihood. The marginal can be efficiently computed by utilising the fact that $\MI \kron (\sKzx \siKzz)$ is block-diagonal, where there are $\Nt$ blocks each of size $\Ms \times \Ms$. 
\begin{lemma}
	Following \ulemma \ref{lemma:kronecker_marginal} the marginal $q(\vf_\dindex)$ is a Gaussian: $q(\vf_\dindex) = \N(\vf_\dindex \mid \qfm_\dindex, \qfs_\dindex)$ where
	\begin{align}
		\qfm_\dindex &= \sKxzn \siKzzn \qum_{\timeindex} \\
		\qfs_\dindex &= \sigma^2_t  \sKxxn + \sKxzn \siKzzn  \big[ - \sigma^2_t \sKzxn +  \qus_{\timeindex} \sKzzn \sKxzn \big] .
	\end{align}
\label{lemma:kronecker_marginal_n}
\end{lemma}
\begin{proof}
	To denote a single observation we subscript by $\dindex$ and let $\timeindex$ denote the matrix/vector of elements at  the $\timeindex$'th time step. First dealing with the mean:
	\begin{align}
		\qfm_\dindex &= \left[ \left[ \MI \kron (\sKxz \siKzz) \right] \qum \right]_\dindex \nonumber \\
		&= \left[ \MI \kron (\sKxz \siKzz) \right]_{\timeindex} \qum_{\timeindex} \nonumber \\
		&= \sKxzn \siKzzn \qum_{\timeindex} .
	\end{align}
	Where the second line holds due to $\MI \kron (\sKzx \siKzz)$ being block diagonal and so each block affects separate elements of $\qum$. The last line simply selects the relevant block from the block diagonal matrix. Deriving the form of the variance follows the same steps:
	\begin{align}
		\qfs_\dindex &= \left[ \qfs \right]_{\dindex} \nonumber \\
		&= \tKxxn \cdot \left[ \sKxx - \sKxz \siKzz \sKzx \right]_{\timeindex} + \left[ \MI \kron \left( \sKxz \siKzz \right) \right]_\timeindex \qus_\timeindex  \left[ \MI \kron \left( \siKzz \sKzx \right) \right]_\timeindex \nonumber \\
		&= \sigma^2_t \cdot \left[ \sKxxn - \sKxzn \siKzzn \sKzxn \right] + \sKxzn \sKzzn \qus_{\timeindex} \sKzzn \sKxzn \nonumber \\
		&= \sigma^2_t  \sKxxn + \sKxzn \siKzzn  \big[ - \sigma^2_t \sKzxn +  \qus_{\timeindex} \sKzzn \sKxzn \big] 
	\end{align}
	which completes the proof.
\end{proof}
\section{Block Diagonal Approximate Likelihood Natural Parameters} \label{appendix:cvi_block_diagional}

We now turn to the form of $\gradm$ and $\grads$. The exact value of these can be easily calculated in any automatic differentiation library, but to use \CVI we need to know where the non-zero elements of $\grads$ are.
\begin{lemma}
	The form of $\grads$ is block diagonal where there are $\Nt$ blocks each of size $\Ms \times \Ms$.
	\label{lemma:grads_block_diagonal}
\end{lemma}
\begin{proof}
	The partial derivative of the expected log likelihood is:
	\begin{equation}
		\grads = \diff{\E_{q(\vf)} \left[ \log p(\MY \mid \vf) \right]}{\qus} = \sum^{\Nt}_\timeindex \sum^{\Ns}_\spaceindex \diff{\E_{q(\vf_\dindex)} \left[ \log p(\MY_\dindex \mid \vf_\dindex) \right]}{\qus} .
	\end{equation}
	Applying chain rule:
	\begin{equation}
		\grads = \sum^{\Nt}_\timeindex \sum^{\Ns}_\spaceindex  \diff{\E_{q(\vf_\dindex)} \left[ \log p(\MY_\dindex \mid \vf_\dindex) \right]}{\qfs_{\dindex}} \diff{\qfs_{\dindex}}{\qus}.
	\label{eqn:grad_ell_grad_s}
	\end{equation}
	The first term is a scalar and so does not affect the final form. The second term is scalar-matrix derivative:
	\begingroup
	\renewcommand*{\arraystretch}{2} %
	\begin{equation}
		\diff{\qfs_{\dindex}}{\MS} = \left[ \begin{matrix}
			\diff{\qfs_{\dindex}}{\qus_{1,1}} & \diff{\qfs_{\dindex}}{\qus_{2,1}} &  \cdots & \diff{\qfs_{\dindex}}{\qus_{M,1}} \\
			\diff{\qfs_{\dindex}}{\qus_{1,2}} & \diff{\qfs_{\dindex}}{\qus_{2,2}} & \cdots & \diff{\qfs_{\dindex}}{\qus_{M,2}}\\
			\vdots & \vdots & \ddots & \vdots \\
			\diff{\qfs_{\dindex}}{\qus_{M,1}} & \diff{\qfs_{\dindex}}{\qus_{M,2}} & \cdots  & \diff{\qfs_{\dindex}}{\qus_{M,M}}
			\end{matrix} \right]
	\label{eqn:grad_ell_full_partial_s_s}
	\end{equation}
	\endgroup
	The inducing locations $\MZ=\vec([\MZ_\timeindex]^{\Nt}_\timeindex)$ are organised in time blocks and thus $\qus$ is organised into time blocks. It is clear that only elements in \cref{eqn:grad_ell_full_partial_s_s} that correspond to the same time index as $\timeindex$ will be non-zero, because $\qfs_{\dindex}$ only depends on $\qus_\timeindex$. Due to the structure of $\MZ$ these non-zero elements will be one of the $\Nt \times \Nt$ blocks on the block diagonal. The sum in \cref{eqn:grad_ell_grad_s} iterates over every $\dindex$ and so the resulting matrix with have non-zero entries only in the block diagonal.

\end{proof} 
\section{Further Details on Experiments} \label{sec:app_further_experiment_details}

\subsection{Metrics Used} \label{sec:appendix_metrics}

Let $\mathbf{Y} \in \R^{N \times 1}$ be the true value of the test data and $\vmu \in \R^{N \times 1}$, $\vxi \in \R^{N \times 1}$ be the predicted mean and variance, then we report,
\begin{align}
	\text{Root mean square error (RMSE)} &= \sqrt{\frac{1}{N} \sum^N_{n=1} (\mathbf{Y}_n - \vmu_n)^2 } , \\
	\text{Negative log predictive density (NLPD)} &= \frac{1}{N} \sum^N_{n=1} \log \int p(\MY_n \mid \vf_n) \N(\vf_n \mid \vmu_n, \vxi_n) \, \text{d}\vf_n .
\end{align}
With a Gaussian likelihood we make use of closed form solutions to the NLPD, otherwise we rewrite the NLPD as a LogSumExp function and approximate using quadrature with 100 points.

\subsection{Pseudo-periodic Functions}
\label{sec:app_pseudo_periodic_function}

We construct toy datasets based on pseudo-periodic functions \cite{pseudo_periodic_data}:
\begin{equation}\label{eq:per-func}
	\phi(t, c) = \sum^7_{i=3} \frac{1}{2^i} \sin \left( 2 \pi \cdot (2^{2+i} + s_i) \cdot t \cdot c \right) ,
\end{equation}
where $s_i$ are samples from a uniform distribution between 0 and $2^i$ . These functions appear periodic but are never exactly repeating. The ground truth generative model is then defined as $f(t, r) = 50 \phi(t, 3) \sin (4 \pi r)$, with a Gaussian likelihood $y  = f(t, r) + \varepsilon$, $\varepsilon \sim \N(0, 0.01)$.

\subsection{Computational Infrastructure}

The experiments were run across various computational infrastructures.

\paragraph{Run time experiment (\cref{fig:scalability})} These experiments were run on an Intel Xeon E5-2698 v4 2.2 GHz CPU. 

\paragraph{Synthetic Experiment} These experiments were run using 8, Intel Xeon CPU E5-2680 v4 @ 2.4 GHz, CPUs.

\paragraph{Real World Experiments} These experiments were run on an Intel Xeon Gold 6248 @ 2.5 GHz CPU or an NVIDIA Tesla V100 GPU.

\subsection{Baselines}
 
We compare against two baselines, \SVGP \citep{hensman_gp_for_big_data:2013}, \SKI \citep{wilison_ski_gp:2015}:
 
\paragraph{\SVGP:} We use the implementation provided in GPFlow \cite{gpflow:2017}.

\paragraph{\SKI:} We use the implementation provided in GPyTorch \cite{gardner_gpytorch:2018}. We construct a \textit{GridInterpolationKernel} and run with the default grid size or by matching the dimensions of the grid to the corresponding \SVGP.
 
\subsection{Synthetic Experiment}

For all models we use 6 spatial inducing points (or an equivalent grid of inducing points), with the spatial locations initialised through K-means. We initialise the likelihood noise to $0.1$, use a Mat\'ern-$\nicefrac{3}{2}$ kernel with lengthscale and variance of $0.1$ and $1.0$ respectively across all input dimensions, and run for 500 training iterations or one hour, whichever is shortest.

\begin{figure}[H]
	 \input{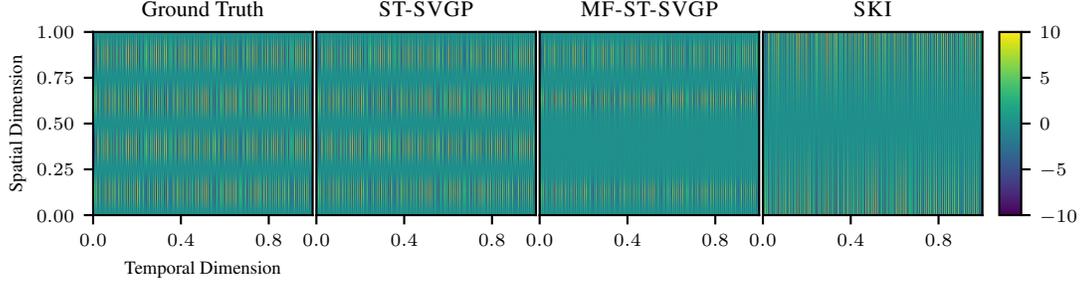}
	 
	 \vspace{-0.4cm}
	 \caption{Test predictions on the synthetic experiment, dataset number 5. The ground truth (first panel) displays rich structure in the temporal dimension, but is smooth in the spatial dimension. Most of the models are able to capture some temporal structure but only \STSVGP is able to accurately recover the ground truth.}
	 \label{fig:app_shutters}
	\vspace*{-0.5em}	 
\end{figure}

We report the RMSE results on the synthetic experiment, detailed in the main paper, in \cref{table:shutters_rmse}. Each model and dataset combination is run five times with a different random seed for the data generation, and the reported means and standard deviations are calculated across these runs in \cref{table:shutters_rmse}. All experiments improve with the increasing training dataset size. The \SVGP does not improve at the same rate as \STSVGP because it very quickly reaches the one hour time limit and so is not trained beyond this. %
\cref{fig:app_shutters} shows the posterior predictive mean for various models on dataset number 5.

\begin{table}[H]
\caption{Synthetic experiment: test RMSE. The size of the training data and its similarity to the test data increase with the dataset number. The mean and standard deviation across five runs is shown, with the data given by five random draws from the generative model given in \cref{eq:per-func}.}
\label{table:shutters_rmse}
\vskip 0.05in
\footnotesize
\setlength{\tabcolsep}{3pt}
\begin{center}
\begin{scriptsize}
\begin{sc}
\begin{tabular}{lccccccc}
\toprule
	Model & 1 & 2 & 3 & 4 & 5 & 6 & 7\\
\midrule
\STCVI  &  4.86 $\pm$ 0.38 & 4.59 $\pm$ 0.21 & 4.42 $\pm$ 0.29 & 3.22 $\pm$ 0.45 & 2.49 $\pm$ 0.07 & 0.45 $\pm$ 0.11 & 0.85 $\pm$ 0.03  \\
\SVGP  &  4.95 $\pm$ 0.38 & 4.61 $\pm$ 0.28 & 4.30 $\pm$ 0.45 & 3.92 $\pm$ 0.10 & 3.78 $\pm$ 0.25 & 3.56 $\pm$ 0.03 & --  \\
\MFSTCVI  &  4.91 $\pm$ 0.38 & 4.63 $\pm$ 0.21 & 4.52 $\pm$ 0.24 & 3.14 $\pm$ 0.36 & 2.71 $\pm$ 0.19 & 1.39 $\pm$ 0.82 & 2.13 $\pm$ 0.04  \\
\SKI  &  3.73 $\pm$ 0.07 & 3.69 $\pm$ 0.03 & 3.71 $\pm$ 0.01 & 3.57 $\pm$ 0.07 & 3.46 $\pm$ 0.02 & 3.34 $\pm$ 0.01 & 3.58 $\pm$ 0.01  \\
\bottomrule
\end{tabular}
\end{sc}
\end{scriptsize}
\end{center}
\vskip -0.1in
\end{table}

\subsection{Comparison of Approximations, \cref{fig:scalability}} \label{app:approx-comparison}

In \cref{fig:scalability} we study the density of a single tree species, \emph{Trichilia tuberculata}, from a 1000$\,$m $\times$ 500$\,$m region of a rainforest in Panama \citep{Condit:2005}. We use a 5$\,$m binning ($N_t=200$) for the first spatial dimension (which we treat as time, $t$), and a varying bin size for the second spatial dimension (which we treat as space, $\Space$). The total number of data points is therefore $N=N_tN_\Space=200N_\Space$. We model the resulting count data use a log-Gaussian Cox process (approximated via a Poisson likelihood with an exponential link function). The spatio-temporal GP has a separable Mat\'ern-$\nicefrac{3}{2}$ kernel. The results show high-resolution binning is required to make accurate predictions on this dataset (the test NLPD falls as the number of spatial bins increases). \cref{fig:rainforest} plots the data for this task, alongside the posterior mean given by the full model.

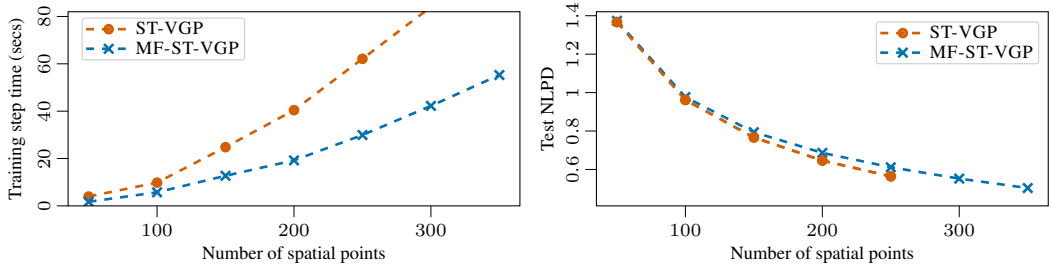
\begin{figure*}[t]
	\scriptsize
	\centering
	\pgfplotsset{yticklabel style={rotate=90}, ylabel style={yshift=0pt}, xlabel style={yshift=1pt},scale only axis,axis on top,clip=true,clip marker paths=true}
	\pgfplotsset{legend style={inner xsep=1pt, inner ysep=1pt, row sep=0pt},legend style={at={(0.98,0.95)},anchor=north east},legend style={rounded corners=1pt}}
	\setlength{\figurewidth}{.43\textwidth}
	\setlength{\figureheight}{.43\figurewidth}
\begin{tikzpicture}

\definecolor{color0}{rgb}{0.835294117647059,0.368627450980392,0}
\definecolor{color1}{rgb}{0,0.447058823529412,0.698039215686274}
\definecolor{color2}{rgb}{0,0.619607843137255,0.450980392156863}

\begin{axis}[
height=\figureheight,
legend cell align={left},
legend style={fill opacity=0.8, draw opacity=1, text opacity=1, at={(0.03,0.97)}, anchor=north west, draw=white!80!black},
tick align=outside,
tick pos=left,
width=\figurewidth,
x grid style={white!69.0196078431373!black},
xlabel={Number of spatial points},
xmin=35, xmax=365,
xtick style={color=black},
y grid style={white!69.0196078431373!black},
ylabel={Training step time (secs)},
ymin=-0, ymax=82,
ytick style={color=black}
]
\addplot [line width=1pt, color0, dashed, mark=*, mark size=1.5, mark options={solid}]
table {%
50 3.94942712783813
100 9.82986462116241
150 24.7969438791275
200 40.4197266817093
250 62.1134148836136
300 84.0306412220001
350 107.727702450752
};
\addlegendentry{\STVGP}
\addplot [line width=1pt, color1, dashed, mark=x, mark size=2.5, mark options={solid}]
table {%
50 1.72235972881317
100 5.72492461204529
150 12.7170865297318
200 19.2151525497437
250 29.9252356767654
300 42.2280537128448
350 55.2650584936142
};
\addlegendentry{\MFSTVGP}
\end{axis}

\end{tikzpicture}
	\hfill
\begin{tikzpicture}

\definecolor{color0}{rgb}{0.835294117647059,0.368627450980392,0}
\definecolor{color1}{rgb}{0,0.447058823529412,0.698039215686274}
\definecolor{color2}{rgb}{0,0.619607843137255,0.450980392156863}

\begin{axis}[
height=\figureheight,
legend cell align={left},
legend style={fill opacity=0.8, draw opacity=1, text opacity=1, draw=white!80!black},
tick align=outside,
tick pos=left,
width=\figurewidth,
x grid style={white!69.0196078431373!black},
xlabel={Number of spatial points},
xmin=35, xmax=365,
xtick style={color=black},
y grid style={white!69.0196078431373!black},
ylabel={Test NLPD},
ymin=0.410543533886014, ymax=1.42083968278815,
ytick style={color=black}
]
\addplot [line width=1pt, color0, dashed, mark=*, mark size=1.5, mark options={solid}]
table {%
50 1.36718791559265
100 0.961850800176402
150 0.766516691730963
200 0.646394969493459
250 0.564309963842669
300 nan
350 nan
};
\addlegendentry{\STVGP}
\addplot [line width=1pt, color1, dashed, mark=x, mark size=2.5, mark options={solid}]
table {%
50 1.37354278062849
100 0.976561477453564
150 0.79400055114041
200 0.686440264902753
250 0.611466298490523
300 0.552460237755402
350 0.503365688238174
};
\addlegendentry{\MFSTVGP}
\addplot [line width=1pt, color0, dashed, mark=*, mark size=1.5, mark options={solid}, forget plot]
table {%
50 1.36718791559265
100 0.961850800176402
150 0.766516691730963
200 0.646394969493459
250 0.564309963842669
300 nan
350 nan
};
\end{axis}

\end{tikzpicture}
	\vspace*{-0.1cm}
	\caption{Comparison of \STVGP and \MFSTVGP. A two-dimensional grid of count data is binned with 200 time steps, $N_t$, and a varying number of spatial bins, $N_\Space$. A Mat\'ern-$\nicefrac{3}{2}$ prior is used ($d_t=2$, so $d=2N_\Space$). We show the time taken to perform one training step, averaged across 10 runs (\textbf{left}), and the test negative log predictive likelihood using 10-fold cross-validation (\textbf{right}).} %
	\label{fig:scalability}
\end{figure*}

\begin{figure*}[t]
	\scriptsize
	\centering
	\pgfplotsset{yticklabel style={rotate=90}, ylabel style={yshift=0pt}, xlabel style={yshift=1pt},scale only axis,axis on top,clip=true,clip marker paths=true}
	\pgfplotsset{legend style={inner xsep=1pt, inner ysep=1pt, row sep=0pt},legend style={at={(0.98,0.95)},anchor=north east},legend style={rounded corners=1pt}}
	\setlength{\figurewidth}{.41\textwidth}
	\setlength{\figureheight}{.5\figurewidth}
\begin{tikzpicture}

\begin{axis}[
height=\figureheight,
tick align=outside,
tick pos=left,
width=\figurewidth,
x grid style={white!69.0196078431373!black},
xlabel={first spatial dimension, \(\displaystyle t\) (metres)},
xmin=0, xmax=1000,
xtick style={color=black},
y grid style={white!69.0196078431373!black},
ylabel={second spatial dimension, \(\displaystyle \Space\) (metres)},
ymin=0, ymax=500,
ytick style={color=black}
]
\addplot graphics [includegraphics cmd=\pgfimage,xmin=-227.028387096774, xmax=1191.75870967742, ymin=-103.309285714286, ymax=610.547857142857] {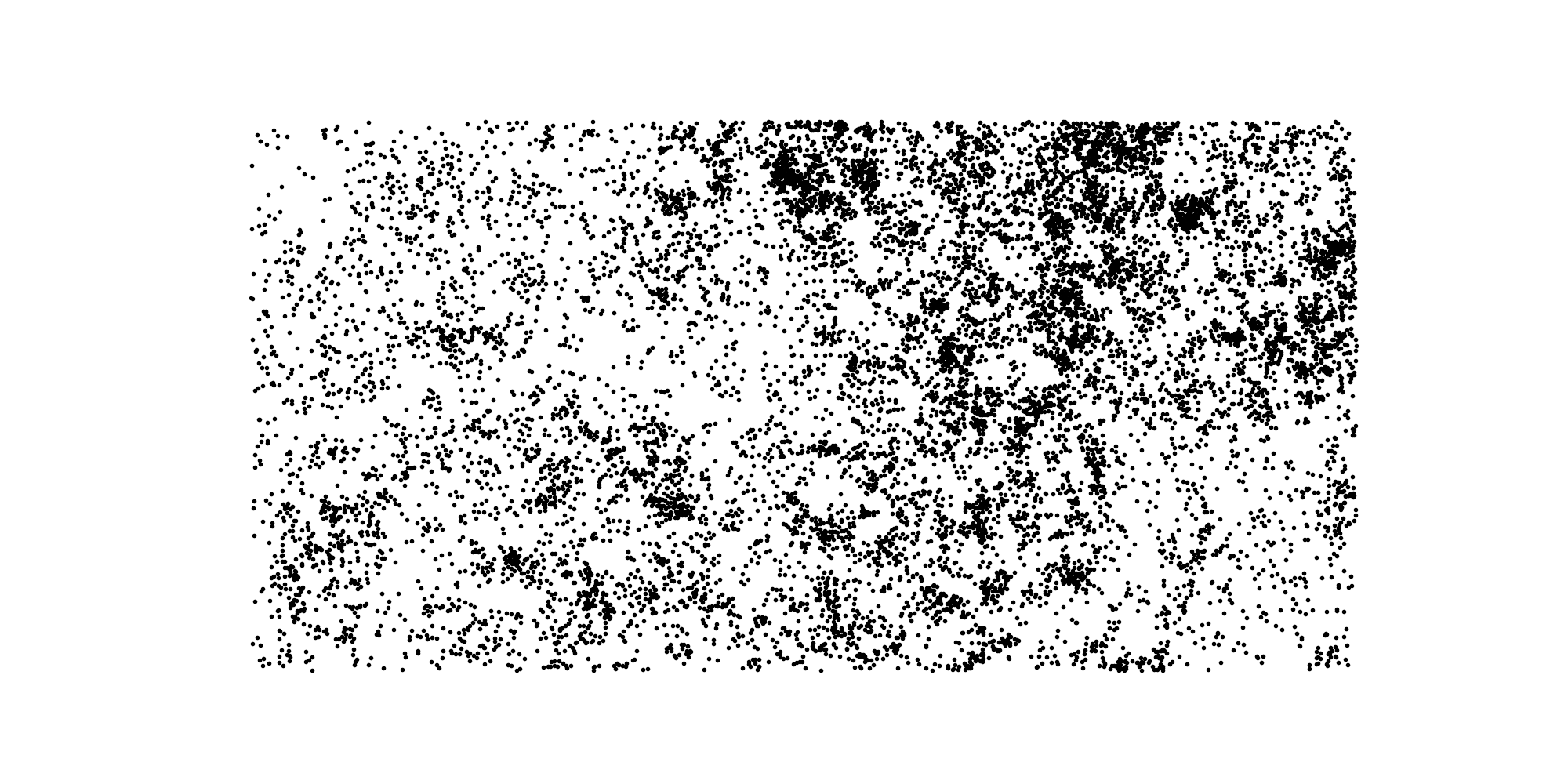};
\end{axis}

\end{tikzpicture}
 	\hfill
	\hspace{-1em}
\begin{tikzpicture}

\begin{axis}[
height=\figureheight,
tick align=outside,
tick pos=left,
width=\figurewidth,
x grid style={white!69.0196078431373!black},
xlabel={first spatial dimension, \(\displaystyle t\) (metres)},
xmin=0, xmax=1000,
xtick style={color=black},
y grid style={white!69.0196078431373!black},
ylabel={\phantom{second spatial dimension, \(\displaystyle \Space\) (metres)}},
ymin=0, ymax=500,
ytick style={color=black}
]
\addplot graphics [includegraphics cmd=\pgfimage,xmin=0, xmax=1000, ymin=0, ymax=500] {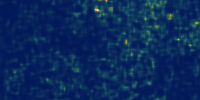};
\end{axis}

\end{tikzpicture}
	\caption{The tree count data used for the comparison of spatial approximations. The tree locations (\textbf{left}), are binned with temporal resolution $N_t=200$ and spatial resolution $N_\Space=100$, and a log-Gaussian Cox process is applied. The posterior mean given by the full model (\STVGP) is shown (\textbf{right}). See text for further details.}\label{fig:rainforest}
\end{figure*}

\subsection{Air Quality}

For all models we initialise the likelihood noise to $5.0$, use a Mat\'ern-$\nicefrac{3}{2}$ kernel with lengthscales initialised to $[0.01, 0.2, 0,2]$ and variance to $1.0$ and run for 300 epochs. See \cref{fig:air_quality_timeseries} for an example of the posterior obtained for a single spatial location over the course of three months.

\subsection{\nyccrime}

For all models we use a Mat\'ern-$\nicefrac{3}{2}$ kernel with lengthscales initialised to $[0.001, 0.1, 0,1]$ and variance to $1.0$ and run for 500 epochs. We use a natural gradient step size of $0.1$. See \cref{fig:teaser} for demonstrative plots of the predicted crime counts over NYC given by \STVGP across eight days in 2015.

\subsection{Downloading Data}

We have published the exact train-test folds for each dataset in \citet{hosted_data}.

 }
\fi

\end{document}